\documentclass[11pt]{article}

\usepackage[margin=1in]{geometry}
\usepackage{amsmath, amssymb}
\usepackage{graphicx}
\usepackage{hyperref}
\usepackage{setspace}
\usepackage{authblk}
\usepackage{amsthm}
\usepackage{cleveref}
\usepackage{booktabs}
\usepackage{subcaption}

\usepackage[authoryear]{natbib}

\bibpunct{[}{]}{,}{a}{,}{,}

\newtheorem{proposition}{Proposition}
\newtheorem{corollary}{Corollary}
\newtheorem{lemma}{Lemma}

\begin{document}

\title{MABLE: Masked Autoencoding with Bi-Lipschitz Decoding for Embeddings and Graph Metric Learning}
\author[1]{Yaniv Shulman\thanks{Correspondence: \href{mailto:yaniv@shulman.info}{yaniv@shulman.info}. This work was done while Yaniv Shulman was working at Fleet Space Technologies.}}
\author[1]{Shaghayegh Akbarpour\thanks{shae.akbarpour@fleet.space}}
\author[1,2]{Jack B. Muir\thanks{jack.muir@fleet.space}}

\affil[1]{Fleet Space Technologies, Adelaide, Australia \\ \url{https://www.fleet.space}}
\affil[2]{Research School of Earth Sciences, Australian National University, Acton, Australia}
\date{\today}
\maketitle

\begin{abstract}
We propose MABLE (\textbf{M}asked \textbf{A}utoencoding with \textbf{B}i-\textbf{L}ipschitz Decoding for \textbf{E}mbeddings and Graph Metric Learning), a self-supervised framework for learning node and graph embeddings from large, heterogeneous graphs, demonstrated here on geospatial mineral-exploration data. MABLE combines masked reconstruction with fixed cosine-similarity losses that align matched augmented views while keeping unpaired embeddings well spread. A bi-Lipschitz feature decoder ties a low-dimensional reconstruction component of each node embedding to feature similarity, while matched-node consistency shapes the remaining context used by graph pooling. Lipschitz-controlled pooling helps stabilize graph-level representations under perturbations of retained node embeddings, while augmentation alignment trains robustness to masking, node dropping, and sampling variation. Across local copper and regional Arabian Shield studies, MABLE embeddings provide complementary downstream signal and produce coherent embedding-derived layers for hypothesis generation without learned discriminators or hard-negative selection.
\end{abstract}

\section{Introduction}

Recent advances in self-supervised learning have fueled the development of
foundation models in domains ranging from computer vision~\citep{He2022MAE}
to geospatial analysis~\citep{Daruna2024GFM4MPM}, where tasks like semantic similarity search of distant regions or mineral prospectivity mapping
require robust representations of spatially distributed data. In particular,
\emph{masked autoencoders} (MAEs) and \emph{contrastive} methods have emerged as two prominent paradigms. MAEs reconstruct missing content from partially observed inputs, thereby learning localized semantics~\citep{He2022MAE, LiEtAl2023Mask, Zhang2022UMAE, Hou2022GraphMAE, Hou2023GraphMAE2}, while contrastive approaches push away “unrelated” samples to avoid representational collapse~\citep{Oord2018CPC,Chen2020SimCLR}. It has recently been shown that plain MAEs can still suffer from partial dimensional collapse in feature space; thus, adding a \emph{uniformity or repulsion term} can spread out different samples’ embeddings~\citep{Zhang2022UMAE}.

In geospatial applications, data are often geographically uneven, and when absolute location is provided as an input feature, it may spuriously dominate learned representations if not properly controlled~\citep{Wu2024TorchSpatial}. Our work \emph{unifies} masked reconstruction and metric learning in a single end-to-end framework on \emph{graphs}. Here, metric learning refers to learning an embedding geometry in which distances or cosine similarities reflect useful node- and graph-level relationships, using self-supervised objectives rather than supervised pair or triplet labels. To avoid trivial geospatial shortcuts, we exclude absolute/global position signals from the input features, and learn representations from observed features, relative positions, and graph structure alone. Specifically, we encourage each node’s latent feature representation to be accurate for reconstructing its masked features, a \emph{positive alignment} principle in the reconstruction-relevant component; to remain consistent across matched nodes in augmented views, a \emph{cross-view node-correspondence} principle in the complementary contextual component; and to remain well spread relative to other sampled nodes in the batch, a \emph{node uniformity} concept, as well as from entirely different graphs, a \emph{graph uniformity} concept. We impose a \emph{bi-Lipschitz} continuity constraint on the feature decoder (acting on a low-dimensional reconstruction component), so that accurate reconstruction encourages proximity in that reconstruction subspace, and we employ \emph{Lipschitz} pooling operators, notably an adaptive softmax attention pooling mechanism (§\ref{sec:softmax-attn-pool}), to aggregate node embeddings into robust graph-level representations. These Lipschitz-continuous aggregators ensure that small local perturbations in node embeddings produce proportionally bounded shifts in the resulting graph embedding, even though they do not in general admit a lower-Lipschitz bound or guarantee semantic separation.

Our framework employs normalized/cosine contrastive loss terms inspired by the Spectral Contrastive Loss (SCL) of \citep{HaoChen2021SCL}. Because MABLE uses cosine similarity rather than unconstrained raw inner products, the resulting fixed-cosine surrogate removes scale ambiguity and admits a heuristic second-order Donsker--Varadhan (DV)-style interpretation under boundedness, approximate symmetry, and small-variance assumptions (see Appendix~\ref{appendix:mi-bound} for details). In place of a learned discriminator and the InfoNCE log softmax \citep{Oord2018CPC}, we use a fixed squared-similarity surrogate, and introduce two key advancements:
First, we impose spectral constraints on the decoder, bounding its singular values to enforce both upper and lower Lipschitz continuity so that small reconstruction errors encourage proximity in the reconstruction subspace.  
Second, we embed the surrogate in a masked autoencoder and apply it at both node and graph scales, while also adding an explicit node-level correspondence objective across augmented views in a contextual component of the node embedding, all without explicit hard-negative selection, crucial for geospatial data where “negatives” can still be semantically related.

\section{A Unified End-to-End Framework}

Our masked-graph pretraining pipeline addresses three challenges in one pass:  
(i) encouraging node similarity in a reconstruction-relevant component via masked feature reconstruction and a bi-Lipschitz feature decoder, while also encouraging cross-view consistency of matched nodes in a complementary contextual component;
(ii) encouraging graph uniformity by pushing graph embeddings toward a well-spread, approximately isotropic distribution in normalized embedding space; and
(iii) suppressing geospatial bias by dispensing with explicit hard-negative selection and absolute/global locations.

Conceptually, the loss is the masked-graph counterpart of SCL, but expressed with cosine or squared-distance penalties instead of any learned discriminator or InfoNCE term. Combined with the spectral constraints, this fixed-surrogate design encourages reconstruction-driven alignment in the reconstruction component, cross-view node consistency in the contextual component, and graph-level alignment/uniformity without a separate substructure-to-graph module, while avoiding the instability of adversarial training. We draw inspiration from mutual-information methods such as Deep InfoMax \citep{Hjelm2018DeepInfoMax, Velickovic2019DeepGraphInfomax} while remaining fully non-adversarial, and from recent graph self-supervised learning methods including GraphCL, GRACE, and BGRL~\citep{You2020GraphCL, Zhu2020GRACE, Thakoor2021BGRL}.

Throughout, we use ``graph'' to denote a variable-sized collection of node observations with optional explicit edges. The experiments in this paper instantiate the edge-free case: no hand-crafted adjacency matrix is supplied, and interactions are induced by the attention mechanism over node tokens (see §\ref{sec:georeformer-exp}). Explicit edge-aware variants are a natural extension when reliable connectivity is available. The objective itself is not tied to geospatial inputs or to a particular encoder: the mineral-exploration setting below is the experimental demonstration and validation domain.

\subsection{Notation and Node Feature Similarity Preservation}
\label{sec:node-similarity}

We begin by formalizing the encoder--decoder model and illustrating how masked reconstruction can help preserve node-level similarity in the latent space, provided the decoder is suitably constrained. By hiding portions of each node’s features, we encourage the model to learn robust, cross-modal or cross-attribute relationships, rather than simply memorizing direct inputs.

\paragraph{Encoder--Decoder Setup.}
Let $G=(V(G), E(G))$ denote a graph with vertex set $V(G)$ and edge set $E(G)$. Each node $u\in V(G)$ is associated with a feature vector $x_u \in \mathbb{R}^{d_x}$, where $d_x$ denotes the feature (attribute) dimension; we write $X(G):=\{x_u\}_{u\in V(G)}$ for the collection of node features, and by abuse of notation also denote by
$X(G)\in\mathbb{R}^{|V(G)|\times d_x}$ the matrix obtained by stacking $x_u^\top$ as rows in an arbitrary but fixed node ordering. When $G$ and $X$ are clear from context, we write $X:=X(G)$. Here $f_\theta$ is a graph-level encoder (e.g.\ a GNN) that processes the whole graph $G$ together with an input feature collection $X$ (which may be masked) and returns a contextualised embedding for every node; for notational simplicity, any additional auxiliary per-node context supplied to the encoder is left implicit. We write $f(u)$ as shorthand for the more explicit $(f_\theta(G,X))_u$ when $G$ and $X$ are clear from context.
Our model consists of an encoder that produces node embeddings of dimension $d$ and a feature decoder that acts on a reconstruction-facing subspace of dimension $d_e\le d$:
\[
f_\theta(G,X(G)) \in \mathbb{R}^{|V(G)|\times d}
\quad\text{and}\quad
g: \mathbb{R}^{d_e} \;\to\; \mathbb{R}^{d_x},
\]
where $z_u := f(u) := (f_\theta(G,X(G)))_u \in\mathbb R^d$ denotes the \emph{node embedding}. When a minibatch contains multiple graphs, we distinguish \emph{node instances} by writing
\[
\mathcal{B}_{\mathrm{node}} := \{(G,u):\, G\in\mathrm{batch},\ u\in V(G)\},
\]
and for $(G,u)\in\mathcal{B}_{\mathrm{node}}$ we write
\[
f(G,u) := (f_\theta(G,X(G)))_u.
\]
Thus, the shorthand $f(u)$ is used only when the underlying graph $G$ (and input view $X$) are fixed by context. 

To reconstruct node features, we define a low-dimensional \emph{feature embedding} $e_u := \pi(z_u)\in\mathbb R^{d_e}$ as a coordinate slice $\pi(z)=z_{1:d_e}$. We denote the complementary \emph{contextual slice} by
\[
c_u := \rho(z_u) := z_{d_e+1:d}\in\mathbb R^{d-d_e}.
\]
The feature decoder $g$ reconstructs node features from $e_u$, and the autoencoder is $h(u):=g(e_u)=g(\pi(f(u)))$. In the sequel, masked reconstruction acts on the feature slice $e_u$, while optional cross-view node-consistency terms may be imposed on the contextual slice $c_u$. This split embeds a constrained reconstruction pathway inside the broader graph encoder: the feature slice is geometrically tied to reconstructible feature similarity through the decoder, while the contextual slice remains available for view-stable information used by pooling and downstream graph representations. It also decouples the full node-embedding dimension from the reconstruction dimension, allowing high-dimensional contextual node embeddings while keeping the decoder input sized relative to the actual feature dimension $d_x$.

Finally, let $P_\phi$ be a permutation-invariant pooling operator mapping multisets of node embeddings to a graph embedding:
\[
P_\phi:\ (\mathbb{R}^d)^{|V(G)|}\to \mathbb{R}^{d_G},\qquad
s_G := P_\phi(\{z_u\}_{u\in V(G)}).
\]

\noindent
In words, $P_\phi$ pools the unordered multiset of node embeddings into a single fixed-dimensional \emph{graph embedding} $s_G$, used for graph-level losses and downstream tasks. Permutation invariance ensures that $s_G$ depends only on the graph content, not on the arbitrary ordering of its nodes, consistent with the general set-function viewpoint of Deep Sets~\citep{Zaheer2017DeepSets}.

\paragraph{View operator for masking and augmentation.}
We model masking and augmentation as a random view operator $T\sim\mathcal{A}$ acting on the graph--feature pair:
\[
(G^{(a)},X^{(a)},\iota^{(a)})=T(G,X).
\]
Here $G^{(a)}$ and $X^{(a)}$ are the augmented graph and feature collection, while $\iota^{(a)}:V(G^{(a)})\to V(G)$ is an injective map recording the original node identity of each retained node. Feature-only masking is the special case $G^{(a)}=G$ and $\iota^{(a)}=\mathrm{id}$; node or subgraph drops remove vertices from $G^{(a)}$ and restrict $\iota^{(a)}$ to the retained nodes. Unless otherwise stated, when a single graph $G$ and input view $X$ are fixed by context, we continue to write $f(u):=(f_\theta(G,X))_u$ for the node embedding under the \emph{current} input $X$ (masked or unmasked). In minibatch settings involving multiple graphs, we instead write $f(G,u)$. For two sampled view operators $T_1,T_2$, we write $(G^{(a)},X^{(a)},\iota^{(a)})=T_a(G,X)$ for $a\in\{1,2\}$. In our masked setting, $h$ predicts missing or masked features from partially observed inputs, forcing the learned representation to capture meaningful dependencies within and across features.

\paragraph{Matrix norms and singular values.}
For any matrix $W\in\mathbb R^{p\times q}$, define
\[
\sigma_{\max}(W) \;:=\; \max_{\|v\|_2=1}\|Wv\|_2 \;=\; \sqrt{\lambda_{\max}(W^\top W)},
\qquad
\sigma_{\min}(W) \;:=\; \min_{\|v\|_2=1}\|Wv\|_2 \;=\; \sqrt{\lambda_{\min}(W^\top W)}.
\]
When $p\ge q$ (tall or square) and $W$ has full column rank, $\sigma_{\min}(W)>0$ and coincides with the smallest (nonzero) singular value in the usual ordering.
When $p<q$ (wide), $\sigma_{\min}(W)=0$.

\paragraph{Cosine similarity.}
We define cosine similarity as
\[
\cos(z,z') \;:=\;
\begin{cases}
\dfrac{z^\top z'}{\|z\|_2\,\|z'\|_2}, & \text{if } z\neq 0 \text{ and } z'\neq 0,\\[6pt]
0, & \text{otherwise.}
\end{cases}
\]

\paragraph{Positive Alignment via Masking.}
Masked autoencoders, including graph variants~\citep{He2022MAE,Zhang2022UMAE,Hou2022GraphMAE,Hou2023GraphMAE2}, suggest that masking can \emph{align} nodes in the reconstruction-relevant feature-embedding space when local structures are similar, because reconstructing missing features or, in graph-autoencoding settings, masked edges/structure from context often forces the model to group similar inputs together. If two nodes $u$ and $v$ share similar local context, features, or incident structure, then under sufficient masking, $h(u)$ and $h(v)$ must correctly fill in similar content or structure. This typically helps their learned representations $f(u)$ and $f(v)$ to become proximate. In this paragraph we use $f(u)$ and $g(\cdot)$ in the standard autoencoder sense (not MABLE-specific notation).

However, if the decoder $g$ is \emph{unconstrained}, the reconstruction loss only requires $h(u) = g\bigl(f(u)\bigr)$ to match $u$ in the
\emph{input--output} space. That alone does not strictly guarantee that $f(u) \approx f(v)$ for two nodes $u,v$ with similar original features.
In other words, $g$ might learn to “invert” $f(u)$ in arbitrary ways without necessarily preserving pairwise distances in the latent space. This scenario shares conceptual parallels with issues like posterior collapse in Variational Autoencoders (VAEs), where an overly powerful or flexible decoder can learn to ignore the latent variables, rendering them uninformative, because it can model the data distribution well without relying on structured latent codes~\citep{Wang2021PosteriorCollapse}. To ensure that the reconstruction-relevant latent space captured by $e_u=\pi(f(u))$ reflects meaningful similarities and is not trivialized by the decoder's capacity, we turn to explicit constraints on $g$.

\paragraph{Bi-Lipschitz Feature Decoding for Latent Alignment.}
To promote embedding-level alignment, we impose a bi-Lipschitz constraint on $g$. The next proposition shows that, when $g$ is bi-Lipschitz and reconstruction is accurate, similarity in the \emph{input} domain translates into controlled proximity in the reconstruction subspace.

\begin{proposition}[Node Feature Similarity Preservation]
\label{prop:lipschitz}
Let $z_u := f(u)\in\mathbb R^d$ denote the node embedding and let $e_u := \pi(z_u)\in\mathbb R^{d_e}$ denote the feature embedding. Let $h(u):=g(e_u)$ be an autoencoder whose feature decoder $g$ is \emph{bi-Lipschitz}; i.e.\ there exist constants $0 < m \le L$ such that
\[
  m\,\|e_1 - e_2\| \;\le\; \|g(e_1) - g(e_2)\| \;\le\; L\,\|e_1 - e_2\|, \quad \forall\, e_1,e_2\in\mathbb{R}^{d_e}.
\]
Suppose $h(u)=g(e_u)$ reconstructs each node $u$ to within $\varepsilon$ of its features:
\[
\|h(u) - x_u\| \le \varepsilon \quad \forall\,u.
\]
Then, for any nodes $u,v$,
\[
    \|e_u-e_v\|
    \le
    \frac{1}{m}\,\|x_u-x_v\| + \frac{2\varepsilon}{m}.
\]
Moreover,
\[
    \|e_u-e_v\|
    \ge
    \max\!\left\{0,\;\frac{\|x_u-x_v\| - 2\varepsilon}{L}\right\}.
\]

(This is a pointwise accuracy assumption on $h$; we specify practical training losses and masking-dependent instantiations later. The same inequality holds in any norm provided $g$ satisfies the corresponding co‑Lipschitz bound in that norm.)
\end{proposition}

See proof in Appendix~\ref{appendix:proof-lipschitz}. Thus, \emph{node feature similarity} is approximately preserved in the feature-embedding space used for reconstruction when the decoder is bi-Lipschitz and reconstruction error is small. The proposition is a pointwise geometric statement, while in training we optimize an empirical/expected mask-dependent reconstruction objective as its practical surrogate. Masking can amplify this effect because the model must predict missing information from partial inputs, implicitly aligning nodes that fill in similar content \citep{Zhang2022UMAE}.

\smallskip
\noindent
\textbf{Scope of the guarantee.}
Proposition~\ref{prop:lipschitz} is a statement about the reconstruction-relevant feature slice $e_u=\pi(z_u)$ under a bi-Lipschitz decoder and small reconstruction error. It does not by itself imply that the full node embedding $z_u$ preserves all task-relevant similarities, nor that masking alone guarantees alignment in the full latent space.

\smallskip
\noindent
\textbf{Practical note.}
We approximate the bi-Lipschitz property with a single linear decoder $g(e)=W_{dec}e+b$ whose singular values are kept in a fixed band
$m\!\le\!\sigma_i(W_{dec})\!\le\!L$. Implementation notes and computational details are discussed in §\ref{sec:condition-number}. For a linear map, achieving a non-trivial lower bound $m>0$ requires $\sigma_{\min}(W_{dec})>0$, i.e.\ $W_{dec}$ has full column rank (and thus $d_e \le d_x$ in the Euclidean setting). If $d_e > d_x$,
then $\sigma_{\min}(W_{dec})=0$ and no nontrivial lower bound exists.

\subsection{Fixed-Cosine Surrogates}
\label{sec:fixed-cosine-surrogates}

InfoMax–style objectives learn representations by contrasting the joint distribution \(p_{\text{pos}}(z,z^{+})\), with \(z=f(x)\) and \(z^{+}=f(x^{+})\) an augmented view, against the product distribution \(p(z)^{2}\) of unpaired embeddings~\citep{Hjelm2018DeepInfoMax,Velickovic2019DeepGraphInfomax,Oord2018CPC,He2019MoCo}. Positive pairs \((z,z^{+})\) are encouraged to score higher than uncorrelated pairs \((z,z^{-})\) treated as approximately independent negatives (e.g., in-batch), typically by maximising a mutual‑information lower bound such as InfoNCE, a Jensen–Shannon variant, or the binary cross‑entropy (BCE) objective~\citep{Gutmann2010NCE}:
\[
\mathcal L_{\text{BCE}}
=
\mathbb E_{(z,z^{+})}\bigl[-\log \sigma\bigl(D(z,z^{+})\bigr)\bigr] +
\mathbb E_{(z,z^{-})}\bigl[-\log\bigl(1-\sigma\bigl(D(z,z^{-})\bigr)\bigr)\bigr],
\]
where \(\sigma(u)=1/(1+e^{-u})\) and \(D\) is a learned discriminator. Besides introducing additional parameters, this formulation assumes that in-batch negatives are semantically unrelated, an assumption that often fails for real-world geospatial corpora, where nearby coordinates typically share land-cover, climate, or geology. In practice, when absolute coordinates are available to the model, a powerful discriminator may latch onto absolute location: the simplest way to separate a random negative from its positive is often just coordinate distance, not a meaningful semantic cue. In our setting, we avoid this shortcut at the source by \emph{not} feeding absolute/global positions as input features. We still prefer fixed-similarity surrogates over learned discriminators to reduce the risk of other trivial separation cues dominating training dynamics in geographically uneven corpora.

To address these limitations and avoid the learned discriminator, we use fixed cosine-similarity objectives. The raw inner-product form of the Spectral Contrastive Loss (SCL) of~\cite{HaoChen2021SCL} already expresses the desired pairwise geometry:
\[
\mathcal L_{\text{SCL,raw}}
=
-2 \, \mathbb{E}_{(z,z^{+})}\bigl[z^\top z^{+}\bigr]
+ \mathbb{E}_{(z,z^{-})}\bigl[(z^\top z^{-})^{2}\bigr].
\]
The positive term rewards large positive-pair inner products, while the quadratic negative term penalizes nonzero inner products for unpaired embeddings. The original population objective is written in terms of raw inner products, but the empirical setting in~\citet{HaoChen2021SCL} still uses norm control, including projection to a bounded ball with experimentally chosen radius, and the finite-sample analysis assumes norm-controlled function classes. Without normalization, however, this objective mixes angular structure with vector scale: after angular alignment saturates, the positive term can still be improved by increasing norms, and the negative term can be reduced either by orthogonality or by shrinking norms.

In MABLE we remove this scale ambiguity by working directly with cosine similarity, \(\cos(z,z') \in [-1,1]\). For unit-normalized embeddings, \(\cos(z,z')=z^\top z'\), and SCL becomes the normalized/cosine objective
\[
\mathcal L_{\text{SCL}}
=
-2 \, \mathbb{E}_{(z,z^{+})}\bigl[\cos(z,z^{+})\bigr]
+ \mathbb{E}_{(z,z^{-})}\bigl[\cos(z,z^{-})^{2}\bigr].
\]
This cleanly targets angular alignment of positive pairs ($\cos(z,z^{+}) \to 1$) and angular decorrelation of unpaired embeddings ($\cos(z,z^{-}) \to 0$).

The same geometry can be motivated from a fixed log-based surrogate. For positive pairs, choose an argument whose negative log is minimized as $\cos(z,z^{+})\to 1$, namely $\exp(\cos(z,z^{+})-1)$. For unpaired embeddings, choose an argument whose negative log is minimized as $\cos(z,z^{-})\to 0$, namely $1-\frac12\cos(z,z^{-})^2$. This gives
\[
\mathcal L_{\text{Surrogate}} = -\mathbb{E}_{(z,z^{+})}[\log(\exp(\cos(z,z^{+})-1))] -\mathbb{E}_{(z,z^{-})}[\log(1 - \frac{1}{2}\cos(z,z^{-})^2)].
\]
The positive term simplifies, up to an additive constant, to $-\cos(z,z^{+})$. The negative term has the Taylor expansion
\[
-\log\!\left(1-\frac{1}{2}\cos(z,z^{-})^2\right)
=
\frac12\cos(z,z^{-})^2 + O(\cos(z,z^{-})^4)
\]
near the desired negative optimum $\cos(z,z^{-})=0$.

Thus the local quadratic approximation to the log-based surrogate is the scaled normalized SCL objective. Let $\mathcal L_{\text{SSCL}} = \tfrac12 \, \mathcal L_{\text{SCL}}$:
\[
\mathcal L_{\text{SSCL}}
=
-\mathbb{E}_{(z,z^{+})}\bigl[\cos(z,z^{+})\bigr]
+ \tfrac12 \, \mathbb{E}_{(z,z^{-})}\bigl[\cos(z,z^{-})^{2}\bigr].
\]
This is the fixed-cosine geometry used throughout MABLE: the first term aligns true pairs, while the quadratic penalty pushes unpaired embeddings toward orthogonality, without requiring a learned discriminator or an InfoNCE softmax.

\paragraph{Alignment, Uniformity, and Downstream Separability.}
Empirically, achieving strong intra-class ``alignment'' (similar samples are close in the latent space) and cross-class ``uniformity'' (dissimilar samples are spread out across the latent space, often on the hypersphere) often correlates with improved linear separability on downstream tasks (see, e.g., \citep{Tosh2020ContrastiveLM}; under additional multi-view redundancy assumptions, their Theorem~3 shows that linear functions of suitable embeddings can approximate the Bayes-optimal predictor).
Uniformity on the hypersphere is particularly valuable in a fixed-cosine similarity space because random embeddings then have a near-zero expected dot product. This near-zero baseline for random dissimilar pairs provides a natural separation from concentrated positive pairs that exhibit high similarity. Here, we enforce alignment and uniformity via fixed similarity-based objectives at the graph level (SCL-style alignment/uniformity), while at the node level alignment is induced by reconstruction in the feature-embedding subspace, cross-view consistency is encouraged in the complementary contextual slice via matched-node correspondence across augmentations, and uniformity is encouraged by a squared-cosine repulsion term. This design shifts the representational burden primarily onto the encoder, which must learn node embeddings whose reconstruction-relevant slice preserves local feature similarity and whose contextual slice remains stable across views while still spreading unrelated samples apart, all without requiring a learned discriminator or the InfoNCE objective itself. In practice, the resulting representation spaces (in particular the graph embeddings, the feature-embedding component used for reconstruction, and the contextual component used for cross-view consistency) often become easier to separate with simple downstream models, so a linear classifier or even a $k$-NN rule can perform competitively \citep{Wang2020AlignmentUniformity}. Recent refinements like \emph{MIM-Refiner}~\citep{AlkinEtAl2025MIM} further demonstrate that intermediate-layers' representations can be re-trained with an instance discrimination head to improve cluster coherence and alignment, thereby making the final embedding layers more suitable for low-shot classification and other downstream tasks.

\subsection{From Node Embeddings to Graph Embeddings via Lipschitz Pooling}
\label{sec:pooling-lipschitz}

After obtaining node embeddings $\{z_u\}_{u \in V(G)}$, many applications
require distilling them into a single \emph{graph-level} representation.
In this subsection, we discuss how to preserve alignment in the embedding
space by employing a \emph{Lipschitz} pooling function.

\paragraph{Generic Lipschitz Pooling.}
Fix a norm $\|\cdot\|$ on $\mathbb{R}^d$ and equip $(\mathbb{R}^d)^n$ with the product metric
$\|(z_i)-(w_i)\|_{\mathrm{sum}} := \sum_{i=1}^n \|z_i-w_i\|$.
We say a function $P:(\mathbb{R}^d)^n\to\mathbb{R}^{d_G}$ is \emph{Lipschitz} if there exists a constant $L\ge 0$ such that for all ordered collections $(z_i)_{i=1}^n,(w_i)_{i=1}^n$,
\[
\|P(z_1,\dots,z_n)-P(w_1,\dots,w_n)\|
\;\le\;
L\,\|(z_i)-(w_i)\|_{\mathrm{sum}}.
\]
In the context of pooling, we view a graph as a \emph{multiset} of node embeddings $Z=\{z_i\}_{i=1}^n\subset\mathbb{R}^d$, where node order is arbitrary.

\smallskip
\noindent
\textbf{Remark.} In this section, $P$ denotes a generic pooling map. Our model uses a (possibly learnable) pooling operator $P_\phi$ as defined in \S\ref{sec:node-similarity}, and all results apply to $P_\phi$ whenever it satisfies the stated assumptions (permutation invariance and Lipschitz continuity).

\paragraph{Permutation invariance and matching distance.}
In our setting, the pooling operator acts on the (orderless) multiset of node embeddings produced by the encoder, so we assume $P$ is \emph{permutation-invariant}:
$P(z_1,\dots,z_n)=P(z_{\xi(1)},\dots,z_{\xi(n)})$ for any permutation $\xi\in S_n$.
Accordingly, when $P$ is permutation-invariant we write $P(Z)$ to denote the value of $P$ on any ordering of the multiset $Z$.
To compare two graphs without a predefined node correspondence, we use the optimal matching distance
\[
d_{\mathrm{match}}(Z,Z')\;:=\;\min_{\xi\in S_n}\sum_{i=1}^n\|z_i-z'_{\xi(i)}\|,
\qquad Z=\{z_i\}_{i=1}^n,\; Z'=\{z'_i\}_{i=1}^n.
\]
(When graph sizes differ, as happens under node-dropping or subgraph augmentations, one can pad the smaller set with a fixed null embedding or use a partial-matching variant; in the main text we use Proposition~\ref{prop:lipschitz-pooling} as an idealized equal-cardinality stability result rather than as a full theorem for every augmentation regime.)

\begin{proposition}[Lipschitz Pooling is Stable Under Optimal Matching]
\label{prop:lipschitz-pooling}
Let $P:(\mathbb{R}^d)^n\to\mathbb{R}^{d_G}$ be permutation-invariant and Lipschitz with constant $L$ with respect to $\|\cdot\|_{\mathrm{sum}}$.
Then for any two multisets of node embeddings $Z=\{z_i\}_{i=1}^n$ and $Z'=\{z'_i\}_{i=1}^n$,
\[
 \bigl\|P(Z)-P(Z')\bigr\|
 \;\le\;
 L\, d_{\mathrm{match}}(Z,Z').
 \]
Consequently, for any chosen correspondence, if the matched node embeddings differ only slightly, then the pooled graph embeddings differ proportionally. In applications, correspondences may be inherited from node identity, augmentation bookkeeping, or coordinate-based matching, but the guarantee is in the matched embedding space rather than in coordinate space itself.
\end{proposition}

\noindent
The proof is given in Appendix~\ref{appendix:pooling-proofs}. We emphasize that Proposition~\ref{prop:lipschitz-pooling} is stated for equal-cardinality multisets; when augmentations change graph size, we use it as the natural equal-cardinality stability result, while practical implementations may instead rely on padding or partial-matching variants as noted above.

\paragraph{Examples of Lipschitz Pooling: Sum, Mean, and Max.}
Many standard pooling operators, such as sum, mean, and coordinate‐wise max, are
Lipschitz transformations. Informally, each uses a linear or piecewise‐linear
aggregation over the embeddings, ensuring bounded output changes when inputs
are perturbed slightly. Moreover, since these operators are permutation-invariant,
their Lipschitz bounds on ordered tuples immediately imply stability for unordered
multisets under the optimal matching distance $d_{\mathrm{match}}$ via
Proposition~\ref{prop:lipschitz-pooling} (with $d_{\mathrm{match}}$ defined using
the same underlying norm as in the input metric).

\begin{proposition}[Sum, Mean, and Max Pooling are Lipschitz]
\label{prop:sum-mean-max-lipschitz}
Consider the following three operators on node embeddings
$\{z_i\}_{i=1}^n \subset \mathbb{R}^d$:
\begin{align*}
\text{SumPool}(\{z_i\}) \; &=\; \sum_{i=1}^n z_i, \\
\text{MeanPool}(\{z_i\}) \; &=\; \frac{1}{n} \sum_{i=1}^n z_i, \\
\text{MaxPool}(\{z_i\})_k \; &=\; \max_{1 \le i \le n} (z_i)_k
\quad \text{for each coordinate } k.
\end{align*}
Then each of these is Lipschitz, with constants that depend on $n$ and
the norm. Specifically:
\begin{itemize}
\item With output norm $\ell_2$ and input metric $\sum_i \|z_i-z'_i\|_2$, SumPool has Lipschitz constant at most $1$.
\item With output norm $\ell_2$ and input metric $\sum_i \|z_i-z'_i\|_2$, MeanPool has Lipschitz constant at most $1/n$.
\item With output norm $\ell_\infty$ and input metric $\sum_{i=1}^n \|z_i-z'_i\|_2$, MaxPool has Lipschitz constant at most $1$
      (since $\max_i\|z_i-z'_i\|_\infty \le \sum_i\|z_i-z'_i\|_2$).
      If one instead measures the output in $\ell_2$, the constant is at most $\sqrt{d}$.
\end{itemize}
\end{proposition}

\noindent
The detailed proof is given in Appendix~\ref{appendix:proof-sum-mean-max-lipschitz}. The essence is a triangle‐inequality argument. While these operators provide basic Lipschitz aggregation, more expressive mechanisms like Softmax Attention Pooling (detailed in §\ref{sec:softmax-attn-pool}) can offer enhanced representational power while maintaining controlled Lipschitz properties.

\subsection{Adaptive Graph Pooling with Softmax Attention}
\label{sec:softmax-attn-pool}

While the pooling operators discussed in §\ref{sec:pooling-lipschitz} (sum, mean, max) offer simplicity and clear Lipschitz bounds, more adaptive mechanisms can provide enhanced flexibility and empirical performance in generating graph-level embeddings, in the spirit of learned set/graph aggregation methods such as Set Transformer and Graph Multiset Transformer pooling~\citep{Lee2019SetTransformer, Baek2021GMT}. This is particularly relevant for graphs with heterogeneous node densities or complex relational structures where uniform weighting of nodes may be suboptimal. Motivated by these learned aggregation methods, we use Softmax Attention Pooling as a learnable, context-aware aggregation mechanism:
\[
P_{\mathrm{attn}}(\{z_i\}_{i=1}^n) = W_{pool} \left( \sum_{i=1}^n \alpha_i(\{z_k\}_{k=1}^n)\, z_i \right),
\]
where $\{z_i\}$ are the node embeddings, $W_{pool} \in \mathbb{R}^{d_G \times d}$ is a trainable projection matrix, and $\alpha_i \ge 0$ (with $\sum_i \alpha_i = 1$) are attention weights. These weights $\alpha_i$ dynamically depend on the full set of node embeddings $\{z_k\}_{k=1}^n$, often computed via a softmax over scores derived from small neural networks (e.g., MLPs) applied to node embeddings.

For $P_{\mathrm{attn}}$ to serve as a robust Lipschitz pooling operator within the MABLE framework, we state the guarantee on bounded subsets of the node-embedding space rather than globally on all of $(\mathbb{R}^d)^n$. In our implementation, the attention-score network is fed unit-normalized inputs
\[
\bar z_i := \frac{z_i}{\max(\|z_i\|_2,\varepsilon)},
\]
its linear maps are constrained only from above via spectral-norm clamping, and the pre-softmax logit gaps are clipped in a shift-invariant way:
\[
\tilde s_i := \psi(\bar z_i),\qquad
s_i := \mathrm{clip}\!\bigl(\tilde s_i - \max_j \tilde s_j,\,-c,\,0\bigr),
\qquad
\alpha_i := \mathrm{softmax}(s)_i.
\]
Assume additionally that the pooled values lie in a bounded region, $\|z_i\|_2 \le B$ for all $i$ in the domain of interest, that the normalization uses a fixed $\varepsilon>0$, and that the score network $\psi$ is composed of standard Lipschitz activations together with linear maps whose operator norms are clamped from above. Because the score network acts on normalized inputs, its Lipschitz constant is controlled by the product of the clamped operator norms of its linear layers, while clipping places the logits in a compact set on which softmax is Lipschitz. Under these conditions, the attention aggregation map
 \[
A(\{z_k\}_{k=1}^n)\;:=\;\sum_{i=1}^n \alpha_i(\{z_k\}_{k=1}^n)\, z_i
 \]
admits a Lipschitz constant on the bounded domain of interest, with respect to the chosen product metric on $(\mathbb{R}^d)^n$ (e.g., $\|\cdot\|_{\mathrm{sum}}$), depending on the value bound $B$, the logit-gap bound $c$, the set size $n$, and the score network $\psi$. Consequently, if the final linear projection satisfies $\sigma_{\max}(W_{pool})\le L_W$, then the composed pooling operator $P_{\mathrm{attn}}:=W_{pool}\circ A$ is also Lipschitz on that bounded domain. Note that because we do \emph{not} normalize the pooled values $z_i$ themselves, this is a bounded-domain (local) Lipschitz statement, not a global one over unbounded node embeddings.

\begin{lemma}[Bounded-domain Lipschitzness of clipped softmax attention]
\label{lem:attn-lipschitz}
Fix $n$ and consider the bounded domain $\mathcal{D}_B=\{(z_1,\dots,z_n):\|z_i\|_2\le B\}$. Suppose the normalization uses $\varepsilon>0$, the score map $\psi$ is $L_\psi$-Lipschitz, logits are shifted by their maximum and clipped to $[-c,0]$ before softmax, and attention weights are used to form
\[
A(Z)=\sum_{i=1}^n \alpha_i(Z)z_i.
\]
Then $A$ is Lipschitz on $\mathcal{D}_B$ with respect to the product metric $\sum_i\|z_i-z'_i\|_2$. Consequently, if $\sigma_{\max}(W_{pool})\le L_W$, then $P_{\mathrm{attn}}=W_{pool}\circ A$ is also Lipschitz on $\mathcal{D}_B$.
\end{lemma}
\noindent
The proof is a direct composition argument and is given in Appendix~\ref{appendix:attn-lipschitz-proof}.

Beyond this upper-Lipschitz/stability guarantee, one may additionally ask whether the pooling head preserves some pairwise separation. In our implementation, the \emph{final pooling projection} $W_{pool}$ is explicitly conditioned by clamping its singular values to a band $[m,L]$ in the regime where the projection is square or full-column-rank, while the \emph{score-network} matrices used to produce the attention logits are constrained only from above. Accordingly, the corollary below isolates a direct consequence relevant in our setting: once the attention-weighted sum has achieved a non-zero separation for a given pair, the lower singular-value bound on $W_{pool}$ helps ensure that this separation is not collapsed by the final projection, though this does not by itself yield a general lower bound with respect to a standard set metric.

\begin{corollary}[Conditioned Projection Preserves Attention-Induced Separation]
\label{prop:cond-pro-pool}
Let $P_{\mathrm{attn}}(\{z_i\}) = W\!\bigl(\sum_{i=1}^n \alpha_i(\{z_k\}_{k=1}^n)\,z_i\bigr)$ define the pooling operation, where \(\alpha_i(\cdot)\ge 0\) are attention weights satisfying $\sum_{i=1}^n \alpha_i(\{z_k\}_{k=1}^n) = 1$, and $W \in \mathbb{R}^{d_G \times d}$ with $\sigma_{\min}(W)\ge m>0$ (hence co-Lipschitz). Consider two distinct sets of embeddings $\{z_i\}$ and $\{z'_i\}$ from the data distribution. Let $\delta_{\text{pair}}$ denote the actual separation between their attention-weighted sums produced by the mechanism $\alpha_i(\cdot)$:
\[
\delta_{\text{pair}} = \Bigl\|\sum_{i=1}^n \alpha_i(\{z_k\}_{k=1}^n)\, z_i \;-\; \sum_{i=1}^n \alpha_i(\{z'_k\}_{k=1}^n)\, z'_i \Bigr\|.
\]
If this achieved separation is positive ($\delta_{\text{pair}} > 0$), then the pooled embeddings for these specific sets satisfy
\[
\bigl\|P_{\mathrm{attn}}(\{z_i\}) - P_{\mathrm{attn}}(\{z'_i\})\bigr\|
\;\;\ge\;\;
m \cdot \delta_{\text{pair}}.
\]
Hence $P_{\mathrm{attn}}$ preserves the separation induced by the attention mechanism for this pair $\{z_i\}, \{z'_i\}$ up to the factor $m$, provided the attention mechanism yields a non-zero separation $\delta_{\text{pair}}$.
\end{corollary}

\noindent
This does not provide a uniform lower bound in terms of a standard set metric; it states only that any separation realized by the attention-weighted sum is not collapsed by the lower-bounded final projection. In our implementation this lower-bound condition is enforced on $W_{pool}$, whereas the attention \emph{scoring} network itself is controlled only through input normalization, upper spectral-norm constraints, and logit clipping (see Appendix~\ref{appendix:bilip-pool-proof}).

\subsection{Condition Number Control for Linear Decoders}
\label{sec:condition-number}

While a general Lipschitz or bi-Lipschitz decoder can be implemented via architectural constraints, which often introduce complexity or complicate training \citep{Miyato2018SpectralNorm, Cisse2017Parseval, Behrmann2019InvertibleResNet, Gulrajani2017WGANGP}, a particularly simple yet powerful approach uses a single linear decoding layer, $g(e) = W_{dec}e + b$, where $W_{dec} \in \mathbb{R}^{d_x \times d_e}$.
For such a linear map from feature-embedding space ($\mathbb{R}^{d_e}$) to feature space ($\mathbb{R}^{d_x}$), the maximum singular value $\sigma_{\max}(W_{dec})=\|W_{dec}\|_2$ dictates its Lipschitz constant (bounding expansion), and the minimum singular value $\sigma_{\min}(W_{dec})$ determines its co-Lipschitz constant (bounding collapse). Specifically, a non-zero $\sigma_{\min}(W_{dec}) > 0$ implies injectivity for the linear map (which in particular requires $d_e \le d_x$), preventing distinct latent vectors from collapsing to the same output. By ensuring both $0 < m \le \sigma_{\min}(W_{dec})$ and $\sigma_{\max}(W_{dec}) \le L < \infty$ (which in particular requires $d_e \le d_x$), the linear decoder becomes bi-Lipschitz, guaranteeing that distances are preserved up to bounded factors: $m\|e_1 - e_2\| \le \|g(e_1) - g(e_2)\| \le L\|e_1 - e_2\|$.

Controlling these singular values in practice adds computational cost. Calculating the full Singular Value Decomposition (SVD) of the $d_x \times d_e$ matrix $W_{dec}$ is typically $O(\min(d_x,d_e) \cdot d_x d_e)$, which can be expensive, especially if performed at every training step for large matrices. However, practical methods exist to approximate or constrain singular values without full SVD. Controlling $\sigma_{\max}(W_{dec})$, the spectral norm, is relatively efficient; techniques like the Power Method are used in Spectral Normalization~\citep{Miyato2018SpectralNorm} to approximate this value iteratively. Controlling $\sigma_{\min}(W_{dec})$ or the condition number $\kappa(W_{dec}) = \sigma_{\max}(W_{dec})/\sigma_{\min}(W_{dec})$ is generally more challenging. Nevertheless, for decoder matrices, whose dimensions $d_x$ and $d_e$ are often manageable (hundreds or low thousands) compared to other parts of a model, the computational cost may be acceptable or can be mitigated by techniques like singular value clamping or projection-based updates. This approach ensures that the primary representational learning burden resides with the encoder $f(\cdot)$, while the well-conditioned linear decoder $W_{dec}$ serves as a predictable mapping that prevents trivial inversions or representational collapse in the output layer.

\subsection{Hypothesis: Output Constraints and Encoder Regularity}
\label{sec:implicit-lip}

Explicitly imposing Lipschitz or bi-Lipschitz constraints layer-by-layer throughout a deep neural network is computationally expensive, often requiring per-layer spectral norm calculations~\citep{Miyato2018SpectralNorm, Cisse2017Parseval}. We treat output-side conditioning as a practical design hypothesis: constraining the feature decoder (§\ref{sec:condition-number}) and using a Lipschitz-controlled graph pooling operation (§\ref{sec:pooling-lipschitz}, §\ref{sec:softmax-attn-pool}) may encourage more stable encoder outputs without requiring layer-wise constraints throughout the network.

This is not a formal guarantee on the encoder. Rather, it motivates the architecture: bounded decoder and pooling stages make the final representation-to-loss map less sensitive, while leaving the encoder expressive enough to learn the data-dependent interactions required by the task. Directly measuring how such output-side constraints affect encoder spectra and optimization dynamics is left for future work.

\subsection{Discussion and Practical Considerations}
\label{sec:discussion}

The choice of pooling operator itself presents several options, each with distinct characteristics. Foundational Lipschitz operators such as mean, sum, and max pooling (detailed in §\ref{sec:pooling-lipschitz}) offer a starting point. Empirically, mean pooling is simple and stable, and it avoids the
scaling issues of sum pooling when $|V(G)|$ varies widely. 
Sum pooling precisely captures node multiplicities (beneficial if
the number of nodes is meaningful), but can become unbounded with large graphs.
Max pooling discards duplicates, effectively capturing the “set” of
embeddings. Each of these simpler methods has advantages depending on how one wishes to
represent the entire graph and whether exact cardinality or
distributional cues are prioritized. See \citep{XuEtAl2019} for a deeper discussion
comparing sum, mean, and max aggregator behaviors in Graph Neural Networks. In MABLE, reconstruction constrains the low-dimensional feature embedding $e_u=\pi(z_u)$, while pooling operates on the full node embedding $z_u$; the complementary contextual slice can therefore contribute to graph-level representations while also being regularized by cross-view matched-node correspondence.

In summary, while these simpler Lipschitz operators (mean, sum, or max) provide robust baseline aggregation, more adaptive mechanisms such as Softmax Attention Pooling (§\ref{sec:softmax-attn-pool}) often yield superior flexibility and empirical performance, especially for graphs with complex or heterogeneous structures. Softmax Attention Pooling, as discussed, can be constructed to maintain overall Lipschitz continuity on bounded embedding domains when the score network uses normalized inputs, upper spectral-norm control, and clipped logits. In our implementation, the final pooling projection is additionally conditioned with both lower and upper singular-value bounds, yielding the partial co-Lipschitz guarantee discussed above.

Furthermore, practical implementations often involve choices about processing the pooled vector before its use in downstream tasks or loss functions. For instance, when angle-based similarity measures are used (such as cosine similarity $\cos(\cdot,\cdot)$ employed in $\mathcal{L}_{\mathrm{graph\_align}}$ §\ref{sec:graph-align}, or normalized SCL variants), the final loss signal primarily depends on embedding direction, not magnitude. However, the magnitude of the pooled vectors can still affect numerical stability and optimization dynamics. Applying a bounded activation function like element-wise $\tanh$ directly after the primary aggregation step (e.g., $s_G = \tanh\!\bigl(P_\phi(\{z_u\}_{u\in V(G)})\bigr)\in\mathbb{R}^{d_G}$) can mitigate these issues by constraining the components of the pre-normalized vectors to the range $(-1, 1)$.
This ensures inputs to the normalization step are well-scaled, potentially improving stability and removing an unconstrained degree of freedom from the optimization. While introducing saturation (which might compress informative large magnitudes), this approach provides component-wise regularization, maintains the overall Lipschitz nature of the graph embedding function (as $\tanh$ is 1-Lipschitz), and produces outputs more readily adaptable for efficient similarity search via binarization techniques. As such, it represents a viable practical strategy within the MABLE toolkit, complementing the choice of pooling operator and decoder constraints.

The considerations for condition-number control on a linear feature decoder $g(e)=W_{dec}e+b$ (§\ref{sec:condition-number}) also involve a conceptual balancing act when setting bounds for its singular values: the minimum $m = \sigma_{\min}(W_{dec})$ and maximum $L = \sigma_{\max}(W_{dec})$. The objective is a decoder that is robustly bi-Lipschitz (requiring $m > 0$ and a finite $L$) while remaining sufficiently expressive for accurate feature reconstruction. In practical terms, bounding $L$ prevents the decoder from hallucinating large feature variances from small latent perturbations, while enforcing a strictly positive $m$ prevents the decoder from ignoring the latent space entirely (a failure mode akin to posterior collapse). Smaller values for $L$, particularly approaching $1.0$, are theoretically appealing as they can enhance training stability by ensuring the decoder is non-expansive or only mildly expansive, fostering more regularized mappings. However, an overly restrictive $L$ might hypothetically curtail the decoder's ability to reconstruct features if significant scaling or transformation from the latent space is necessary, suggesting a need for adequate "slack" or expressiveness. Simultaneously, $m$ must be strictly positive to uphold the bi-Lipschitz property, ensuring distinct latent states can lead to separable reconstructions. The ratio $L/m$ (the condition number) dictates the decoder's allowance for anisotropic distortion; a larger condition number implies more "slack" but less uniform geometric preservation, whereas a smaller condition number (larger $m$ relative to $L$) makes the decoder more rigid, strengthening the geometric coupling outlined in Proposition~\ref{prop:lipschitz}. Thus, the selection of $m$ and $L$ involves a hypothetical trade-off between decoder flexibility for reconstruction, the stability afforded by tighter singular value control, and the desired strength of the geometric link between the latent and feature spaces, with optimal ranges likely being sensitive to specific dataset characteristics and modeling goals.

\section{End-to-End Training Objectives for Node \& Graph Embeddings}
\label{sec:end-to-end-holistic-training}

We now present the end-to-end objectives for learning both node-level and graph-level representations. By training \emph{end-to-end}, the model enforces similarity among similar nodes, encourages consistency of matched nodes across augmented views, aggregates those nodes into a graph-level embedding, and encourages graph-level uniformity among unpaired embeddings. Our system uses in-batch sampling for node embeddings and leverages augmented/un-augmented pairs at the graph level.

\medskip
\noindent
Drawing inspiration from fixed-cosine surrogate principles (alignment, uniformity, avoiding learned discriminators and the InfoNCE objective itself), our framework employs loss functions instantiated differently across scales. At the node level, we combine a masked reconstruction loss with a bi-Lipschitz feature decoder to promote alignment in the reconstruction-relevant feature slice, an explicit cross-view node-correspondence loss to promote consistency in the complementary contextual slice, and a node-level uniformity penalty. At the graph level, we enforce both alignment of augmented graph views and uniformity among unpaired graph embeddings using penalties analogous to those derived for the fixed-cosine surrogate. This design provides a scalable, geometrically interpretable implementation of contrastive principles, adapted to masked graph data.

\subsection{Node Alignment, Node Correspondence, and Node Uniformity}
We start with a \emph{masked} reconstruction objective that encourages the learned representation to capture meaningful dependencies within and across attributes, and to align similar nodes in the reconstruction-relevant embedding space. This loss penalizes the model's ability to accurately reconstruct the original or masked features of a node from its latent embedding:
\[
 \mathcal{L}_{\mathrm{node\_align}}
 =
 \mathbb{E}_{G\sim\mathcal{D}}
 \ \mathbb{E}_{T\sim\mathcal{A}}
 \ \mathbb{E}_{v\sim V(G^{(a)})}
 \Bigl[\ell_{\mathrm{recon}}\bigl(g(\pi(f_\theta(G^{(a)}, X^{(a)})_v)), x_{\iota^{(a)}(v)}; T\bigr)\Bigr],
 \]

where $\mathcal{D}$ is the training distribution over graphs, $T\sim\mathcal{A}$ is the view operator from §\ref{sec:node-similarity}, and $(G^{(a)},X^{(a)},\iota^{(a)})=T(G,X)$.
The encoder $f_\theta$ consumes the augmented input $(G^{(a)},X^{(a)})$ and outputs a contextualized embedding for each retained node $v\in V(G^{(a)})$, while $x_{\iota^{(a)}(v)}$ is the original reconstruction target associated with that retained node identity.
The slice $\pi$ selects the feature-embedding subspace used for reconstruction and $g$ is the (bi-Lipschitz) feature decoder (§\ref{sec:condition-number}).
The reconstruction loss $\ell_{\mathrm{recon}}(\cdot,\cdot;T)$ is a norm-based penalty in feature space and may depend on the view mask (e.g., evaluated only on masked components, only on unmasked components, or on both with feature-type weights).

Masking encourages the encoder to learn feature embeddings that support consistent feature prediction from context for nodes with similar local structure or attributes, thereby promoting grouping in the reconstruction-relevant feature-embedding space. Importantly, due to the bi-Lipschitz nature of $g$ (see Proposition~\ref{prop:lipschitz}), minimizing reconstruction error acts as a geometric regularizer in the learned feature-embedding space. If two nodes have similar reconstruction targets and both are reconstructed accurately, then their feature-embedding slices are forced to be close up to the reconstruction-error slack in Proposition~\ref{prop:lipschitz}. Without the lower-Lipschitz component of the decoder constraint, accurate reconstruction alone would not rule out arbitrarily arranged latent codes.

\medskip
\noindent
\paragraph{Cross-View Node Correspondence in the Contextual Slice.}
Reconstruction constrains only the feature slice $e_u=\pi(z_u)$. To encourage the remaining contextual slice $c_u=\rho(z_u)$ to encode view-stable information rather than augmentation-specific variation, we additionally align matched nodes across two augmented views of the same graph, in the spirit of bootstrap-style cross-view consistency methods such as BGRL~\citep{Thakoor2021BGRL}. Let $(G^{(a)},X^{(a)},\iota^{(a)})=T_a(G,X)$ for $a\in\{1,2\}$ denote two masked/augmented views of the same underlying graph $G$. We define the common retained identity set
\[
\mathcal{C}(G,T_1,T_2)
:=
\iota^{(1)}(V(G^{(1)}))\cap \iota^{(2)}(V(G^{(2)})) \subseteq V(G).
\]
For $u\in\mathcal{C}(G,T_1,T_2)$, write $v_a(u):=(\iota^{(a)})^{-1}(u)$ for the corresponding node in view $a$.\footnote{In practice, $\mathcal{C}(G,T_1,T_2)$ is the set of nodes that survive the two augmentations and can be aligned by identity or inherited indexing; for efficiency we may evaluate the loss on a random subset of $\mathcal{C}(G,T_1,T_2)$ in each minibatch.}
We define
\[
\mathcal{L}_{\mathrm{node\_corr}}
\;=\;
\mathbb{E}_{G\sim\mathcal{D}}
\ \mathbb{E}_{T_1,T_2\sim\mathcal{A}}
\ \mathbb{E}_{u\sim \mathcal{C}(G,T_1,T_2)}
\Bigl[
  \ell_{\mathrm{corr}}\!\bigl(
    \rho(f_\theta(G^{(1)},X^{(1)})_{v_1(u)}),\,
    \rho(f_\theta(G^{(2)},X^{(2)})_{v_2(u)})
  \bigr)
\Bigr],
\]
where $\ell_{\mathrm{corr}}$ is a similarity-matching loss on the contextual slice (e.g., squared $\ell_2$ distance between normalized contextual vectors, or equivalently a cosine-alignment penalty up to constants). This term encourages matched nodes across augmentations to remain close in the contextual subspace while leaving the reconstruction-relevant feature slice to be shaped primarily by the masked autoencoding objective.

\medskip
\noindent
\paragraph{In-Batch Node Sampling for Node Uniformity.}
We then promote uniformity among node embeddings by randomly pairing node instances \emph{within the same batch} and penalizing the squared cosine similarity between embeddings of unpaired nodes. This term is analogous to the negative term in the Spectral Contrastive Loss (SCL) and the log-based surrogate, which encourages orthogonality among negative pairs:
\[
\mathcal{L}_{\mathrm{node\_uniform}}
\;=\;
\mathbb{E}_{\eta,\eta'\sim \mathcal{Z}_{\mathrm{node}},\ \eta\neq \eta'}
\Bigl[
  \cos(\eta,\eta')^2
\Bigr].
\]
Here $\mathcal{Z}_{\mathrm{node}}$ denotes the current minibatch pool of valid retained node embeddings, distinct from the node-instance set $\mathcal{B}_{\mathrm{node}}$ defined above. In our implementation, $\mathcal{Z}_{\mathrm{node}}$ contains token embeddings from both augmented views of all graphs in the minibatch, sampled uniformly from this pooled token set. Exact cross-view matches of the same retained node identity are rare under this uniform sampling in large token pools, so we use this term as a lightweight isotropy/decorrelation regularizer rather than as a strict semantic-negative objective. This objective penalizes large absolute dot products between sampled token embeddings, pushing their cosine similarity towards orthogonality (zero).

\subsection{Graph-Level Pooling and Positive Pair Augmentation}
\label{sec:graph-align}
Next, we consider the graph-level embedding $s_G$. For a given view $(G^{(a)},X^{(a)},\iota^{(a)})$, this embedding is typically generated using the Adaptive Graph Pooling with Softmax Attention mechanism ($P_{\mathrm{attn}}$), as detailed in §\ref{sec:softmax-attn-pool}, applied to the set of retained node embeddings:
\[
s_G^{(a)} = P_{\mathrm{attn}}\bigl(\{f_\theta(G^{(a)},X^{(a)})_v:\, v \in V(G^{(a)})\}\bigr),
\]
yielding a single vector per graph view. This adaptive mechanism is often favored for its flexibility in handling varied graph structures and for its potential empirical performance benefits. While simpler Lipschitz operators like $\mathrm{MeanPool}$ (cf.~§\ref{sec:pooling-lipschitz}) can serve as alternatives or for comparative studies, any chosen Lipschitz pooling operator helps control sensitivity to changes in graph size or node embedding scale. However, pooling alone is insufficient to guarantee invariance to structured perturbations such as node masking, feature noise, or irregular subgraph sampling. The role of positive pair augmentation is thus critical: by explicitly training the encoder and pooling to align perturbed views of the same graph, the model learns to ignore incidental variations and focus on stable, semantically meaningful patterns. We create positive graph-view pairs by sampling $T_1,T_2\sim\mathcal{A}$ from the same underlying $(G,X)$, e.g.:
\begin{enumerate}
    \item \textbf{Augmented copy:} a view might keep the same node set but apply partial node masking, node dropping, or slight feature corruption.
    \item \textbf{Subgraph sampling:} If $G$ is large, we can sample subgraphs (e.g., via node dropping) to form augmented views, treating the embedding of a subgraph as a positive view of the entire graph, provided they remain semantically related.
\end{enumerate}

Using a penalty based on cosine similarity (for nonzero vectors $a,b$, writing $\hat a := a/\|a\|_2$ and $\hat b := b/\|b\|_2$, so that the usual identity $1-\cos(a,b)=\tfrac12\|\hat a-\hat b\|_2^2$ holds), we define the inter-graph alignment loss:
\[
\mathcal{L}_{\mathrm{graph\_align}} = \mathbb{E}_{G\sim\mathcal{D}}\ \mathbb{E}_{T_1,T_2\sim\mathcal{A}}\left[ 1 - \cos\!\bigl(s_G^{(1)}, s_G^{(2)}\bigr) \right].
\]
Here $s_G^{(a)}$ is computed from the view $(G^{(a)},X^{(a)},\iota^{(a)})=T_a(G,X)$ as above.

Minimizing this loss is equivalent to maximizing the cosine similarity between the two view embeddings. When one view is a subgraph of the other (a possible outcome of our augmentation strategy, as previously described), this alignment objective specifically encourages the larger-view embedding to be predictive of the representation of its own local (subgraph) structure. Such an alignment between global and local views of the same data instance is a core tenet of mutual information maximization strategies like Deep InfoMax~\citep{Hjelm2018DeepInfoMax, Velickovic2019DeepGraphInfomax}. This alignment is particularly important in geospatial and geological datasets, where differences in node count or spatial density often reflect sampling artifacts rather than true semantic divergence. Together, the pooling and augmentation strategy is designed to improve robustness to cardinality mismatches, topological edits, and local feature inconsistencies.

\subsection{Graph Uniformity}
Lastly, to promote inter-graph uniformity, for an anchor graph identity $G$, a view embedding $s_G^{(a)}$ is pushed away from embeddings of views of other underlying graphs sampled from the current batch. Concretely, for a minibatch $\mathcal{B}_{\mathrm{graph}}$ of graphs and their augmented views, the anchor identity is the underlying graph $G$; its own views include all embeddings $s_G^{(a)}$ derived from that same graph. We take $\mathrm{NegPool}(G)$ to be the empirical distribution over all graph-level view embeddings in the batch excluding those views of $G$ itself. In particular, it may comprise full-graph or subgraph view embeddings derived from other graphs $H\ne G$. Analogous to the SCL negative term, we penalize the squared cosine similarity:
\[
\mathcal{L}_{\mathrm{graph\_uniform}} = \mathbb{E}_{G,\,T_a,\, Z^- \in \mathrm{NegPool}(G)}\bigl[ \cos\!\bigl(s_G^{(a)}, Z^-\bigr)^2 \bigr].
\]
Here, for each anchor graph $G$, $s_G^{(a)}$ is computed from a sampled anchor view $T_a(G,X)$, and $Z^-$ is sampled from the empirical distribution that is uniform over elements of $\mathrm{NegPool}(G)$. The importance of such explicit uniformity-promoting terms for enhancing feature diversity and preventing issues like representation collapse has also been highlighted in masked autoencoder frameworks such as U-MAE~\citep{Zhang2022UMAE}. In practice, these unpaired in-batch graphs are a heuristic source of contrast rather than guaranteed semantic negatives. When $Z^-$ includes subgraph view embeddings from another graph $H$, this objective contributes to making views of $G$ discriminable from local views of other graphs, complementing the positive local-global alignment encouraged by $\mathcal{L}_{\mathrm{graph\_align}}$.

\subsection{Combined Objective}
Putting it all together, we train \emph{end-to-end} with a single combined objective:
\[
\mathcal{L}
\;=\;
\alpha\,\mathcal{L}_{\mathrm{node\_align}}
\;+\;
\beta\,\mathcal{L}_{\mathrm{node\_corr}}
\;+\;
\gamma\,\mathcal{L}_{\mathrm{node\_uniform}}
\;+\;
\lambda\,\mathcal{L}_{\mathrm{graph\_align}}
\;+\;
\mu\,\mathcal{L}_{\mathrm{graph\_uniform}}.
\]

By tuning the above scalars, we balance local (node) alignment via reconstruction, cross-view node correspondence in the contextual slice, node-level uniformity, graph-level augmentation alignment, and graph uniformity across the dataset.

\subsection{Holistic Benefits}
This fully integrated pipeline yields robust node and graph embeddings with the following advantages:
\begin{itemize}
\item \textbf{Principled Surrogates:} By using fixed-cosine based penalties analogous to normalized SCL, we avoid learned discriminators and their associated stability and bias issues. The normalized/cosine SCL-style surrogate also admits a heuristic second-order connection to a DV-style mutual-information objective under boundedness, approximate symmetry, and small-variance assumptions (\Cref{appendix:mi-bound}).
\item \textbf{Local Quality:} Masked node autoencoding plus bi-Lipschitz decoding encourages the \emph{feature-embedding slice} used for reconstruction to reflect local similarity: similar node attributes lead to nearby feature embeddings up to the reconstruction-error slack in Proposition~\ref{prop:lipschitz}. A complementary node-correspondence loss encourages the remaining contextual slice to stay consistent across matched nodes in augmented views. The node uniformity penalty shapes the full node embeddings by encouraging unpaired in-batch node embeddings to stay near-orthogonal. Together, these terms encourage node representations that are locally faithful, cross-view stable, and well spread.
\item \textbf{Reliable Graph Embeddings:} The use of Lipschitz-controlled pooling mechanisms, particularly adaptive Softmax Attention Pooling (§\ref{sec:softmax-attn-pool}) as well as simpler operators like mean or sum (cf.~\Cref{prop:sum-mean-max-lipschitz}), ensures that small changes in node embeddings lead to proportionally bounded changes in $s_G$. Furthermore, the graph alignment penalty ($\mathcal{L}_{\mathrm{graph\_align}}$) contributes to the robustness of these embeddings against augmentation variations.
\item \textbf{Global Uniformity:} By drawing data from the same batch or dataset, both at node and graph level, the uniformity penalties encourage a well-spread, approximately isotropic embedding distribution across the collection of graphs, complementing local alignment and helping prevent representation collapse.
\item \textbf{Implicit Deep InfoMax-Style Learning:} By leveraging subgraph augmentations for positive pairs in $\mathcal{L}_{\mathrm{graph\_align}}$ (aligning two views of the same underlying graph, potentially including full-graph and local-subgraph views) and employing a broad in-batch contrastive strategy for $\mathcal{L}_{\mathrm{graph\_uniform}}$ (contrasting views of one graph against full-graph or subgraph views of other graphs), MABLE encourages a learning dynamic similar to Deep InfoMax. It learns global representations that are predictive of their own local structures while being distinct from the local structures of other graphs, all achieved without a learned discriminator due to the SCL-based fixed-cosine surrogates.
\item \textbf{End-to-End Efficiency:} Because node- and graph-level objectives are co-trained, we do not need separate staging or complex hard-negative selection steps. In-batch sampling, straightforward augmentation, and sampled matched-node correspondence provide the necessary positive and regularizing signals within a single training loop.
\item \textbf{Output-Side Stability:} Constraining the bi-Lipschitz decoder and Lipschitz pooling layers provides controlled final-stage mappings from embeddings to reconstruction and graph-level objectives, while leaving encoder-level regularity as a design hypothesis rather than a formal guarantee.
\item \textbf{Conditioned Decoder Control:} If $g$ is a single linear layer with constrained singular values (§\ref{sec:condition-number}), then the final reconstruction map is protected against trivial collapse or unbounded expansion, reducing the risk of poorly conditioned latent-to-feature mappings.
\end{itemize}

Overall, this holistic design encourages similarity, correspondence, and uniformity at both the node and graph levels in one end-to-end pass. Crucially, by leveraging principled fixed-cosine surrogates, simple augmentations, in-batch sampling, sampled matched-node correspondence, and a well-conditioned decoder, the burden of representation is shifted to the encoder, which must learn meaningful and robust features. This principle is strongly supported by findings like MIM-Refiner~\citep{AlkinEtAl2025MIM}, which show that lightweight decoders lead to more discriminative intermediate representations, ultimately improving performance in downstream tasks.

\section{Experiments}
\label{sec:experiments}

We evaluate MABLE at two spatial scales (prospect and regional; details deferred to subsequent subsections). Here we first describe the shared backbone and training procedure used across both experiments, including masking, augmentation, and the instantiation of the losses introduced in §\ref{sec:end-to-end-holistic-training}.

\subsection{GeoReformer backbone, tokenization, and training procedure}
\label{sec:georeformer-exp}

\paragraph{Backbone (GeoReformer).}
All experiments use a Reformer-style encoder~\citep{Kitaev2020} with LSH self-attention, adapted to geospatial node sets and referred to as the \textbf{GeoReformer}. GeoReformer is the encoder instantiation used in our experiments; the MABLE objectives only require an encoder that maps a graph or node collection to node embeddings, so other GNN, graph-transformer, or set-transformer backbones could be substituted. Each graph is encoded as a padded sequence of node tokens and processed by a stack of LSH attention blocks to support thousands of nodes per sample. We remove sequence/axial positional encodings because node order is arbitrary.

\paragraph{Design rationale: edge-free graphs via induced connectivity.}
We represent each sample as an \emph{edge-free} set of node observations and avoid constructing an explicit adjacency matrix. In mineral prospectivity settings, heuristic edges (e.g., $k$-NN or radius graphs) can inject strong, dataset-specific priors and become brittle under heterogeneous sampling densities and multi-sensor geometry \citep{Sihombing2024Improved,Daruna2024GFM4MPM, Zuo2023NewGeneration}. Instead, we let the model \emph{induce} an example-specific sparse connectivity pattern through its attention mechanism: the Reformer-style LSH routing performs a hard partition of tokens into candidate neighborhoods, yielding a block-sparse attention graph that scales to tens of thousands of nodes while remaining agnostic to arbitrary node ordering. This shifts the burden of discovering useful spatial interactions from a hand-crafted graph construction rule to learned attention over data-dependent neighborhoods.

\paragraph{Node tokenization and input projections.}
Each node token concatenates (i) a feature-type identifier \(\tau_u\in\{1,\dots,K\}\) (stored as a one-hot vector over \(K\) feature IDs, and optionally mapped to a fixed unit-norm code vector \(q(\tau_u)\in\mathbb{R}^{K}\) for training stability), (ii) an optional scalar feature value \(r_u\) (or a sentinel when not applicable), and (iii) relative geometric features (e.g., coordinates expressed relative to the graph center, with dataset-specific scaling). The resulting encoder token augments the reconstructible feature vector with relative geometric information. Concretely, the reconstructible feature vector \(x_u\) comprises the feature-type identifier together with the scalar value when present, while relative geometry is provided to the encoder as auxiliary positional/context information concatenated to this feature vector. The full token is then projected into the model hidden dimension \(d\) by a learned linear embedder, \(t_u = E(\cdot) \in \mathbb{R}^{d}\), providing a minimal and modality-agnostic input interface for heterogeneous geospatial observations.

\paragraph{Relative geometry only; no absolute spatial identity.}
To avoid trivial location-based shortcuts and to improve transfer across regions, we do not inject absolute or global position embeddings into the sequence. Geometric information enters only through \emph{relative} positional features and through augmentation, which preserves the local spatial relationships while discouraging memorization of absolute coordinates.

\paragraph{LSH self-attention and routing regularization.}
GeoReformer uses LSH self-attention to approximate full self-attention with sub-quadratic cost. Concretely, tokens are assigned to buckets by hashing projected representations; these bucket assignments define the \emph{support} of the attention computation (i.e., which token pairs are eligible to interact) and thus act as a hard, example-specific sparsification of the attention adjacency. During training, we regularize this routing by sampling bucket assignments from a temperature-scaled distribution over the hash projection scores, rather than using a deterministic argmax. This stochastic routing reduces brittleness of the induced block structure and improves robustness. At evaluation time, routing reverts to deterministic argmax assignment, recovering the standard Reformer behavior.

\paragraph{Masking with feature-type specific mask tokens.}
We use a feature-type aware masking scheme with a dedicated learned mask token (a trainable replacement vector) per feature type. When a node is masked, we keep its relative geometric features unchanged but replace both the feature-ID block and the scalar value (when present) with the corresponding mask token. In this experimental tokenization, the model is trained to reconstruct the original feature-type identifier (represented either as a one-hot vector or its fixed code \(q(\tau_u)\)) together with the scalar value for feature types that admit one; MABLE itself only requires a reconstructible target appropriate to the domain. For feature types without an associated scalar value, the scalar reconstruction term is omitted. Because the mask token is feature-type specific by design, the nontrivial learning signal comes primarily from reconstructing masked values/content from cross-feature context and relative spatial structure, rather than from inferring the feature type itself.

\paragraph{Two-view augmentation (mask-only vs drop--rotate--mask).}
Training forms positive view pairs from the same underlying graph. The \emph{original} view applies node masking only. The \emph{positive} view combines (i) stochastic node dropping, (ii) random Z-axis rotations around a local reference point (e.g., the graph center), and (iii) masking. Dropping is applied both \emph{globally} and in \emph{spatial patches}: with some probability we sample a patch (random center and radius) and remove nodes within that region. Patch dropping can be \emph{uniform} (the entire patch is dropped) or \emph{feature-conditional}, where each feature type is independently retained or removed with a sampled probability, effectively simulating missing modalities or incomplete acquisition within localized areas. A minimum-retained-node constraint avoids degenerate views under aggressive dropping. Rotation augmentation is used specifically to encourage the learned representations, especially graph-level embeddings, to be insensitive to local orientation, i.e., to favor rotational invariance in the embedding space. This is desirable in the mineral-exploration setting considered here, where local patch orientation is typically not itself the quantity of interest, whereas relative spatial configuration and cross-modal context are.

\paragraph{Decoder conditioning and bounded attention scores.}
We use a bi-Lipschitz feature decoder for reconstruction by clamping the singular values of the decoder matrix to $[\sigma_{\min}, \sigma_{\max}]$. For graph pooling, we use standard softmax attention: the attention MLP receives unit-normalized node embeddings, its internal linear maps are clamped only in \emph{maximum} singular value, and the pre-softmax logits are shifted by their within-graph maximum and clipped to $[-c,0]$, thereby bounding relative logit gaps in a shift-invariant way. The final pooling projection $W_{pool}$ is separately conditioned by clamping its singular values to $[m,L]$. Thus, the score network is controlled for upper-Lipschitz stability, while the output projection additionally provides the partial co-Lipschitz guarantee stated in §\ref{sec:softmax-attn-pool}.

\paragraph{Losses and optimization.}
The training objective combines reconstruction with representation-shaping terms: node-level correspondence and uniformity, graph-level alignment between two views of the same underlying graph, and graph-level uniformity within a batch. Throughout, the decoder is trained to predict the full node reconstruction target $x_u$.\footnote{We use $x_u$ to denote the reconstruction target; the encoder input is formed by masking/corrupting the reconstructible components of $x_u$ (where selected components are replaced by learned mask tokens) and concatenating the resulting vector with the relative geometric context.}
Masking is applied to the \emph{encoder input} by corrupting a subset of feature dimensions/types with learned mask tokens, so that predicting the corrupted entries is inherently context-dependent and does not admit trivial copy-through shortcuts. In training we use (i) a \emph{masked} reconstruction loss, evaluated on the masked entries, and (ii) an auxiliary reconstruction loss on the unmasked entries; together these form a weighted reconstruction loss on the full vector $x_u$, with weights chosen as a task- and modality-dependent implementation choice. We use an $\ell_2$-style reconstruction loss (MSE) on all reconstructed components, including the feature-ID representation (one-hot or fixed code) and the scalar value when present. This keeps the reconstruction target in a common Euclidean space consistent with the geometric role of the decoder. In addition, for two augmented views of the same graph we apply a node-correspondence loss on the \emph{contextual slice} of matched nodes retained in both views; for efficiency and to avoid over-constraining the representation, this term is evaluated on a random subset of valid matched nodes per graph in each minibatch. Uniformity terms use squared cosine similarity surrogates, and node-level uniformity is likewise estimated by sampling valid tokens (optionally from both views) for efficiency. Optimization uses Adam with a cosine learning-rate schedule and linear warmup; some experiment-specific hyperparameters are reported with each experimental setting.

\paragraph{Implementation-detail availability.}
The experiments were run in a proprietary commercial environment. We report the architectural choices, data splits, embedding dimensions, downstream feature construction, and evaluation protocols needed to interpret the results, but most internal low-level implementation details and training hyperparameters are not fully disclosed here. All reported comparisons use the same internal training configuration unless an ablation explicitly states otherwise, so the experimental claims should be read as controlled comparisons within that fixed implementation rather than as a complete public reproduction recipe.

\subsection{Experiment: Regional-scale Arabian Shield dataset}
\label{sec:exp2-arabian-shield}

\paragraph{Study area and geological context.}
Our second experiment targets the \emph{Arabian Shield}, a structurally complex Neoproterozoic shield comprising diverse accreted terranes, sutures, and intrusive bodies. This setting provides a large, heterogeneous multi-modal corpus for node and graph-level pretraining at regional scale, where robustness to uneven coverage and strong geological non-stationarity are central.

\paragraph{Data modalities.}
We assemble a multi-sensor stack over the Arabian Shield using:
\begin{itemize}
    \item \textbf{Surface geochemistry:} Saudi Geological Survey (SGS) surface geochemistry.
    \item \textbf{Optical imagery:} Sentinel-2 raw optical channels (Channel~1 removed), with invalid pixels masked using the Sentinel-2 Scene Classification Layer.
    \item \textbf{Geophysics:} Total Magnetic Intensity (TMI) Reduced-to-Pole (RTP), with partial coverage over the Shield.
    \item \textbf{Regional seismic:} KSA seismic $V_S$ data resampled to 600\,m horizontal spacing, truncated to 10\,km depth, and spatially filtered to the Shield footprint.
    \item \textbf{Topography:} DEM data (used for dataset collation but \emph{not} used as an explicit feature).
\end{itemize}

\paragraph{Graph construction (random-patch sampling).}
Graphs are constructed by randomly selecting a horizontal center $(y_0,x_0)$ within the dataset footprint and extracting a fixed-radius patch. Each sample is an \emph{edge-free} graph whose nodes correspond to observations from the available modalities falling inside a 200\,m radius disc in the horizontal plane:
\[
\sqrt{(y-y_0)^2 + (x-x_0)^2} \le 200~\text{m}.
\]
We retain each observation's vertical coordinate $z$ as part of the node token, but do not use $z$ for inclusion, i.e., the receptive field is a cylinder in $(y,x)$ with full retained vertical structure. Node geometry is encoded in a local relative frame centered at $(y_0,x_0)$, consistent with the relative-geometry-only design of §\ref{sec:georeformer-exp}.

\paragraph{Dataset size and structure.}
The resulting dataset contains approximately $145$k graphs and $951$M total nodes. Graphs contain about $6.6$k nodes on average (std.\ $\approx 1.3$k), ranging from roughly $50$ to $8.0$k nodes per graph. Across these modalities there are 71 distinct feature types, which are tokenized using the feature-type identifier scheme described in §\ref{sec:georeformer-exp}.

\paragraph{Train/validation splits.}
We use a 90\% / 10\% split into training and validation graphs for model selection. Subsequent subsections report embedding diagnostics and downstream analyses under this split.

\subsubsection{Study 1: Regional Embedding Structure and Emergent Zonation Patterns}
\label{sec:exp2-embeddings}

We evaluate the structure of the 256-dimensional embeddings learned by MABLE over the Arabian Shield using a complementary set of spatial, statistical, spectral, and graph-based analyses. Although the model is trained without access to absolute/global geographic coordinates or geological labels, these diagnostics indicate that the learned representation contains spatially coherent, qualitatively geologically plausible structure across multiple scales.

\paragraph{Dimension-Level Spatial Structure and Latent Geometry.}
We begin by visualizing individual embedding dimensions geospatially (Fig.~\ref{fig:dim_maps}). Most dimensions exhibit clear spatial coherence, large contiguous patches, gradients, and sharp transitions, rather than unstructured noise. Several dimensions appear to respond to broad terrane-scale domains, while others emphasize finer lithological texture, suggesting meaningful specialization across embedding channels and a genuinely distributed representation.

A complementary geometric sanity check is the cosine-similarity distribution between 5{,}000 randomly sampled embedding pairs (Fig.~\ref{fig:dot_product_distribution}). Many representation learning pipelines suffer from \emph{anisotropy} (a ``cone effect''), where embeddings concentrate around a dominant direction and cosine similarity becomes less discriminative. Here, the near-zero mean ($\mu \approx 0.054$) is consistent with low anisotropy and broad angular coverage, and there is no evident collapse in the learned embedding distribution. The pronounced positive tail, in turn, reflects dense regional domains in latent space: randomly sampled pairs occasionally fall within the same large-scale terrane or lithological regime, yielding high similarity.

\begin{figure}[h]
    \centering
    \includegraphics[height=0.95\textheight]{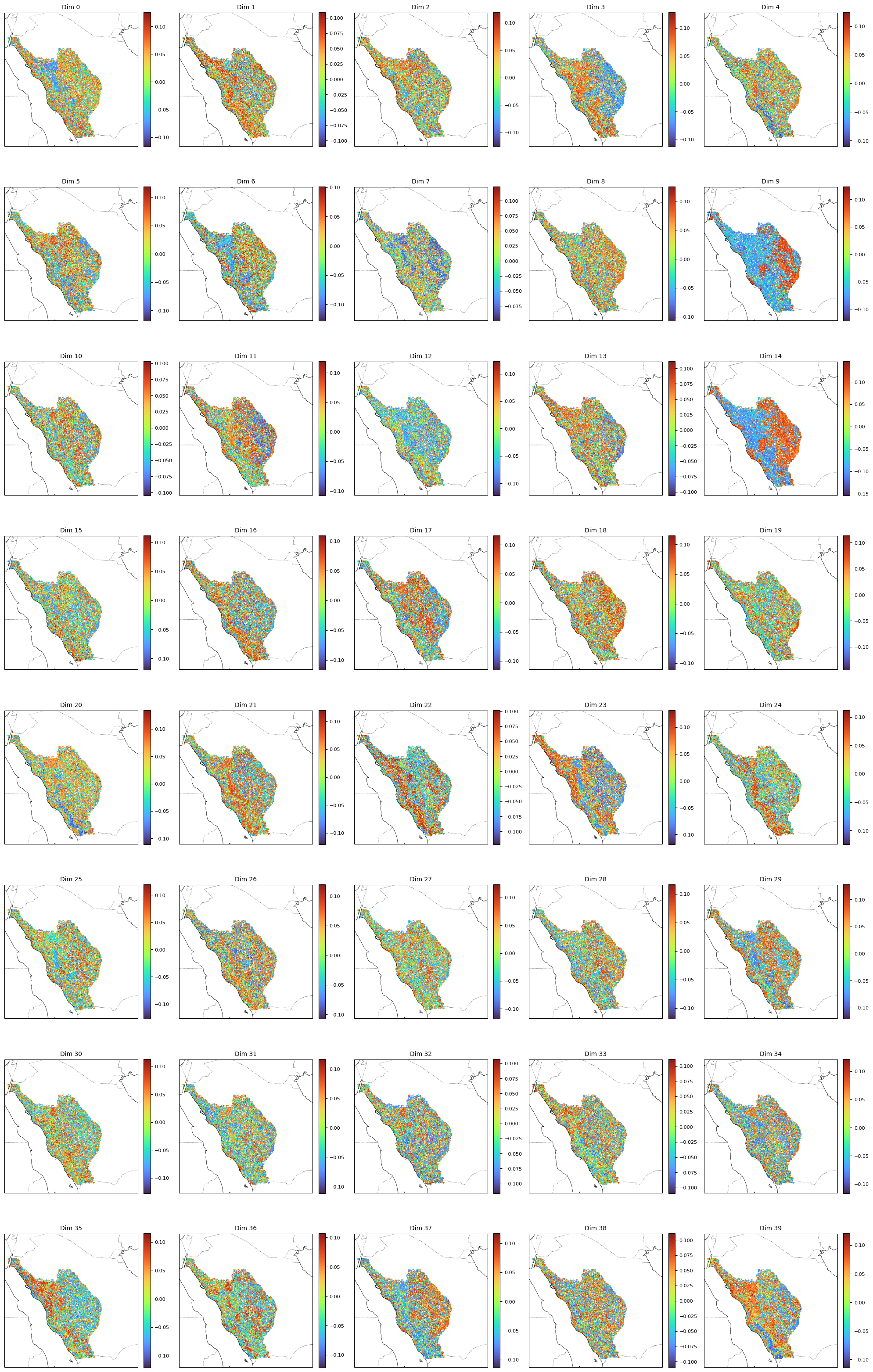}
    \caption{Spatial activation maps for selected Arabian Shield graph-embedding dimensions (Experiment~2). Most dimensions exhibit coherent, non-random spatial structure, consistent with structured geoscience signal.}
    \label{fig:dim_maps}
\end{figure}

\begin{figure}[h]
    \centering
    \includegraphics[width=0.5\textwidth, trim={0 0 0 1.2cm},clip]{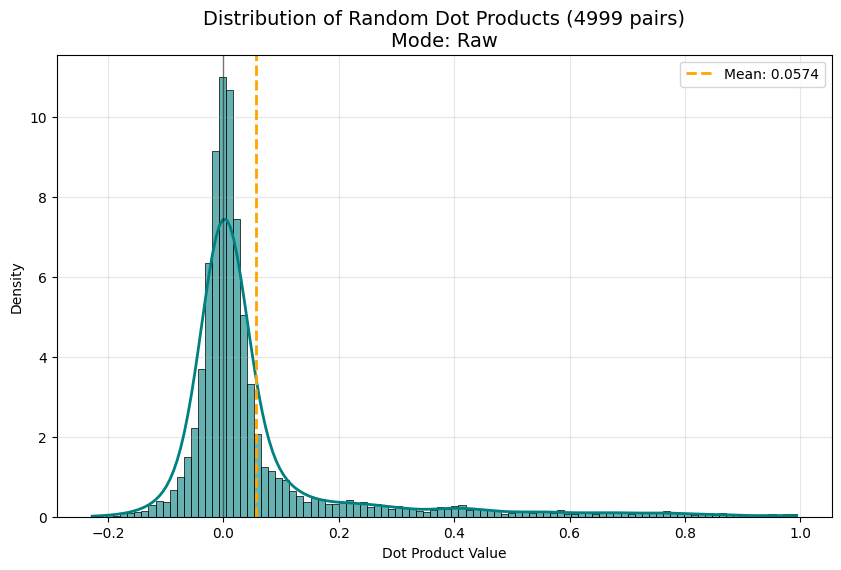}
\caption{
Cosine-similarity distribution for 5{,}000 randomly sampled Arabian Shield graph embedding pairs (Experiment~2, Study 1). The near-zero mean ($\mu \approx 0.054$) is consistent with low anisotropy in the learned embedding distribution, avoiding the common failure mode in self-supervised and contrastive models where representations collapse toward a dominant direction.
The approximately symmetric central mass suggests broad angular coverage of the latent space, while the pronounced long tail reflects dense regional clusters (e.g., terrane-scale domains) encoded without sacrificing global diversity.
}

    \label{fig:dot_product_distribution}
\end{figure}

\paragraph{Manifold Projections Reveal Regional Organization.}
Low-dimensional projections provide a compact view of how the embedding space is organized at regional scale (Fig.~\ref{fig:umap_pca_rgb}). The UMAP RGB map (cosine metric, $n_{\mathrm{neighbors}}{=}30$) exhibits sharply bounded, spatially coherent domains, consistent with terrane-scale segmentation emerging from signal content alone (no absolute/global coordinates were used during training). In contrast, the PCA RGB map captures a broad west--east gradient but appears substantially more diffuse. This combination suggests that the embedding geometry is strongly non-linear: regional structure is present in the representation, but it is not well described by a small number of linear modes. Consistent with this, the first three principal components explain only $\sim$17\% of the variance, indicating that information is distributed across many embedding dimensions rather than concentrated into a few dominant axes.

\begin{figure}[t]
    \centering
    \begin{subfigure}[t]{0.48\textwidth}
        \centering
        \includegraphics[width=\textwidth, trim={0 0 0 1.2cm},clip]{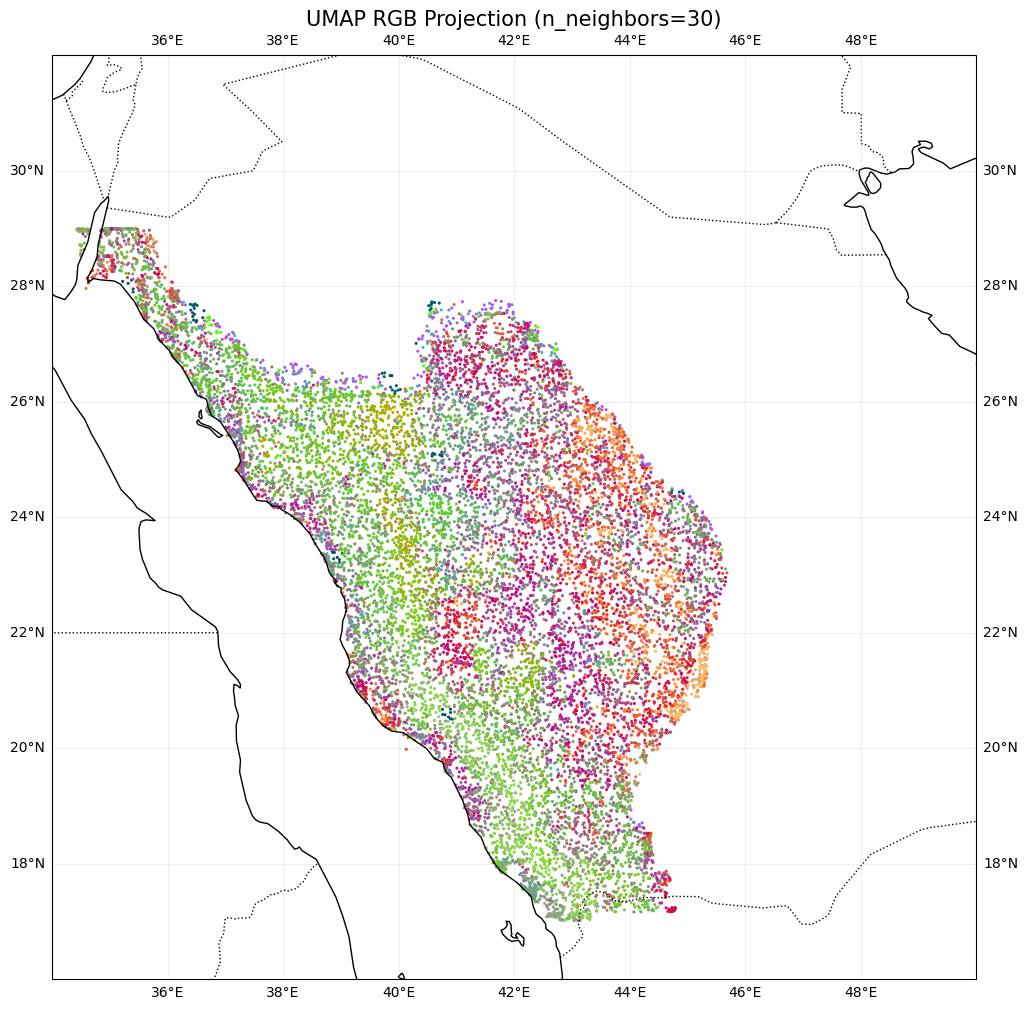}
        \caption{UMAP (cosine), $n_\mathrm{neighbors}=30$.}
        \label{fig:umap_rgb}
    \end{subfigure}\hfill
    \begin{subfigure}[t]{0.48\textwidth}
        \centering
        \includegraphics[width=\textwidth, trim={0 0 0 1.2cm},clip]{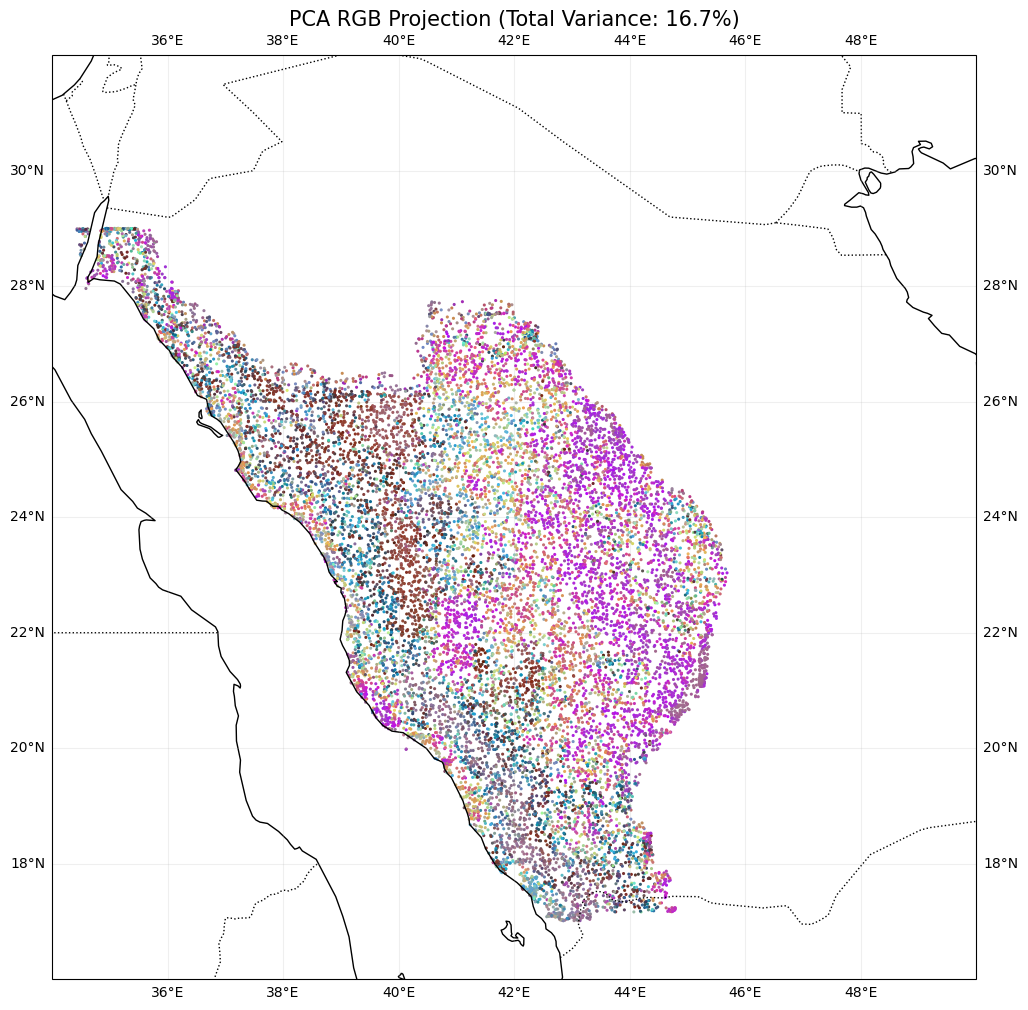}
        \caption{PCA (top 3 components).}
        \label{fig:pca_rgb}
    \end{subfigure}

    \caption{
Low-dimensional projections of the 256-dimensional Arabian Shield graph embeddings (Experiment~2, Study 1), visualized by mapping the first three coordinates of each projection to RGB. UMAP reveals sharply bounded, spatially coherent domains, consistent with terrane-scale segmentation emerging without absolute/global coordinate input during training.
PCA provides a linear baseline that recovers a broad west--east gradient but appears noisier, consistent with variance being distributed across many embedding dimensions rather than concentrated in a few linear modes.}
    \label{fig:umap_pca_rgb}
\end{figure}

\paragraph{Similarity Search Supports Anomalies as Structured Signals.}
To test whether outlier scores reflect repeatable structure rather than preprocessing artifacts, we run a cosine-similarity retrieval over the full validation set using representative \emph{anomalous} and \emph{typical} query embeddings selected by Isolation Forest (global rarity) and Local Outlier Factor (local density). For each query, the similarity-score distribution is sharply concentrated near zero with a heavy right tail (high kurtosis), indicating that high-similarity neighbors are sparse and highly selective. Importantly, the spatial footprint of the top matches differs by query type: anomaly queries concentrate into a small number of compact, geographically coherent patches, whereas normal queries return broader, higher-coverage regions consistent with membership in a dense manifold region. Because ranking uses \emph{only} embedding similarity (coordinates are used solely for plotting), the emergence of contiguous match regions supports the interpretation that the detected anomalies correspond to rare but repeatable, spatially structured signatures in the learned representation (Fig.~\ref{fig:similarity_search}).

\begin{figure*}[t]
    \centering
    \begin{subfigure}[t]{0.49\textwidth}
        \centering
        \includegraphics[width=\textwidth, trim={0 0 0 1.3cm},clip]{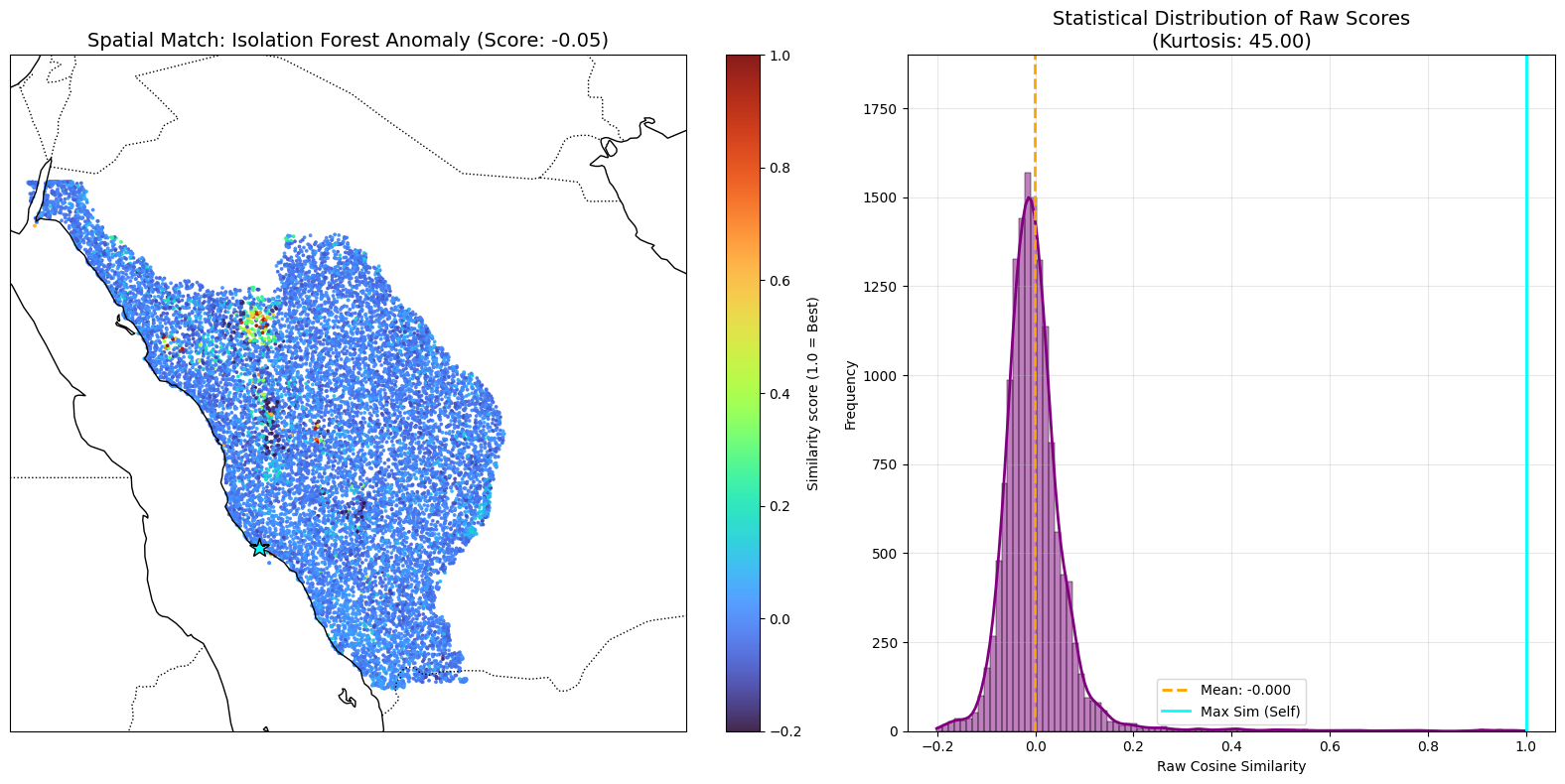}
        \caption{Isolation Forest: anomaly query.}
        \label{fig:sim_if_anom}
    \end{subfigure}\hfill
    \begin{subfigure}[t]{0.49\textwidth}
        \centering
        \includegraphics[width=\textwidth, trim={0 0 0 1.3cm},clip]{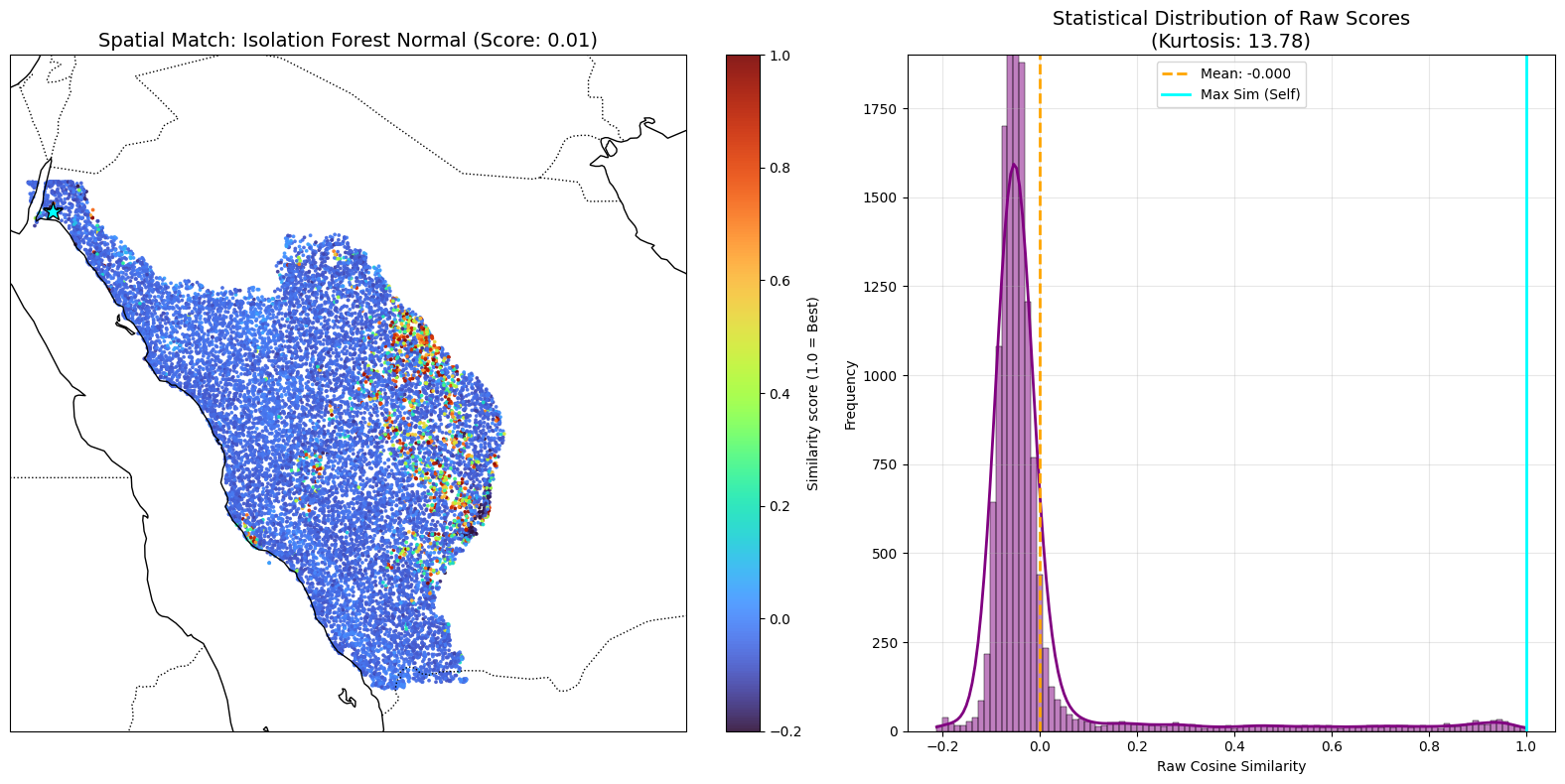}
        \caption{Isolation Forest: normal query.}
        \label{fig:sim_if_norm}
    \end{subfigure}

    \vspace{0.6em}

    \begin{subfigure}[t]{0.49\textwidth}
        \centering
        \includegraphics[width=\textwidth, trim={0 0 0 1.3cm},clip]{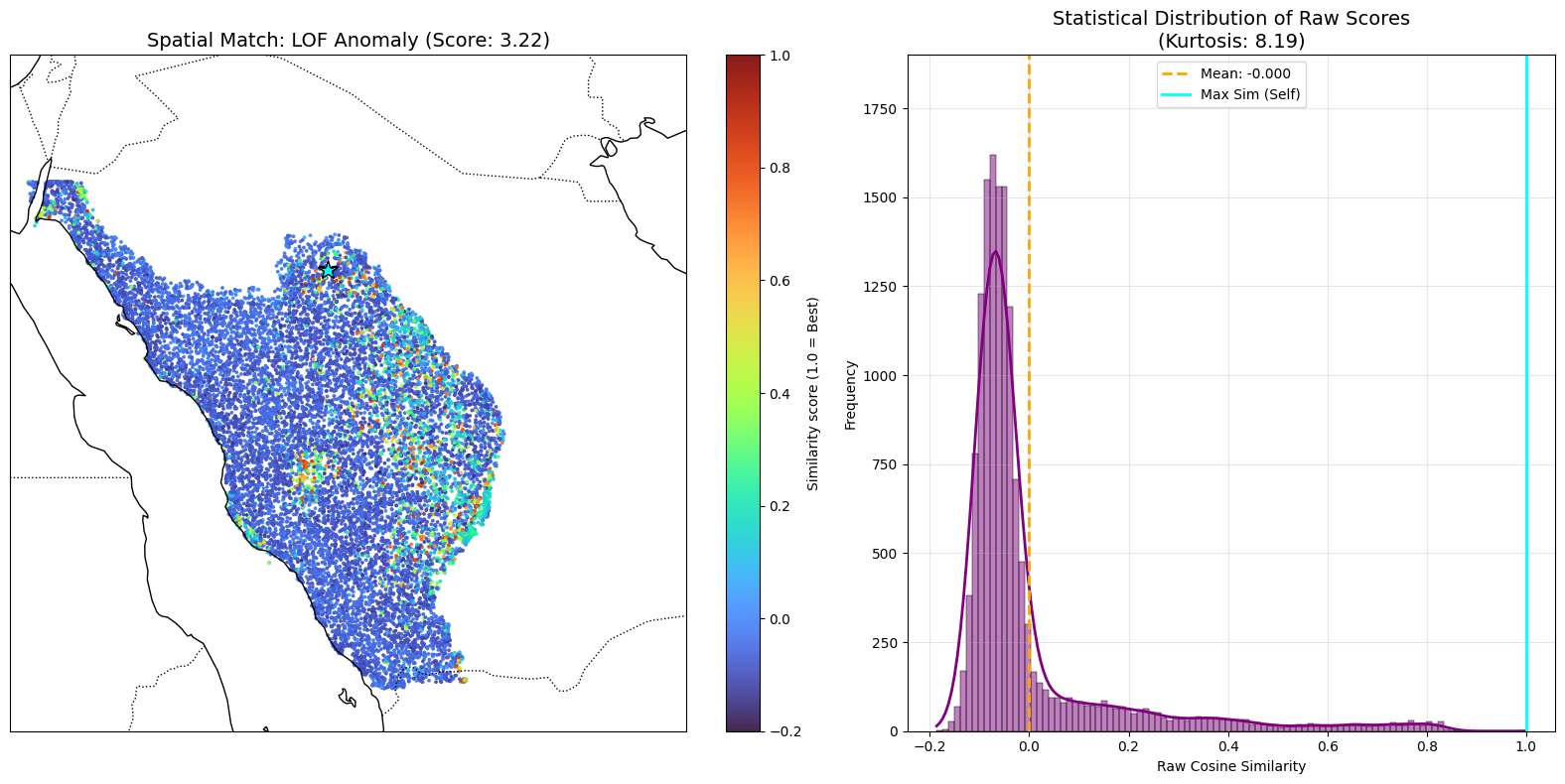}
        \caption{LOF: anomaly query.}
        \label{fig:sim_lof_anom}
    \end{subfigure}\hfill
    \begin{subfigure}[t]{0.49\textwidth}
        \centering
        \includegraphics[width=\textwidth, trim={0 0 0 1.3cm},clip]{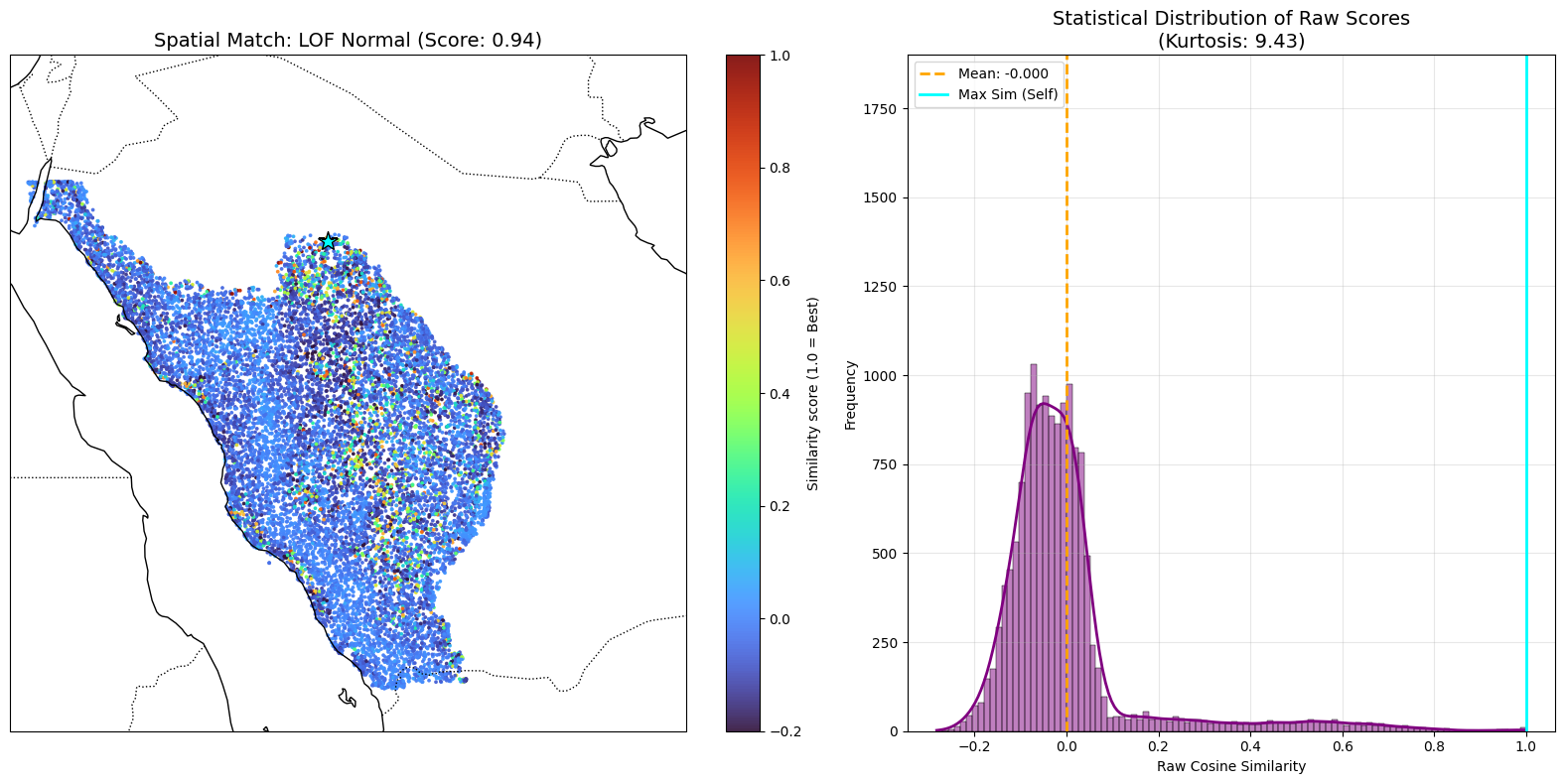}
        \caption{LOF: normal query.}
        \label{fig:sim_lof_norm}
    \end{subfigure}

\caption{
Similarity-search diagnostics for Arabian Shield graph embedding-space outliers (Experiment~2, Study 1). For each query (anomalous vs.\ typical; selected by Isolation Forest or LOF), nearest neighbors are retrieved by cosine similarity in the 256-dimensional embedding space and visualized as (left) the spatial footprint of similarity scores and (right) the similarity-score histogram. In all cases the score distributions are strongly leptokurtic (kurtosis $\approx 45.0$ for IF anomaly, $13.8$ for IF normal, $8.2$ for LOF anomaly, and $9.4$ for LOF normal), with means close to zero and a heavy right tail, indicating that high-similarity neighbors are sparse and highly selective. Anomaly queries exhibit weaker near-duplicate mass in the extreme-high similarity range and concentrate into compact, spatially coherent patches, whereas normal queries show broader high-similarity support consistent with dense manifold regions. Coordinates are not used for retrieval, only for visualization, so spatial coherence reflects structure learned in the embedding space rather than an explicit spatial prior.}

    \label{fig:similarity_search}
\end{figure*}

\paragraph{Spatial Correlation and Characteristic Scale.}
A global spatial correlogram computed over random embedding pairs shows a clear decay of cosine similarity with physical distance, crossing zero at about 300 km and flattening at approximately 400–600 km (Fig.~\ref{fig:global_correlogram}). This provides an estimate of the characteristic correlation length of the learned representation, plausibly associated with regional geological structure. Local radial profiles further differentiate anomalies from normal points: anomalous embeddings exhibit rapid decay to near-orthogonality, while normal points retain weak correlation over longer distances (Fig.~\ref{fig:radial_profiles}).

\begin{figure}[h]
    \centering
    \includegraphics[width=0.65\textwidth, trim={0 0 0 0.8cm},clip]{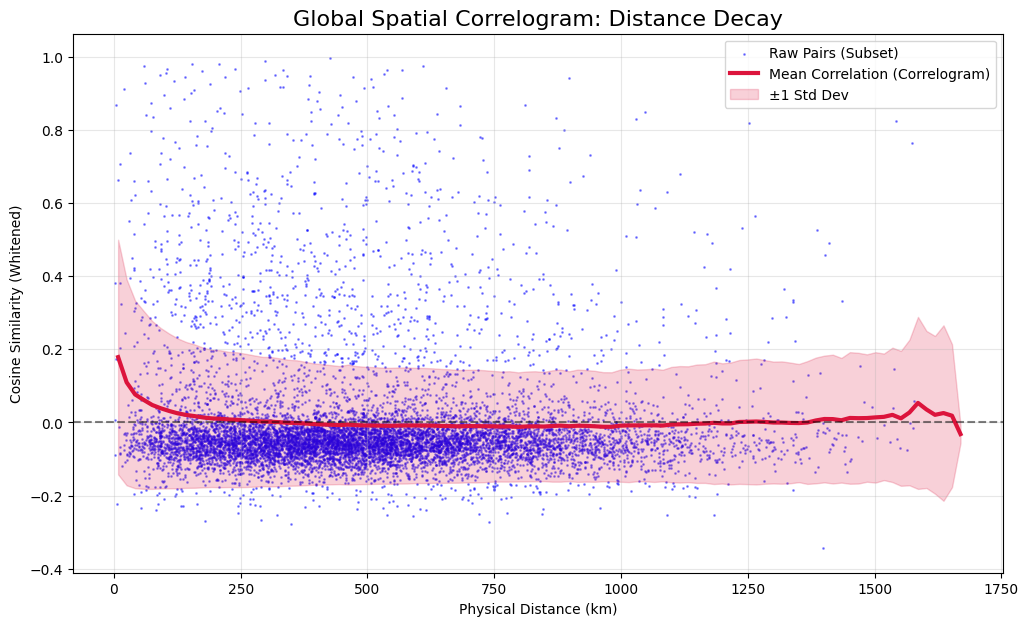}
    \caption{Global spatial correlogram of Arabian Shield graph embedding similarity versus geographic distance (Experiment~2, Study 1). Each blue point is the whitened cosine similarity of a sampled pair; the red curve shows the binned mean similarity, with shaded band indicating $\pm 1$ standard deviation. Similarity decays rapidly with distance, crossing $\approx 0$ near $\sim$300~km and approaching a near-flat regime by $\sim$400--600~km, consistent with terrane-scale continuity captured by the learned representation.}

    \label{fig:global_correlogram}
\end{figure}

\begin{figure}[h]
    \centering
    \begin{subfigure}[t]{0.49\textwidth}
        \centering
        \includegraphics[width=\textwidth, trim={0 0 0 0.8cm},clip]{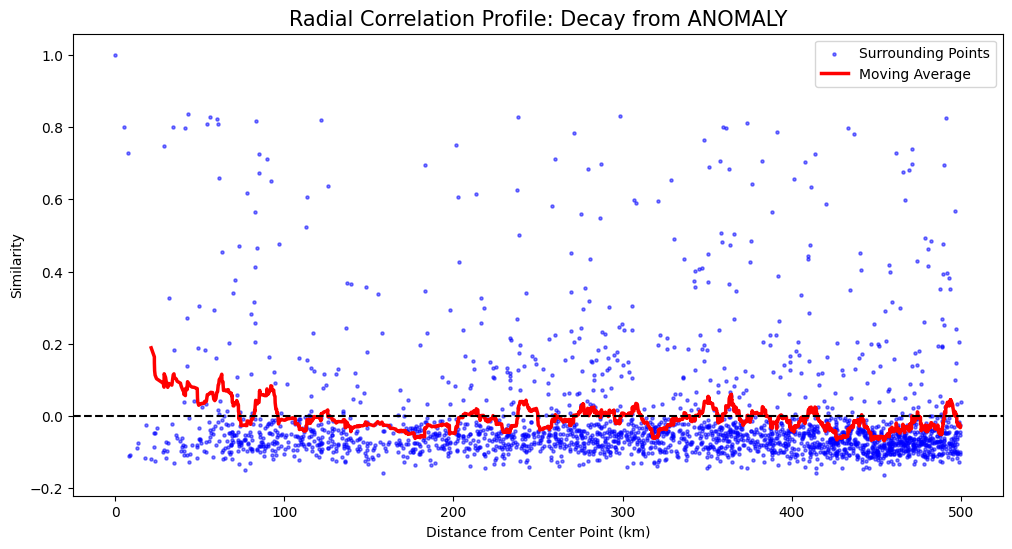}
        \caption{Radial similarity decay from a LOF-selected anomaly reference point.}
        \label{fig:radial_anomaly}
    \end{subfigure}\hfill
    \begin{subfigure}[t]{0.49\textwidth}
        \centering
        \includegraphics[width=\textwidth, trim={0 0 0 0.8cm},clip]{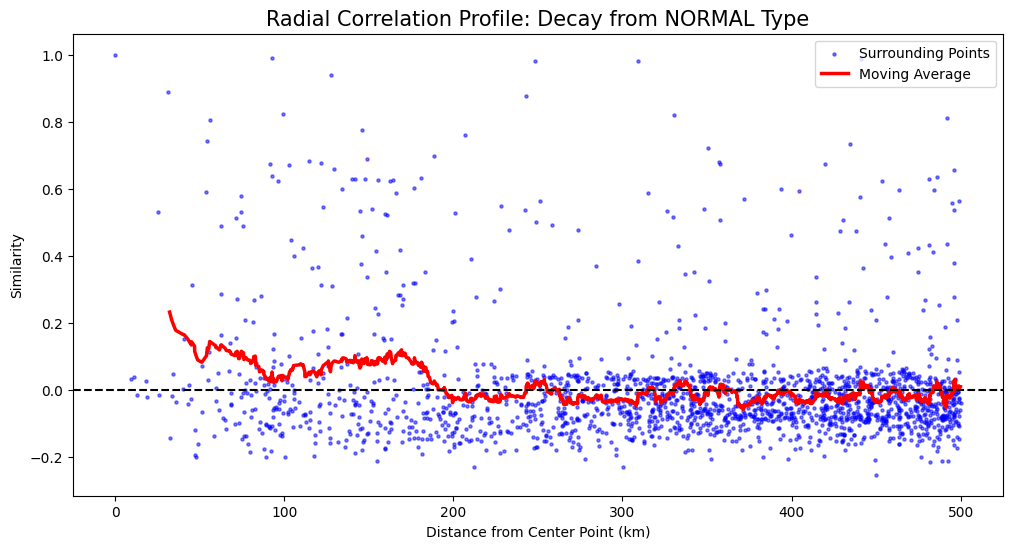}
        \caption{Radial similarity decay from a LOF-selected normal reference point.}
        \label{fig:radial_normal}
    \end{subfigure}
    \caption{Local radial correlograms comparing an Arabian Shield graph embedding outlier versus a typical embedding (Experiment~2, Study 1). Blue points show whitened cosine similarity between the reference embedding and surrounding locations as a function of distance; the red curve is a moving-average trend (dashed line denotes zero similarity). The anomaly decorrelates rapidly to near-orthogonality with many near-zero similarities, indicating a spatially localized signature, whereas the normal point retains weak positive correlation over longer ranges, consistent with membership in a broader regional embedding mode.}
    \label{fig:radial_profiles}
\end{figure}

\paragraph{Spectral Analysis: Scale, Texture, and Anisotropy.}
To quantify the spatial scales encoded by the embeddings, we analyze their frequency content in both 1D and 2D. Along a latitudinal transect (Lat $\approx 20.30^\circ$), the mean embedding power spectrum follows an approximate power law with slope close to $1/f^{0.75}$ over the resolvable band (Fig.~\ref{fig:psd_1d}), consistent with multi-scale geological organization rather than white noise or an overly smoothed field. Because the transect sampling is sparse (grid spacing $\Delta x \approx 2.05$~km), this result is band-limited by the Nyquist wavelength ($\lambda_N \approx 2\Delta x \approx 4.1$~km) and should be interpreted as a regional-scale signature. A complementary 2D Fourier analysis over a dense local patch reveals structured texture in the latent RMS field and direction-dependent spectral decay (Fig.~\ref{fig:fft_2d}), indicating that the embeddings preserve both spatial texture and a weak-to-moderate anisotropic fabric aligned with structural trends. Such approximate power-law spectra are commonly observed in natural spatial fields; here we treat it as evidence of multi-scale organization rather than a strict fractal model.

\begin{figure}[t]
    \centering
    \includegraphics[width=0.8\textwidth, trim={0 0 0 2.0cm},clip]{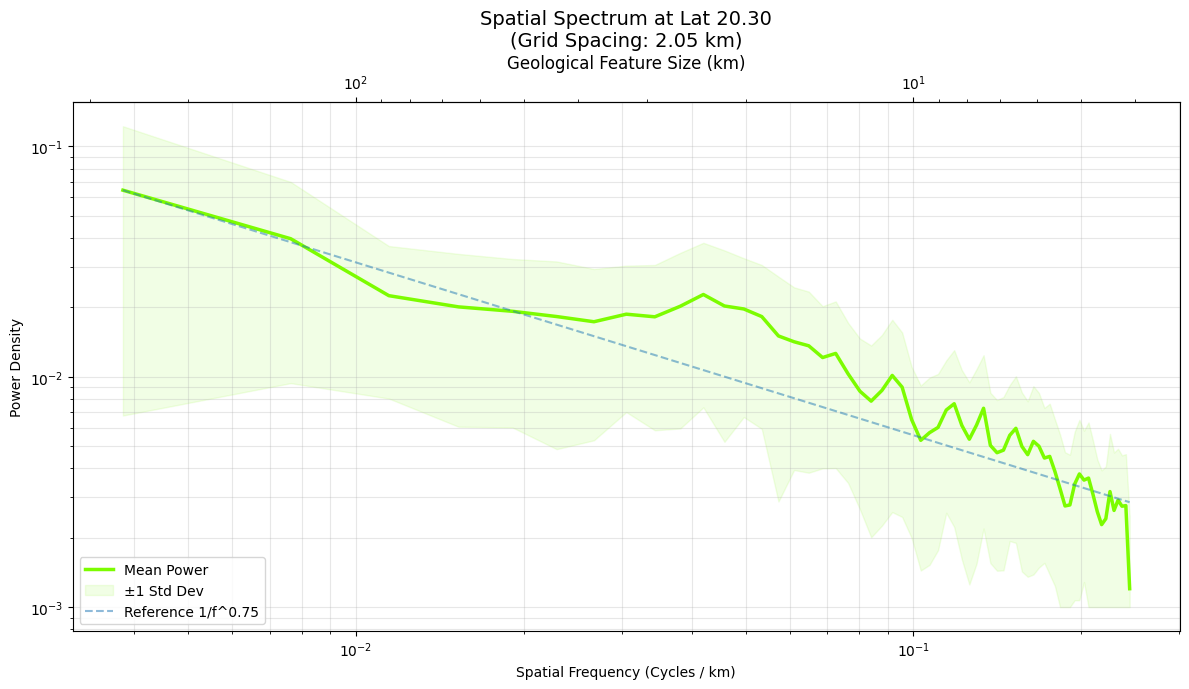}
    \caption{1D power spectral density (PSD) of the Arabian Shield 256-dimensional graph embeddings along a latitudinal transect (Lat $\approx 20.30^\circ$; Experiment~2, Study 1). The green curve shows mean power across embedding dimensions (shaded: $\pm1$ s.d.), compared to a $1/f^{0.75}$ reference trend. Over the resolved band the spectrum exhibits an approximately linear decay in log--log space, consistent with scale-rich (pink-noise-like) structure rather than white noise. The analysis is limited by the transect sampling ($\Delta x \approx 2.05$~km; Nyquist wavelength $\lambda_N \approx 4.1$~km), so sub-kilometre texture learned by the encoder is not directly observable here.}
    \label{fig:psd_1d}
\end{figure}

\begin{figure}[h]
    \centering
    \includegraphics[width=0.85\textwidth]{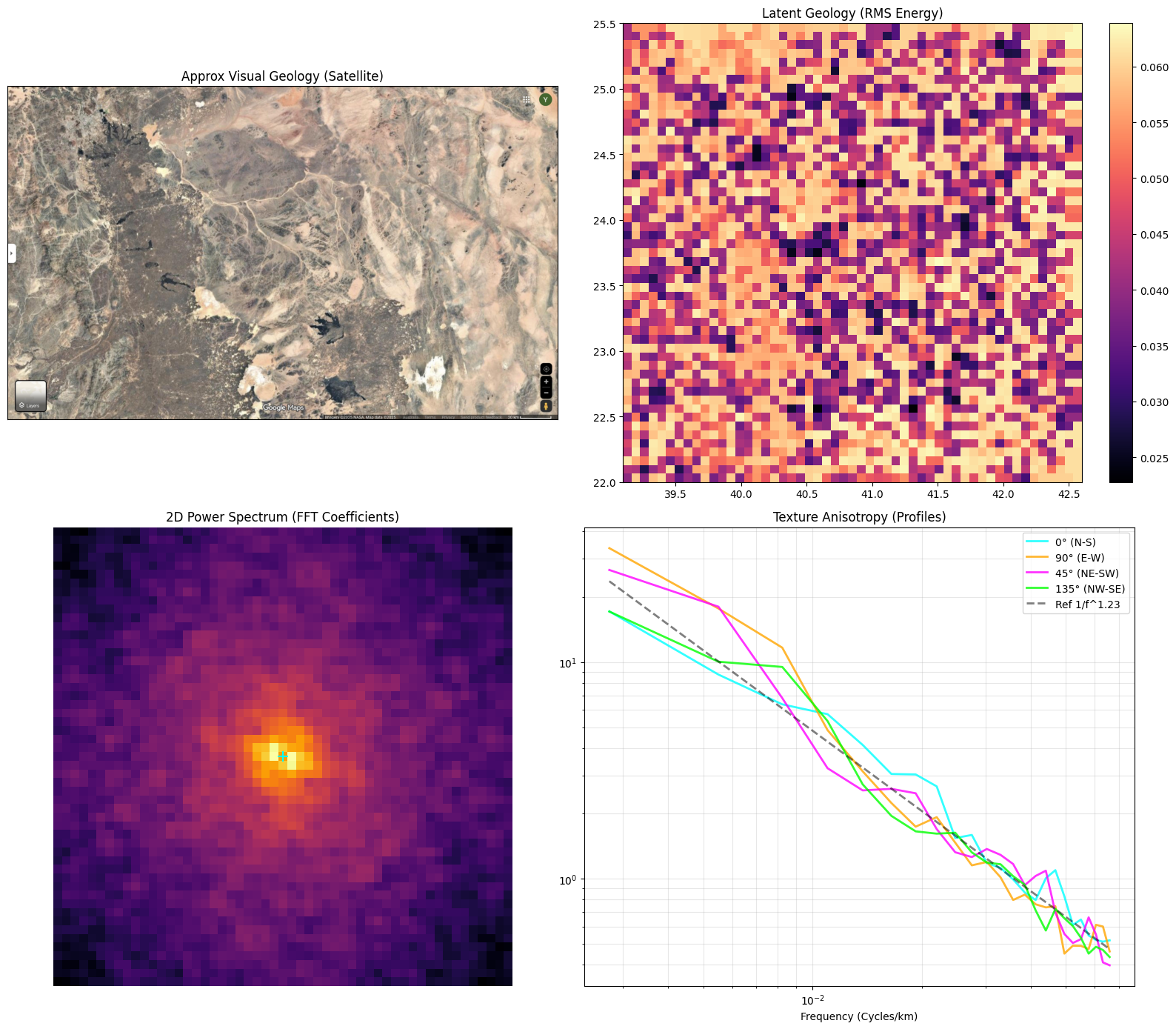}
    \caption{2D spectral analysis of a dense local patch of Arabian Shield graph embeddings (Experiment~2, Study 1). \textbf{Top-left:} approximate surface-domain reference from satellite imagery. \textbf{Top-right:} latent RMS energy map (aggregated embedding amplitude), showing coherent spatial texture in feature space. \textbf{Bottom-left:} 2D FFT power spectrum of the latent field, highlighting non-uniform frequency content. \textbf{Bottom-right:} directional radial PSD profiles (0$^\circ$, 45$^\circ$, 90$^\circ$, 135$^\circ$) with a power-law reference; systematic separation between directions indicates anisotropy (directional fabric) rather than isotropic texture, while the broad-band decay indicates multi-scale structure.}
    \label{fig:fft_2d}
\end{figure}

\paragraph{Graph Spectral Structure and Distributed Representation.}
Finally, we analyze embedding smoothness using Dirichlet energy over a spatial nearest-neighbor graph. Sorting dimensions by energy reveals a near-linear ramp (Fig.~\ref{fig:dirichlet_energy}), with no sharp separation between low- and high-frequency channels. Even the smoothest dimensions exhibit substantial variation, while the highest-energy dimensions encode fine-scale texture. This confirms a distributed, multi-frequency representation rather than a small set of dominant components.

\begin{figure}[h!]
    \centering
    \includegraphics[width=0.65\textwidth, trim={0 0 0 0.8cm},clip]{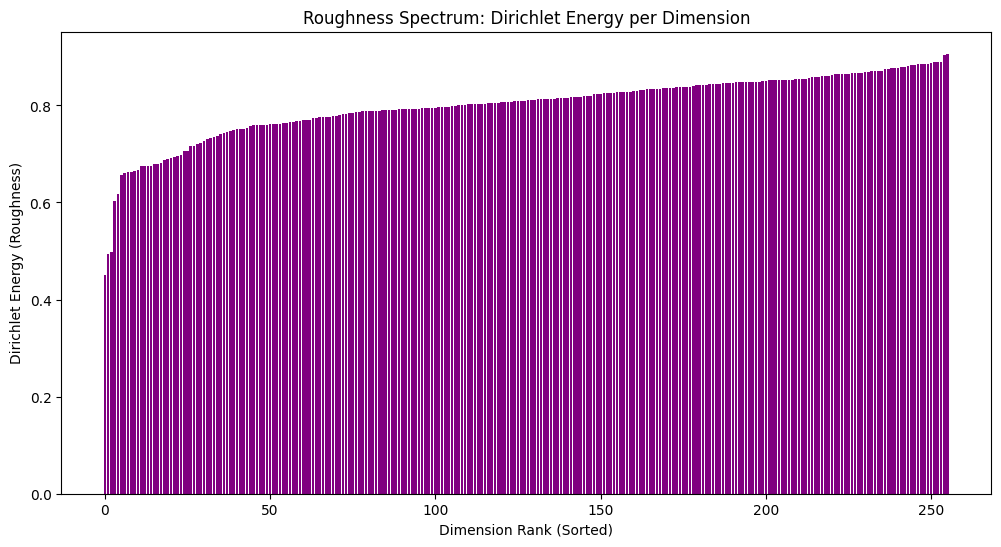}
    \caption{Dirichlet energy (graph Laplacian quadratic form) computed for each Arabian Shield graph embedding dimension on a $k$-NN spatial graph, sorted from smoothest to roughest (Experiment~2, Study 1). The absence of a sharp elbow and the near-linear ramp indicate that spatial frequency content is distributed across channels rather than concentrated in a small subset; even the lowest-energy dimensions retain nontrivial spatial structure, while the highest-energy dimensions capture localized texture.}

    \label{fig:dirichlet_energy}
\end{figure}

\paragraph{Qualitative Comparison via Unsupervised Segmentation.}
To qualitatively compare embedding clusters with visible surface domains, we apply K-Means clustering ($k=3$) to the embeddings within a dense local patch and compare the result to satellite imagery (Fig.~\ref{fig:kmeans_ground_truth}). The segmentation aligns closely with visible domains: a dark, rugged western belt consistent with mafic lithologies; a high-albedo eastern domain consistent with felsic or sandy cover; and an intermediate unit concentrated near contacts. Importantly, this spatial coherence emerges from the embedding vectors alone, with coordinates and imagery used only for visualization, suggesting that the embedding space captures surface-domain structure.

\begin{figure}[h]
    \centering
    \includegraphics[width=0.95\textwidth, trim={0 0 0 0.7cm},clip]{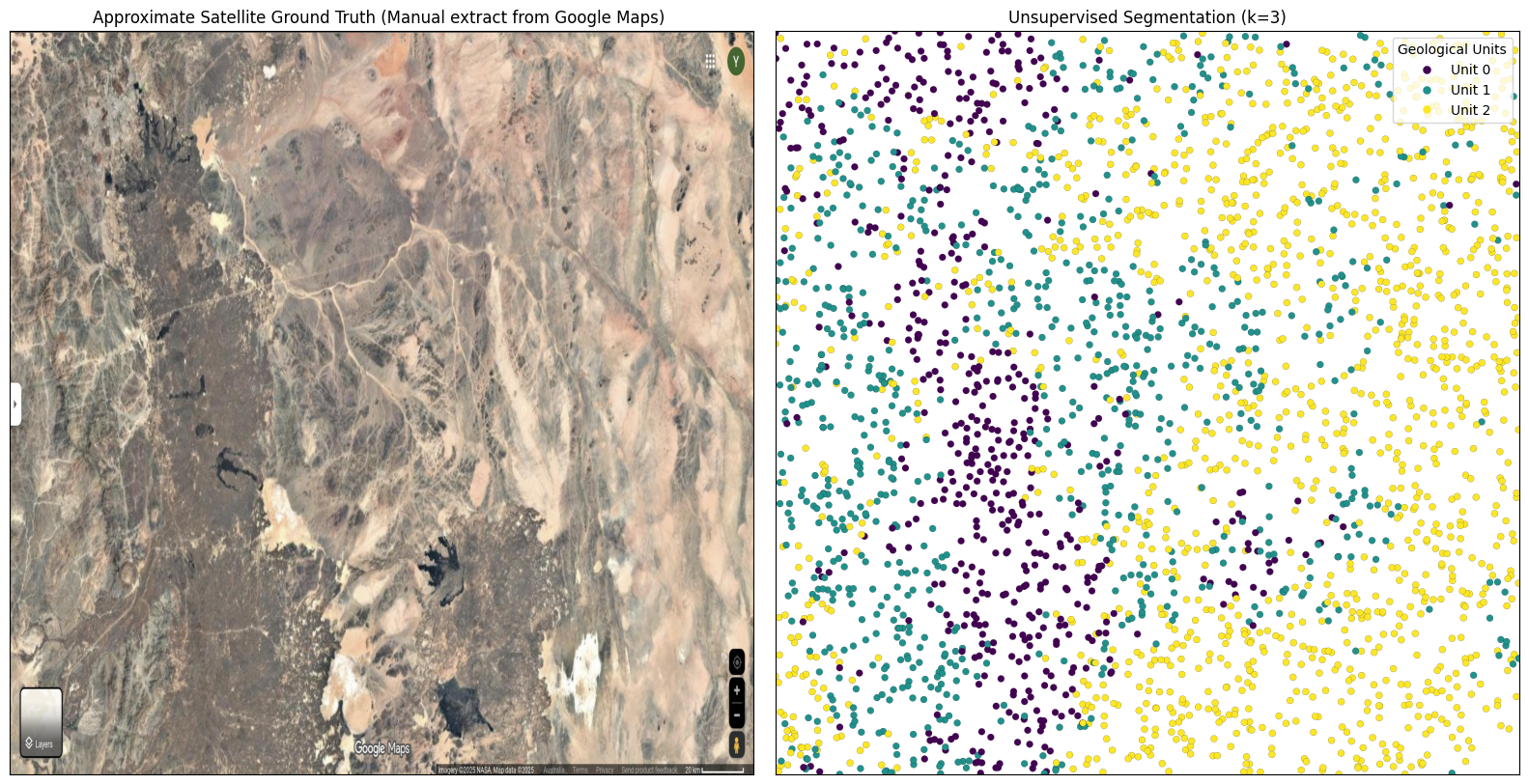}
\caption{
Left: optical satellite imagery of the patch, used here as a qualitative surface reference (Experiment~2, Study 1).
Right: unsupervised K-Means clustering ($k{=}3$) applied to the 256-dimensional Arabian Shield graph embeddings for the same area.
The resulting segments are spatially coherent and align with major visible domains, including a dark, rugged western belt, a lighter high-albedo eastern domain, and an intermediate/contact-associated unit.
From an ML perspective, the agreement is consistent with improved \emph{clusterability} of the learned representation: a simple centroid-based method (Euclidean Voronoi partition in embedding space) recovers coherent, separable groups without supervision.
Clustering is performed purely in embedding space (coordinates and imagery are used only for visualization).
}
    \label{fig:kmeans_ground_truth}
\end{figure}

\paragraph{Summary.}
Across independent diagnostics, the Arabian Shield embeddings exhibit consistent, spatially coherent and qualitatively geologically plausible structure without using absolute/global coordinates or labels during training. Individual dimensions form coherent spatial fields (Fig.~\ref{fig:dim_maps}), while random-pair similarities have near-zero mean ($\mu\approx0.054$), consistent with no strong embedding collapse and with a heavy right tail reflecting clustered terrane-scale domains (Fig.~\ref{fig:dot_product_distribution}). Nonlinear organization is evident in UMAP, which yields sharply bounded regional zones compared to a diffuse linear PCA baseline (only $\sim$17\% variance in the first three PCs; Fig.~\ref{fig:umap_pca_rgb}). In a local patch, even a simple K-Means partition ($k{=}3$) recovers spatially coherent units visually aligned with visible surface domains (Fig.~\ref{fig:kmeans_ground_truth}). Outliers selected by IF/LOF are supported by selective, leptokurtic similarity distributions and geographically compact match footprints (Fig.~\ref{fig:similarity_search}). Finally, similarity decays with distance at terrane scales (zero-crossing near $\sim$300~km; Fig.~\ref{fig:global_correlogram}) and spectral/graph analyses indicate multi-scale texture with distributed spatial frequencies across embedding channels (Figs.~\ref{fig:psd_1d}--\ref{fig:dirichlet_energy}).

\subsubsection{Study 2: Prospect-scale transfer via a dense embedding field over Suwaj}
\label{sec:exp2-suwaj-dense}

Study~1 characterized embedding structure at regional scale on the Arabian Shield. Here we test whether the same pretrained encoder \emph{transfers across scale} when evaluated densely over a focused area of interest (AOI) on a new dense evaluation grid, and whether the resulting embedding field supports reproducible, GIS-consumable \emph{unsupervised domain maps} and \emph{novelty proxies} (Figure~\ref{fig:suwaj_layers}).

Because the model is trained without explicit absolute/global geographic coordinates or geological labels, the key scientific question is whether the learned representation nevertheless induces a \emph{spatially coherent latent field} at prospect scale. Concretely, we test whether dense evaluation yields: (i) stable, contiguous domains rather than fragmented ``salt-and-pepper'' structure; (ii) sharp transitions localized in space (consistent with boundaries in the underlying signals); and (iii) controllable granularity under explicit analysis knobs (e.g., varying $K$ in K-Means) instead of unstable partitioning.

\paragraph{Dense embedding grid (100\,m).}
We evaluate the encoder trained in Study~1 (trained on random-patch samples distributed across the Arabian Shield) on a regular 2D grid over the Suwaj AOI at 100\,m spacing. The resulting field contains $1{,}250{,}414$ embedding vectors of dimension 256, defining a dense latent field $E(x,y)\in\mathbb{R}^{256}$ over the AOI. The AOI spans approximately 23.26--24.40$^\circ$N and 43.55--44.51$^\circ$E (CRS: EPSG:4979), with vertical extent $z \in [-10{,}000,\, 1{,}372]$\,m. Importantly, the model and preprocessing are held fixed; only the evaluation sampling density and spatial support change, making this a controlled test of representation stability under a new spatial regime.

\paragraph{Derived layers as controlled probes of representation structure.}
To probe the organization of the dense embedding field, we derive several families of raster layers (exported as Cloud-Optimized GeoTIFFs for reproducible GIS analysis).
(1) \textbf{Manifold projections} (PCA, UMAP) and corresponding RGB composites provide qualitative visualization of neighborhood structure in embedding space (Figure~\ref{fig:suwaj_layers}a--b), while individual principal components provide interpretable scalar fields (Figure~\ref{fig:suwaj_layers}e--f).
(2) \textbf{Unsupervised zonation} via K-Means across $K\in\{4,8,16,32,64\}$ supplies an explicit control knob on domain granularity; categorical maps at fixed $K$ summarize domain partitions (Figure~\ref{fig:suwaj_layers}c).
(3) \textbf{Novelty / boundary proxies} include distance-to-centroid under K-Means (Figure~\ref{fig:suwaj_layers}d) and multi-scale anomaly scores (Isolation Forest pyramid; Figure~\ref{fig:suwaj_layers}g) that emphasize signatures that are rare in embedding space and/or concentrated near transitions. Finally, texture proxies such as gray-level co-occurrence matrix (GLCM) statistics computed on selected components provide an additional controlled view of spatial heterogeneity (Figure~\ref{fig:suwaj_layers}h).

\paragraph{Qualitative evidence of coherent domains and multi-scale structure.}
The Suwaj rasters indicate that the embedding-derived layers are not spatially random: both nonlinear (UMAP) and linear (PCA) composites form contiguous regions separated by sharp boundaries (Figure~\ref{fig:suwaj_layers}a--b). Consistent with the strongly non-linear geometry observed at regional scale (Study~1), the UMAP composite tends to emphasize discrete, sharply bounded domains, while PCA captures smoother large-scale variation. K-Means zonation exhibits the expected coarse-to-fine behavior as $K$ increases: low $K$ recovers macro-domains, whereas higher $K$ refines these into more detailed subdomains (example at $K{=}8$ in Figure~\ref{fig:suwaj_layers}c). The distance-to-centroid layers provide a complementary continuous view, highlighting mixed/transition zones and locally atypical signatures that are difficult to capture with categorical maps alone (Figure~\ref{fig:suwaj_layers}d). The anomaly and texture proxies further concentrate on spatially localized deviations and heterogeneous structure (Figure~\ref{fig:suwaj_layers}g--h), serving as hypothesis-generating layers for downstream validation rather than geological labels.

\paragraph{Summary.}
This study demonstrates a scale-transfer property of the learned representation: a regional pretrained encoder, trained without absolute/global coordinate input, can be evaluated densely on a 2D grid at 100\,m spacing to yield a spatially coherent embedding field over Suwaj (Figure~\ref{fig:suwaj_layers}). While these layers are not geological labels, their consistency across projection methods, the controllable granularity via $K$, and the agreement between categorical (zonation) and continuous (distance/anomaly/texture) probes provide evidence that the embedding space captures structured, multi-scale variation suitable for downstream interpretation and targeting.

\begin{figure*}[t]
\captionsetup{font=small,skip=2pt}
\centering

\begin{subfigure}[t]{0.24\textwidth}
  \includegraphics[width=\linewidth]{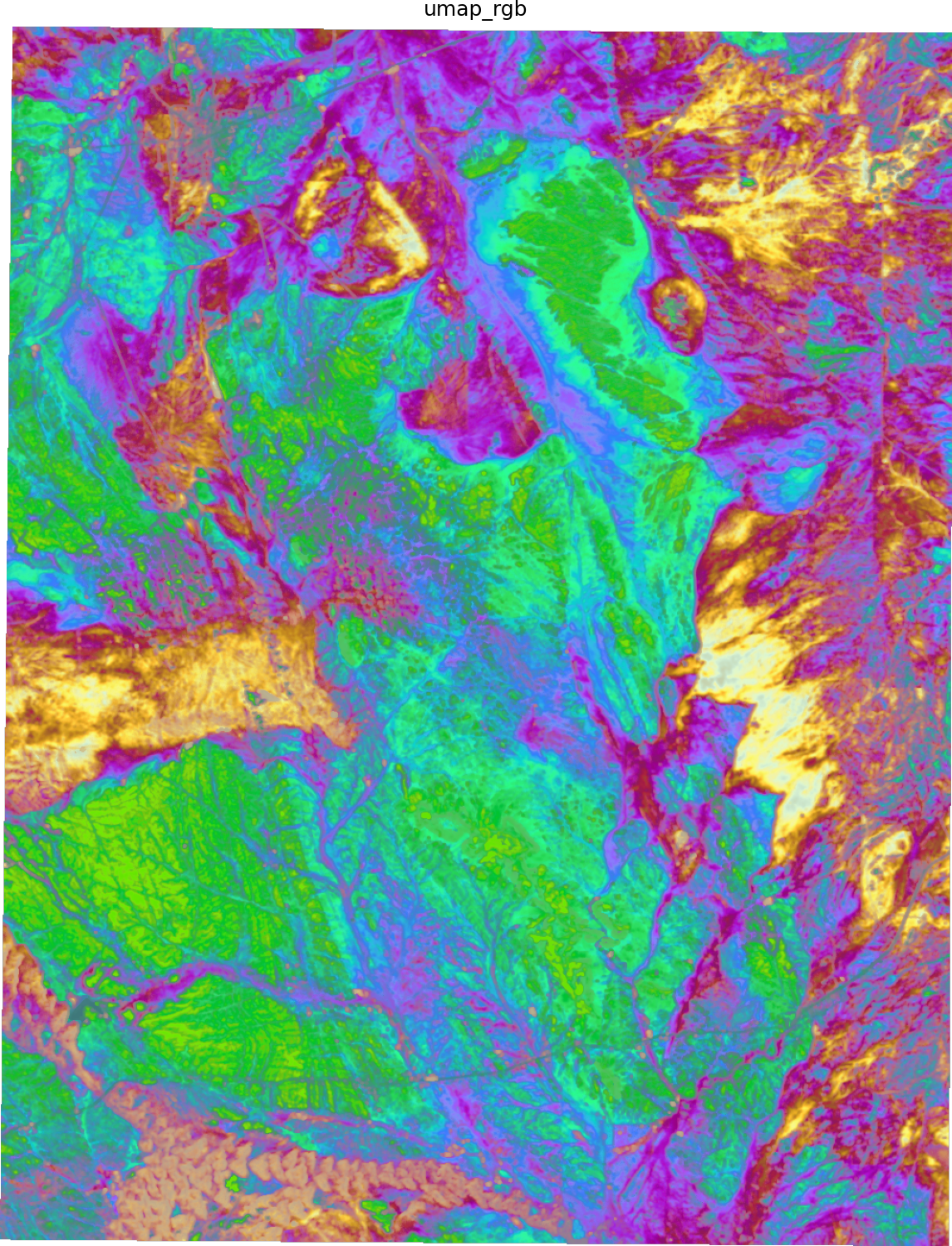}
  \caption{UMAP RGB}
\end{subfigure}\hfill
\begin{subfigure}[t]{0.24\textwidth}
  \includegraphics[width=\linewidth]{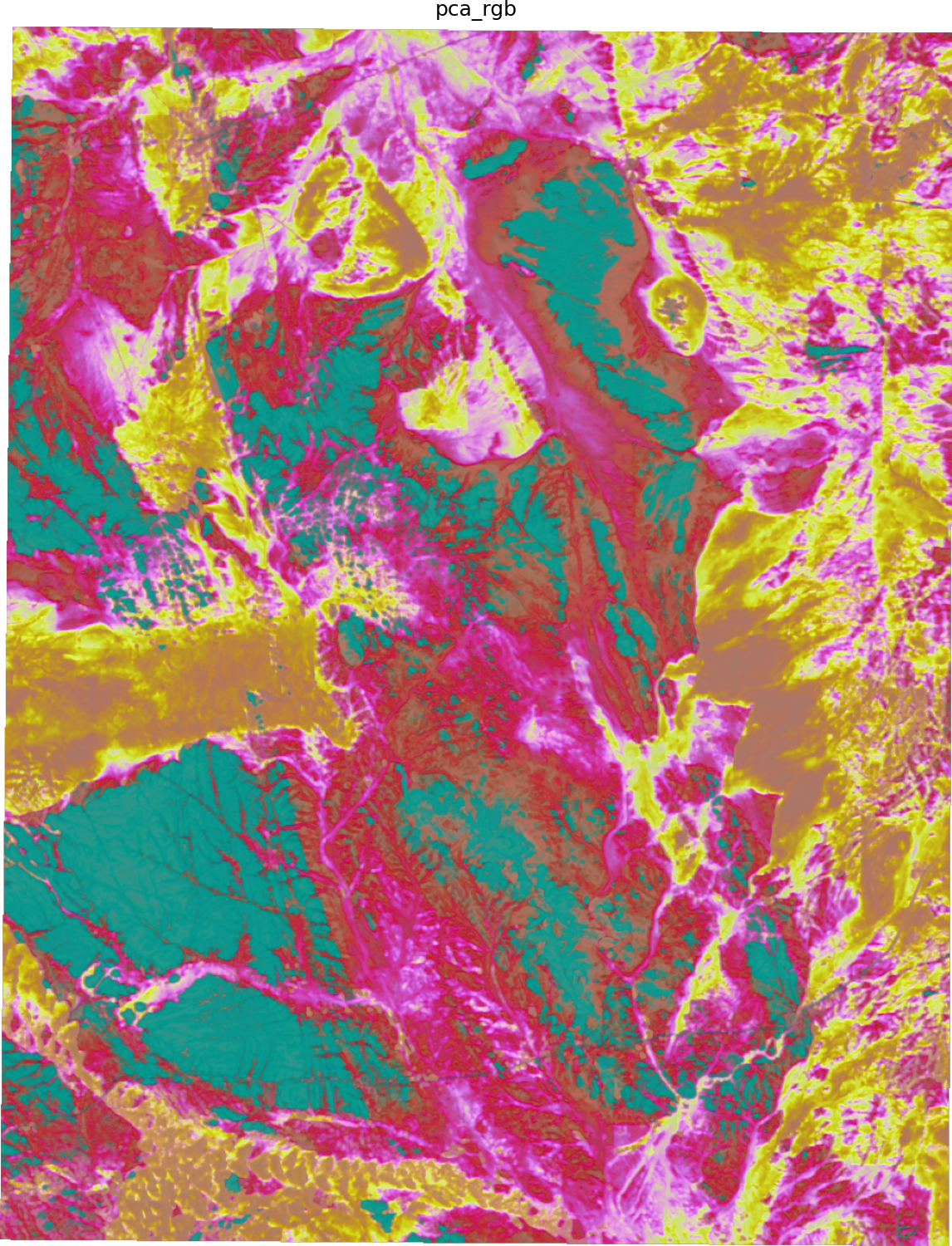}
  \caption{PCA RGB}
\end{subfigure}\hfill
\begin{subfigure}[t]{0.24\textwidth}
  \includegraphics[width=\linewidth]{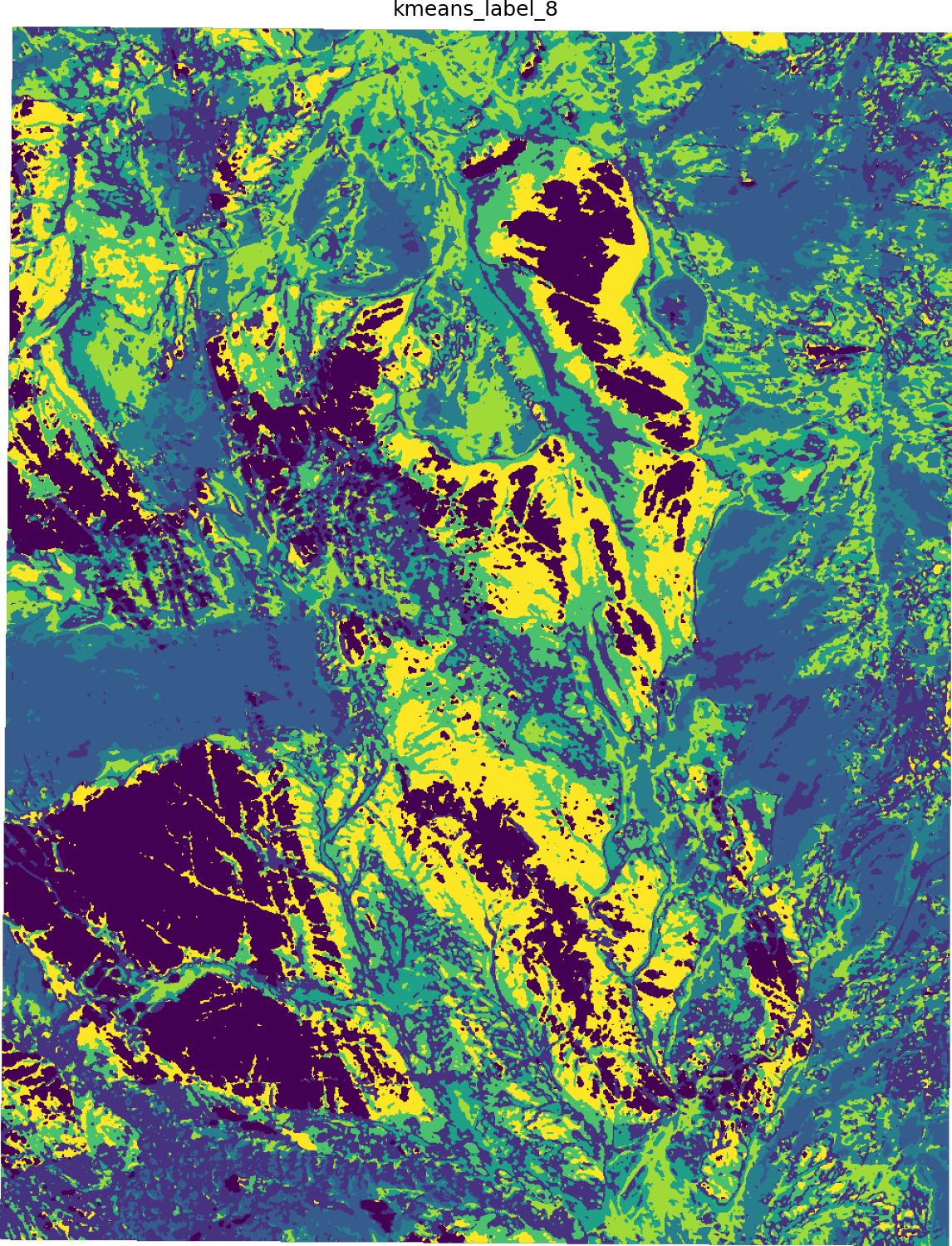}
  \caption{$K$-Means labels ($K{=}8$)}
\end{subfigure}\hfill
\begin{subfigure}[t]{0.24\textwidth}
  \includegraphics[width=\linewidth]{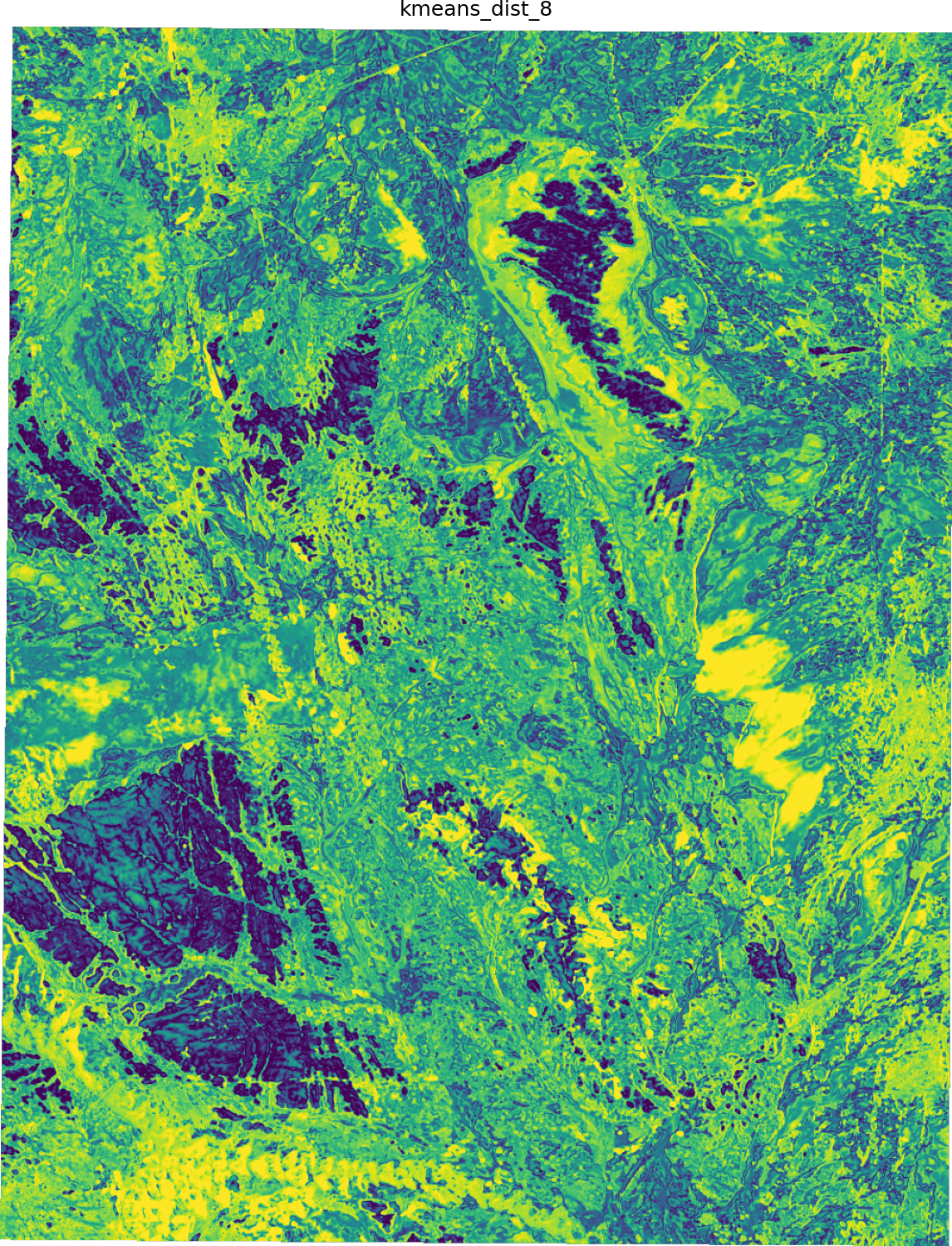}
  \caption{Dist. to centroid ($K{=}8$)}
\end{subfigure}

\vspace{0.2em} 

\begin{subfigure}[t]{0.24\textwidth}
  \includegraphics[width=\linewidth]{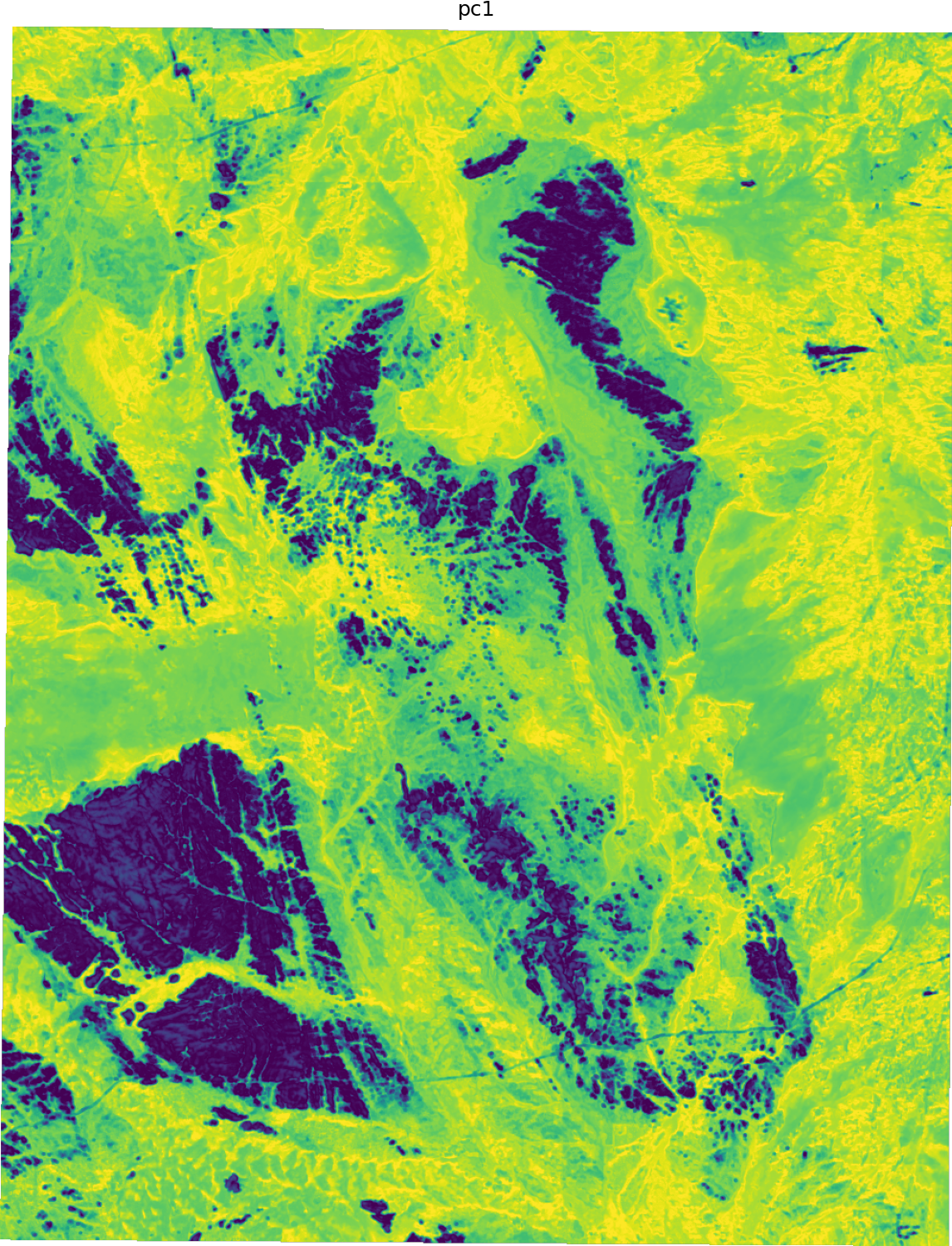}
  \caption{PC1}
\end{subfigure}\hfill
\begin{subfigure}[t]{0.24\textwidth}
  \includegraphics[width=\linewidth]{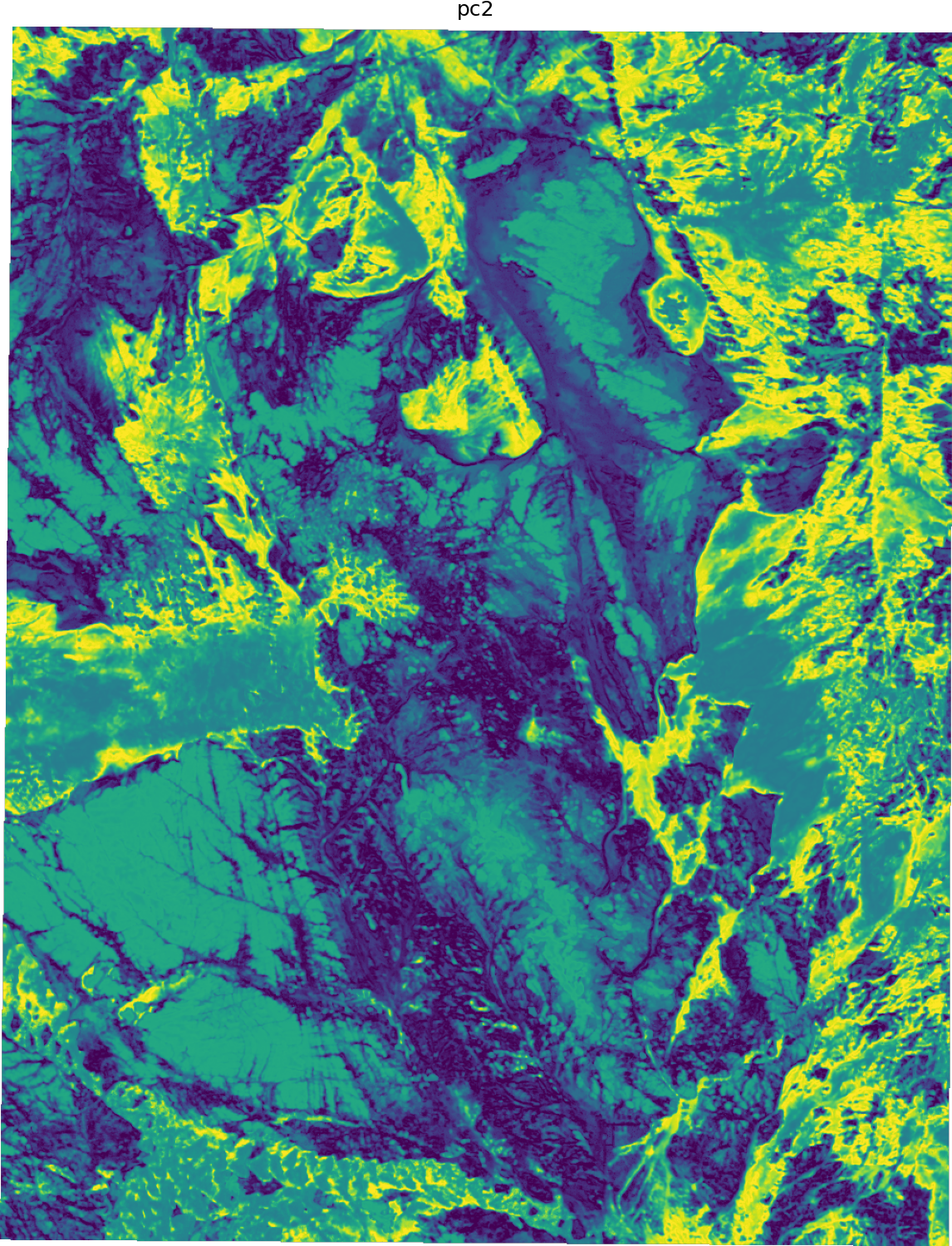}
  \caption{PC2}
\end{subfigure}\hfill
\begin{subfigure}[t]{0.24\textwidth}
  \includegraphics[width=\linewidth]{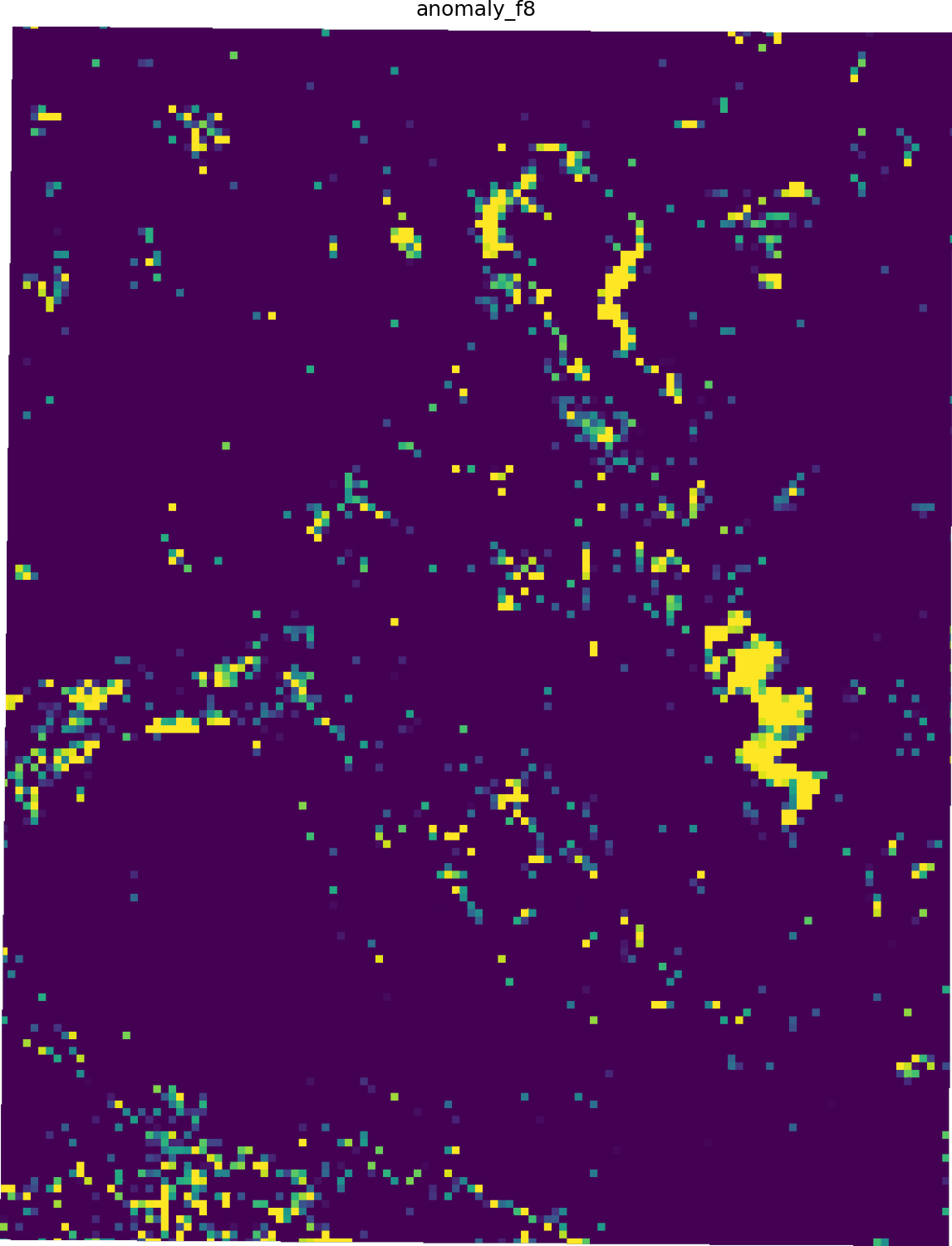}
  \caption{Anomaly ($f{=}8$)}
\end{subfigure}\hfill
\begin{subfigure}[t]{0.24\textwidth}
  \includegraphics[width=\linewidth]{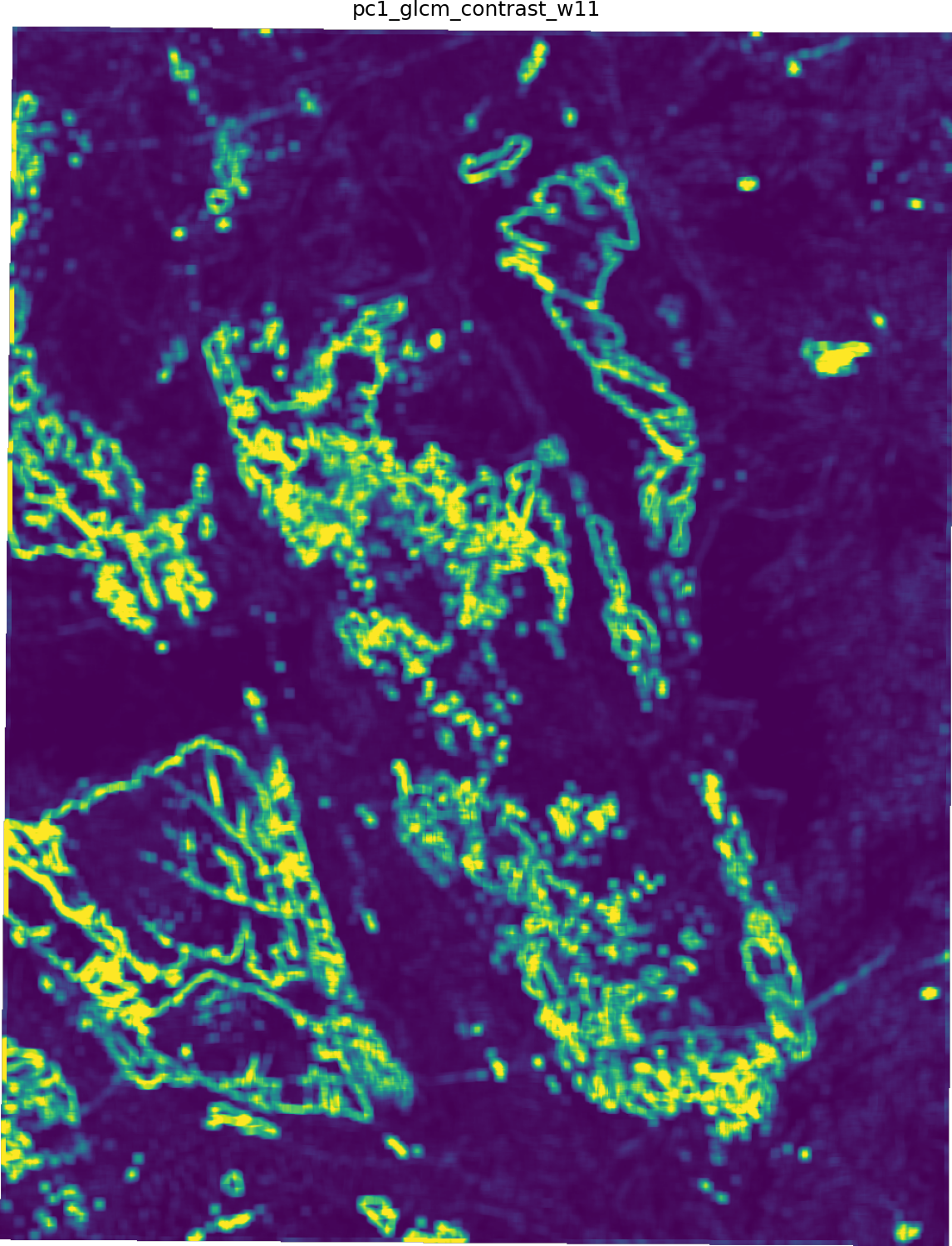}
  \caption{PC1 GLCM contrast}
\end{subfigure}

\caption{Dense graph embedding-derived layers over the Suwaj AOI evaluated on a 2D regular grid at 100\,m spacing (Experiment~2, Study 2). Top row: RGB composites of the embedding manifold under UMAP and PCA, and $K$-Means zonation with $K{=}8$ with the corresponding continuous distance-to-centroid proxy (highlighting mixed/transition areas and locally atypical signatures). Bottom row: principal component rasters (PC1--PC2), a multi-scale Isolation Forest anomaly score at window factor $f{=}8$ (the displayed anomaly-pyramid scale), and a texture probe computed from PC1 (GLCM contrast, window $w{=}11$), emphasizing local heterogeneity and edge/structural texture.}
\label{fig:suwaj_layers}
\end{figure*}

\section{Conclusion}
\label{sec:conclusion}

We presented \textbf{MABLE}, a self-supervised representation learning pipeline for large, heterogeneous geospatial graphs that encourages robust \emph{node} and \emph{graph}-level embeddings using a tightly integrated set of objectives and stability constraints. The method couples masked node reconstruction with representation-shaping losses that explicitly encourage useful geometry: (i) node correspondence and node uniformity to promote cross-view consistency while discouraging collapse and improving feature diversity, (ii) graph alignment to enforce invariance across augmented views (including subgraph views), and (iii) graph uniformity to maintain global spread across graphs within a batch. On the architectural side, we combine a modality-agnostic token interface with a Lipschitz-controlled pooling operator and a well-conditioned (bi-Lipschitz) reconstruction head, linking practical stability to the theoretical sensitivity bounds established for the reconstruction slice and pooling components discussed throughout the paper.

Across the mineral-prospectivity experiments, MABLE embeddings are most useful as complementary contextual representations rather than as standalone predictors under strong geographic distribution shift. On the local-scale copper benchmark (redacted Experiment 1), embeddings alone are not sufficient to extrapolate into a geographically distinct holdout region, but compact embedding components improve coordinate-only and engineered geophysical baselines, indicating transferable context not fully summarized by local MAG/ANT statistics. On the regional-scale Arabian Shield corpus, the learned 256-dimensional graph embeddings exhibit strong spatial coherence without using absolute/global coordinates or geological labels during training: manifold projections, similarity search, anomaly detection, and graph/spectral diagnostics consistently indicate a distributed, multi-scale representation that organizes the Shield into stable domains and highlights rare signatures. Finally, a dense 100\,m grid evaluation over Suwaj demonstrates \emph{scale transfer}: a regional pretrained encoder can be evaluated prospect-scale to produce GIS-consumable, embedding-derived layers (zonation, novelty proxies, and texture probes) that are spatially coherent and controllably granular. These regional and prospect-scale products should be interpreted as hypothesis-generating latent layers rather than geological labels.

A core architectural differentiator explored in this work is the \emph{feature-embedding slice} for reconstruction with a constrained decoder, together with a complementary contextual slice regularized by cross-view node correspondence. Ablations indicate that this architectural choice, when properly tuned, improves training dynamics and reduces overfitting pressure. These results support the practical use of the slice+bi-Lipschitz decoder together with contextual cross-view consistency as stability-oriented design choices that need not materially sacrifice predictive utility when appropriately parameterized.

\paragraph{Limitations and future work.}
Several directions remain open. First, while embeddings can complement engineered context under spatial distribution shift, stronger downstream models (e.g., residual or hierarchical predictors, or end-to-end fine-tuning) may better exploit the learned context in the strictest extrapolation regimes. Second, the current construction uses edge-free graphs; incorporating lightweight, uncertainty-aware relational structure (e.g., learned sparse edges or physics-informed neighborhoods) is a natural extension when explicit connectivity is meaningful. Third, prospect-scale products derived from embeddings (zonation, anomalies, textures) should be validated more directly against mapped geology and known mineral systems to quantify correspondence beyond qualitative alignment. More broadly, evaluating transfer across regions, sensor stacks, sampling densities and characterizing failure modes under missing modalities or strong coverage gaps will further clarify the operational envelope of MABLE for exploration workflows.

Overall, MABLE provides a principled, end-to-end approach to learning stable geospatial graph representations that (i) remain well-conditioned under augmentation and sampling artifacts, (ii) support both node-level reconstruction fidelity and cross-view node consistency together with global embedding robustness, and (iii) yield coherent, multi-scale latent structure that can be operationalized as GIS-ready layers for interpretation and targeting in the settings studied here.

\newpage

\bibliographystyle{plainnat}
\bibliography{refs}

@inproceedings{Zhang2022UMAE,
  author       = {Zhang, Qi and Wang, Yifei and Wang, Yisen},
  title        = {How Mask Matters: Towards Theoretical Understandings of Masked Autoencoders},
  booktitle    = {Proceedings of the 36th International Conference on Neural Information Processing Systems},
  series       = {NeurIPS '22},
  year         = {2022},
  publisher    = {Curran Associates Inc.},
  address      = {Red Hook, NY, USA},
  articleno    = {1967},
  numpages     = {13},
  location     = {New Orleans, LA, USA},
  isbn         = {9781713871088}
}

@inproceedings{LiEtAl2023Mask,
  author       = {Li, Jintang and Wu, Ruofan and Sun, Wangbin and Chen, Liang and Tian, Sheng and Zhu, Liang and Meng, Changhua and Zheng, Zibin and Wang, Weiqiang},
  title        = {What's Behind the Mask: Understanding Masked Graph Modeling for Graph Autoencoders},
  booktitle    = {Proceedings of the 29th ACM SIGKDD Conference on Knowledge Discovery and Data Mining},
  series       = {KDD '23},
  year         = {2023},
  publisher    = {Association for Computing Machinery},
  address      = {New York, NY, USA},
  doi          = {10.1145/3580305.3599546},
  url          = {https://doi.org/10.1145/3580305.3599546},
  pages        = {1268--1279},
  isbn         = {979-8-4007-0103-0}
}

@inproceedings{XuEtAl2019,
  author       = {Xu, Keyulu and Hu, Weihua and Leskovec, Jure and Jegelka, Stefanie},
  title        = {How Powerful are Graph Neural Networks?},
  booktitle    = {Proceedings of the 7th International Conference on Learning Representations},
  series       = {ICLR '19},
  year         = {2019},
  publisher    = {OpenReview.net},
  address      = {New Orleans, LA, USA},
  url          = {https://openreview.net/forum?id=ryGs6iA5Km}
}

@inproceedings{AlkinEtAl2025MIM,
  author       = {Alkin, Benedikt and Miklautz, Lukas and Hochreiter, Sepp and Brandstetter, Johannes},
  title        = {{MIM-Refiner}: A Contrastive Learning Boost from Intermediate Pre-Trained Representations},
  booktitle    = {International Conference on Learning Representations (ICLR)},
  year         = {2025},
  url          = {https://openreview.net/forum?id=F01HbewbwU}
}

@inproceedings{Daruna2024GFM4MPM,
  author = {Daruna, Angel and Zadorozhnyy, Vasily and Lukoczki, Georgina and Chiu, Han-Pang},
  title = {GFM4MPM: Towards Geospatial Foundation Models for Mineral Prospectivity Mapping},
  booktitle = {Proceedings of the 32nd ACM International Conference on Advances in Geographic Information Systems},
  series = {SIGSPATIAL '24},
  year = {2024},
  pages = {565--568},
  publisher = {Association for Computing Machinery},
  address = {New York, NY, USA},
  url = {https://doi.org/10.1145/3678717.3691268},
  doi = {10.1145/3678717.3691268}
}

@inproceedings{He2022MAE,
  title     = {Masked Autoencoders Are Scalable Vision Learners},
  author    = {Kaiming He and Xinlei Chen and Saining Xie and Yanghao Li and Piotr Doll{\'a}r and Ross B. Girshick},
  booktitle = {Proceedings of the IEEE/CVF Conference on Computer Vision and Pattern Recognition (CVPR)},
  year      = {2022},
  pages     = {15979--15988},
  doi       = {10.1109/CVPR52688.2022.01551},
  url       = {https://doi.org/10.1109/CVPR52688.2022.01551}
}

@article{Oord2018CPC,
  author    = {A{\"{a}}ron van den Oord and Yazhe Li and Oriol Vinyals},
  title     = {Representation Learning with {Contrastive Predictive Coding}},
  journal   = {CoRR},
  volume    = {abs/1807.03748},
  year      = {2018},
  url       = {http://arxiv.org/abs/1807.03748},
  eprinttype= {arXiv},
  eprint    = {1807.03748}
}

@inproceedings{Chen2020SimCLR,
  author    = {Ting Chen and Simon Kornblith and Mohammad Norouzi and Geoffrey Hinton},
  title     = {A Simple Framework for Contrastive Learning of Visual Representations},
  booktitle = {Proceedings of the 37th International Conference on Machine Learning},
  series    = {ICML'20},
  year      = {2020},
  pages     = {1597--1607},
  publisher = {JMLR.org},
  url       = {https://proceedings.mlr.press/v119/chen20j.html}
}

@inproceedings{HaoChen2021SCL,
  author    = {Jeff Z. HaoChen and Colin Wei and Adrien Gaidon and Tengyu Ma},
  title     = {Provable Guarantees for Self-Supervised Deep Learning with Spectral Contrastive Loss},
  booktitle = {Advances in Neural Information Processing Systems},
  editor    = {M. Ranzato and A. Beygelzimer and Y. Dauphin and P.S. Liang and J. Wortman Vaughan},
  volume    = {34},
  pages     = {5000--5011},
  year      = {2021},
  publisher = {Curran Associates, Inc.},
  url       = {https://proceedings.neurips.cc/paper_files/paper/2021/file/27debb435021eb68b3965290b5e24c49-Paper.pdf}
}

@article{Tosh2020ContrastiveLM,
  author      = {Christopher Tosh and Akshay Krishnamurthy and Daniel J. Hsu},
  title       = {Contrastive Learning, Multi-View Redundancy, and Linear Models},
  journal     = {CoRR},
  volume      = {abs/2008.10150},
  year        = {2020},
  eprinttype  = {arXiv},
  eprint      = {2008.10150},
  url         = {https://arxiv.org/abs/2008.10150}
}

@inproceedings{Wu2024TorchSpatial,
  author    = {Nemin Wu and Qian Cao and Zhangyu Wang and Zeping Liu and Yanlin Qi and Jielu Zhang and Joshua Ni and Xiaobai Yao and Hongxu Ma and Lan Mu and Stefano Ermon and Tanuja Ganu and Akshay Nambi and Ni Lao and Gengchen Mai},
  title     = {{TorchSpatial}: A Location Encoding Framework and Benchmark for Spatial Representation Learning},
  booktitle = {Advances in Neural Information Processing Systems},
  editor    = {Amir Globerson and Lester Mackey and Dianna Belgrave and Angela Fan and Ulrich Paquet and Jakub Tomczak and Cheng Zhang},
  volume    = {37},
  pages     = {81437--81460},
  year      = {2024},
  publisher = {Curran Associates, Inc.},
  url       = {https://proceedings.neurips.cc/paper_files/paper/2024/file/9449c2d5b0cc8c9a445752f3ff195a1c-Paper-Datasets_and_Benchmarks_Track.pdf}
}

@article{Hjelm2018DeepInfoMax,
  author      = {R. Devon Hjelm and Alex Fedorov and Samuel Lavoie-Marchildon and Karan Grewal and Phil Bachman and Adam Trischler and Yoshua Bengio},
  title       = {Learning Deep Representations by Mutual Information Estimation and Maximization},
  journal     = {CoRR},
  volume      = {abs/1808.06670},
  year        = {2018},
  eprinttype  = {arXiv},
  eprint      = {1808.06670},
  url         = {https://arxiv.org/abs/1808.06670}
}

@inproceedings{Velickovic2019DeepGraphInfomax,
  author    = {Petar Veli{\v{c}}kovi{\'{c}} and William Fedus and William L. Hamilton and Pietro Li{\`o} and Yoshua Bengio and R. Devon Hjelm},
  title     = {Deep Graph Infomax},
  booktitle = {International Conference on Learning Representations (ICLR)},
  year      = {2019},
  address   = {New Orleans, USA},
  url       = {https://openreview.net/forum?id=rklz9iAcKQ}
}

@inproceedings{Gutmann2010NCE,
  author    = {Michael Gutmann and Aapo Hyv{\"a}rinen},
  title     = {Noise‑Contrastive Estimation: A New Estimation Principle for Unnormalized Statistical Models},
  booktitle = {Proceedings of the 13th International Conference on Artificial Intelligence and Statistics (AISTATS)},
  series    = {Proceedings of Machine Learning Research},
  volume    = {9},
  pages     = {297--304},
  year      = {2010},
  address   = {Chia Laguna Resort, Sardinia, Italy},
  publisher = {PMLR},
  url       = {https://proceedings.mlr.press/v9/gutmann10a.html}
}

@article{He2019MoCo,
  author      = {Kaiming He and Haoqi Fan and Yuxin Wu and Saining Xie and Ross Girshick},
  title       = {Momentum Contrast for Unsupervised Visual Representation Learning},
  journal     = {CoRR},
  volume      = {abs/1911.05722},
  year        = {2019},
  eprinttype  = {arXiv},
  eprint      = {1911.05722},
  note        = {CVPR~2020 camera-ready},
  url         = {https://arxiv.org/abs/1911.05722}
}

@inproceedings{Wang2020AlignmentUniformity,
  author    = {Tongzhou Wang and Phillip Isola},
  title     = {Understanding Contrastive Representation Learning through Alignment and Uniformity on the Hypersphere},
  booktitle = {Proceedings of the 37th International Conference on Machine Learning (ICML)},
  series    = {Proceedings of Machine Learning Research},
  volume    = {119},
  pages     = {9929--9939},
  year      = {2020},
  address   = {Virtual Conference},
  publisher = {PMLR},
  url       = {https://proceedings.mlr.press/v119/wang20k.html}
}

@inproceedings{Miyato2018SpectralNorm,
  author    = {Takeru Miyato and Toshiki Kataoka and Masanori Koyama and Yuichi Yoshida},
  title     = {Spectral Normalization for Generative Adversarial Networks},
  booktitle = {International Conference on Learning Representations (ICLR)},
  year      = {2018},
  address   = {Vancouver, Canada},
  url       = {https://openreview.net/forum?id=B1QRgziT-}
}

@inproceedings{Cisse2017Parseval,
  author    = {Moustapha Cisse and Piotr Bojanowski and {\'E}douard Grave and Yann Dauphin and Nicolas Usunier},
  title     = {Parseval Networks: Improving Robustness to Adversarial Examples},
  booktitle = {Proceedings of the 34th International Conference on Machine Learning (ICML)},
  series    = {Proceedings of Machine Learning Research},
  volume    = {70},
  pages     = {854--863},
  year      = {2017},
  address   = {Sydney, Australia},
  publisher = {PMLR},
  url       = {https://proceedings.mlr.press/v70/cisse17a.html}
}

@inproceedings{Behrmann2019InvertibleResNet,
  author    = {Jens Behrmann and Will Grathwohl and Ricky T. Q. Chen and David Duvenaud and J{\"o}rn-Henrik Jacobsen},
  title     = {Invertible Residual Networks},
  booktitle = {Proceedings of the 36th International Conference on Machine Learning (ICML)},
  series    = {Proceedings of Machine Learning Research},
  volume    = {97},
  pages     = {573--582},
  year      = {2019},
  address   = {Long Beach, USA},
  publisher = {PMLR},
  url       = {https://proceedings.mlr.press/v97/behrmann19a.html}
}

@inproceedings{Gulrajani2017WGANGP,
  author    = {Ishaan Gulrajani and Faruk Ahmed and Martin Arjovsky and Vincent Dumoulin and Aaron Courville},
  title     = {Improved Training of Wasserstein {GANs}},
  booktitle = {Advances in Neural Information Processing Systems (NeurIPS)},
  volume    = {30},
  pages     = {5767--5777},
  year      = {2017},
  address   = {Long Beach, USA},
  publisher = {Curran Associates, Inc.},
  url       = {https://papers.nips.cc/paper/2017/hash/892c3b60920f561af15af1b7f4ab8af4-Abstract.html}
}

@inproceedings{Wang2021PosteriorCollapse,
  author    = {Yixin Wang and David M. Blei and John P. Cunningham},
  title     = {Posterior Collapse and Latent Variable Non-identifiability},
  booktitle = {Advances in Neural Information Processing Systems (NeurIPS)},
  volume    = {34},
  pages     = {5443--5455},
  year      = {2021},
  address   = {Virtual Conference},
  publisher = {Curran Associates, Inc.},
  url       = {https://proceedings.neurips.cc/paper_files/paper/2021/file/2b6921f2c64dee16ba21ebf17f3c2c92-Paper.pdf}
}

@article{Sihombing2024Improved,
  author    = {Sihombing, Felix M. H. and Palin, Richard M. and Hughes, Hannah S. R. and Robb, Laurence J.},
  title     = {Improved Mineral Prospectivity Mapping Using Graph Neural Networks},
  journal   = {Ore Geology Reviews},
  volume    = {172},
  year      = {2024},
  month     = {September},
  pages     = {106215},
  doi       = {10.1016/j.oregeorev.2024.106215},
  publisher = {Elsevier},
  issn      = {0169-1368}
}

@article{Zuo2023NewGeneration,
  author = {Zuo, Renguang and Xiong, Yihui and Wang, Ziye and Wang, Jian and Kreuzer, Oliver P.},
  title = {A New Generation of Artificial Intelligence Algorithms for Mineral Prospectivity Mapping},
  journal = {Natural Resources Research},
  volume = {32},
  number = {5},
  pages = {1859--1869},
  year = {2023},
  month = {October},
  doi = {10.1007/s11053-023-10237-w},
  url = {https://doi.org/10.1007/s11053-023-10237-w},
  issn = {1573-8981}
}

@misc{Kitaev2020,
  author = {Kitaev, Nikita and Kaiser, {\L}ukasz and Levskaya, Anselm},
  title = {Reformer: The Efficient Transformer},
  year = {2020},
  url = {http://arxiv.org/abs/2001.04451},
  note = {arXiv preprint arXiv:2001.04451, ICLR 2020}
}

@inproceedings{Hou2022GraphMAE,
  author    = {Zhoujing Hou and Fei Liu and Yuxiao Dong and Chao Huang and Mingxuan Ju and Zhen Zhang and Yu Qiao and Yingyan Lin},
  title     = {GraphMAE: Self-Supervised Masked Graph Autoencoders},
  booktitle = {Proceedings of the 28th ACM SIGKDD Conference on Knowledge Discovery and Data Mining},
  year      = {2022},
  pages     = {594--604},
  publisher = {Association for Computing Machinery},
  doi       = {10.1145/3534678.3539321},
  url       = {https://doi.org/10.1145/3534678.3539321}
}

@article{Hou2023GraphMAE2,
  author  = {Zhoujing Hou and Xiaolin Zhang and Linna Wang and Zhiwei Zhu and Yushun Dong and Mingxuan Ju and Fei Liu and Yi Yang and Chuan Shi},
  title   = {GraphMAE2: A Decoding-Enhanced Masked Self-Supervised Graph Learner},
  journal = {Proceedings of the ACM Web Conference 2023 Companion},
  year    = {2023},
  pages   = {737--746},
  doi     = {10.1145/3543873.3587651},
  url     = {https://doi.org/10.1145/3543873.3587651}
}

@inproceedings{You2020GraphCL,
  author    = {Yuning You and Tianlong Chen and Yongduo Sui and Ting Chen and Zhangyang Wang and Yang Shen},
  title     = {Graph Contrastive Learning with Augmentations},
  booktitle = {Advances in Neural Information Processing Systems},
  volume    = {33},
  year      = {2020},
  pages     = {5812--5823},
  url       = {https://proceedings.neurips.cc/paper/2020/hash/3fe230348e9a12c13120749e3f9fa4cd-Abstract.html}
}

@inproceedings{Zhu2020GRACE,
  author    = {Yanqiao Zhu and Yichen Xu and Feng Yu and Qiang Liu and Shu Wu and Liang Wang},
  title     = {Deep Graph Contrastive Representation Learning},
  booktitle = {ICML Workshop on Graph Representation Learning and Beyond},
  year      = {2020},
  url       = {https://arxiv.org/abs/2006.04131}
}

@inproceedings{Thakoor2021BGRL,
  author    = {Shantanu Thakoor and Corentin Tallec and Mohammad Gheshlaghi Azar and Petar Veli{\v{c}}kovi{\'{c}}},
  title     = {Bootstrapped Representation Learning on Graphs},
  booktitle = {International Conference on Learning Representations (ICLR)},
  year      = {2022},
  url       = {https://openreview.net/forum?id=QrzVRAA49Ud}
}

@inproceedings{Zaheer2017DeepSets,
  author    = {Manzil Zaheer and Satwik Kottur and Siamak Ravanbakhsh and Barnab{\'a}s P{\'o}czos and Ruslan Salakhutdinov and Alexander J. Smola},
  title     = {Deep Sets},
  booktitle = {Advances in Neural Information Processing Systems},
  volume    = {30},
  year      = {2017},
  url       = {https://proceedings.neurips.cc/paper/2017/hash/f22e4747da1aa27e363d86d40ff442fe-Abstract.html}
}

@inproceedings{Lee2019SetTransformer,
  author    = {Juho Lee and Yoonho Lee and Jungtaek Kim and Adam R. Kosiorek and Seungjin Choi and Yee Whye Teh},
  title     = {Set Transformer: A Framework for Attention-based Permutation-Invariant Neural Networks},
  booktitle = {Proceedings of the 36th International Conference on Machine Learning},
  year      = {2019},
  pages     = {3744--3753},
  url       = {https://proceedings.mlr.press/v97/lee19d.html}
}

@inproceedings{Baek2021GMT,
  author    = {Jinheon Baek and Minki Kang and Sung Ju Hwang},
  title     = {Accurate Learning of Graph Representations with Graph Multiset Pooling},
  booktitle = {International Conference on Learning Representations (ICLR)},
  year      = {2021},
  url       = {https://openreview.net/forum?id=JHcqXGaqiGn}
}

\newpage
\appendix

\section{Appendix}

\subsection{Proof of Proposition~\ref{prop:lipschitz} - Node Feature Similarity Preservation.}
\label{appendix:proof-lipschitz}

\begin{proof}
Recall $h(u)=g(e_u)$ with $e_u=\pi(f(u))$, and assume $g$ is bi-Lipschitz: for all $e_1,e_2$,
\[
m\|e_1-e_2\| \le \|g(e_1)-g(e_2)\| \le L\|e_1-e_2\|.
\]
In particular, for $e_u,e_v$ we have
\begin{equation}\label{eq:bilip-h}
m\|e_u-e_v\| \le \|h(u)-h(v)\| \le L\|e_u-e_v\|.
\end{equation}
By the reconstruction assumption, $\|h(u)-x_u\|\le \varepsilon$ and $\|h(v)-x_v\|\le \varepsilon$.

\smallskip\noindent
\textbf{Upper bound on $\|e_u-e_v\|$.}
By the triangle inequality,
\[
\|h(u)-h(v)\|
\le \|h(u)-x_u\| + \|x_u-x_v\| + \|x_v-h(v)\|
\le \|x_u-x_v\| + 2\varepsilon.
\]
Combining with the left inequality in~\eqref{eq:bilip-h} gives
\[
m\|e_u-e_v\| \le \|x_u-x_v\| + 2\varepsilon
\quad\Longrightarrow\quad
\|e_u-e_v\| \le \frac{1}{m}\|x_u-x_v\| + \frac{2\varepsilon}{m}.
\]

\smallskip\noindent
\textbf{Lower bound on $\|e_u-e_v\|$.}
Using the reverse triangle inequality (equivalently, a two-term triangle bound),
\[
\|h(u)-h(v)\|
\ge \|x_u-x_v\| - \|h(u)-x_u\| - \|h(v)-x_v\|
\ge \|x_u-x_v\| - 2\varepsilon.
\]
Combining with the right inequality in~\eqref{eq:bilip-h} yields
\[
L\|e_u-e_v\| \ge \|x_u-x_v\| - 2\varepsilon
\quad\Longrightarrow\quad
\|e_u-e_v\| \ge \frac{\|x_u-x_v\| - 2\varepsilon}{L}.
\]
Since norms are nonnegative, one can equivalently write
$\|e_u-e_v\| \ge \max\!\left\{0,\frac{\|x_u-x_v\|-2\varepsilon}{L}\right\}$.

\smallskip
The same argument applies to any $\ell_p$ norm by replacing $\|\cdot\|$ with $\|\cdot\|_p$
and taking $m$ and $L$ to be the corresponding co-Lipschitz/Lipschitz constants for $g$
under $\ell_p$.
\end{proof}

\subsection{Relationship Between the Log-Based Surrogate and Normalized SCL}
\label{app:scl-deriv}

This appendix details the local relationship between the proposed log-based surrogate and the normalized/cosine form of Spectral Contrastive Loss (SCL). Assume embeddings are compared by cosine similarity, equivalently by inner product after unit normalization, so $\cos(z,z') \in [-1,1]$.

\paragraph{Log-Based Surrogate and Approximation by SCL.}
We define a log-based surrogate loss designed to explicitly target positive alignment ($\cos(z,z^{+}) = 1$) and negative orthogonality ($\cos(z,z^{-}) = 0$). The loss takes the form:
\[
\mathcal L_{\text{Surrogate}} = -\mathbb{E}_{(z,z^{+})}[\log (\text{PositiveArg})] - \mathbb{E}_{(z,z^{-})}[\log (\text{NegativeArg})].
\]
To ensure the arguments are in $(0,1]$ and achieve the value 1 at the desired optimal cosine similarities, we choose:
\[
\text{PositiveArg} = \exp(\cos(z,z^{+})-1) \quad \text{and} \quad \text{NegativeArg} = 1 - \frac{1}{2}\cos(z,z^{-})^2.
\]
The log-based surrogate is then:
\[
\mathcal L_{\text{Surrogate}} = -\mathbb{E}_{(z,z^{+})}[\log(\exp(\cos(z,z^{+})-1))] - \mathbb{E}_{(z,z^{-})}[\log(1 - \frac{1}{2}\cos(z,z^{-})^2)].
\]
Simplifying the positive term:
\[
\mathcal L_{\text{Surrogate}} = -\mathbb{E}_{(z,z^{+})}[\cos(z,z^{+})-1] - \mathbb{E}_{(z,z^{-})}[\log(1 - \frac{1}{2}\cos(z,z^{-})^2)].
\]
Dropping the constant $-1$ from the positive term (which does not affect minimization):
\[
\mathcal L_{\text{Surrogate}}' = -\mathbb{E}[\cos(z,z^{+})] - \mathbb{E}[\log(1 - \frac{1}{2}\cos(z,z^{-})^2)].
\]
Let $t^{+} = \cos(z,z^{+})$ and $t^{-} = \cos(z,z^{-})$ and $\tfrac12 \, \mathcal L_{\text{SCL}} = -\mathbb E[t^{+}] +\tfrac12\mathbb E[(t^{-})^{2}]$. The positive term of $\mathcal L_{\text{Surrogate}}'$ is exactly the positive term of $\tfrac12 \, \mathcal L_{\text{SCL}}$.

We show that the negative term of $\mathcal L_{\text{Surrogate}}'$ is approximated by the negative term of $\tfrac12 \, \mathcal L_{\text{SCL}}$ around the desired negative optimum $t^{-} = 0$. We perform a Taylor expansion of $-\log(1 - \frac{1}{2}t^2)$ around $t=0$:
The Taylor expansion of $-\log(1-x)$ around $x=0$ is $x + \frac{x^2}{2} + \frac{x^3}{3} + \dots$ for $|x|<1$.
Substitute $x = \frac{1}{2}t^2$. For $t \in [-1, 1]$, $t^2 \in [0, 1]$, so $\frac{1}{2}t^2 \in [0, 1/2]$, which means $|x|<1$ is satisfied.
\[
-\log(1 - \frac{1}{2}t^2) = (\frac{1}{2}t^2) + \frac{(\frac{1}{2}t^2)^2}{2} + \frac{(\frac{1}{2}t^2)^3}{3} + \dots = \frac{1}{2}t^2 + \frac{1}{8}t^4 + \frac{1}{24}t^6 + \dots
\]
So, the expectation of the negative term of $\mathcal L_{\text{Surrogate}}'$ is:
\[
-\mathbb{E}[\log(1 - \frac{1}{2}\cos(z,z^{-})^2)] = \mathbb{E}[\frac{1}{2}\cos(z,z^{-})^2 + \frac{1}{8}\cos(z,z^{-})^4 + \frac{1}{24}\cos(z,z^{-})^6 + \dots].
\]
Taking only the leading term from this expansion around $\cos(z,z^{-})=0$, we get $\mathbb{E}[\frac{1}{2}\cos(z,z^{-})^2]$.
Thus, the negative term of $\mathcal L_{\text{Surrogate}}'$ is approximated by the negative term of SCL around $\cos(z,z^{-})=0$.

Combining the exact positive term and the approximated negative term, we see that $\mathcal L_{\text{Surrogate}}'$ is approximated by scaled normalized SCL around the desired cosine geometry ($\cos(z,z^{+}) = 1, \cos(z,z^{-}) = 0$):
\[
\mathcal L_{\text{Surrogate}}' \approx -\mathbb{E}[\cos(z,z^{+})] + \mathbb{E}\!\left[\frac{1}{2}\cos(z,z^{-})^2\right] = \tfrac12 \, \mathcal L_{\text{SCL}}.
\]
This demonstrates that normalized SCL is a computationally simpler form that captures the leading behavior of a principled log-based surrogate near the desired cosine geometry.

\paragraph{Conclusion.}
Through the construction of a log-based surrogate whose terms explicitly target the desired cosine geometry, we obtain normalized SCL as the leading local quadratic approximation: positive pairs are driven toward angular alignment and unpaired embeddings toward angular decorrelation. This is the fixed-cosine geometry used by MABLE without requiring a learned discriminator.

\subsection{A DV-Style Interpretation of the Normalized/Cosine SCL Surrogate}
\label{appendix:mi-bound}

This appendix concerns the normalized/cosine SCL-style surrogate used in MABLE, not unconstrained raw inner-product SCL except when embeddings are unit-normalized or otherwise tightly norm-controlled. The original spectral contrastive loss of~\citet{HaoChen2021SCL} is written using raw inner products, whereas MABLE uses cosine similarity to remove scale ambiguity and keep scores bounded in $[-1,1]$. The DV-style argument below should therefore be read as a heuristic justification for the normalized/cosine surrogate used here.

Mutual information, $I(Z, Z^+)$, quantifies the statistical dependence between the latent representation $Z=f(X)$ and its augmented view $Z^+=f(X^+)$. A standard method to optimize MI is through variational lower bounds like the Donsker--Varadhan (DV) representation. The DV bound states that for any measurable function $T(z, z')$, a lower bound on MI is given by:
\[
I(Z,Z^+) \ge \mathbb{E}_{(Z,Z^+)}\bigl[T(Z,Z^+)\bigr] - \log \mathbb{E}_{Z,Z^- \sim p(Z)^2}\Bigl[e^{T(Z,Z^-)}\Bigr].
\]
Here, $Z^-$ denotes an independent sample from the marginal $p(Z)$ (so the second expectation is under the product distribution $p(Z)p(Z^-)$; we write $p(Z)^2$ for brevity). To relate this to the normalized/cosine SCL-style surrogate used in MABLE, we choose the test function $T(z,z') = \cos(z,z')$. Let $t^+ = \cos(Z,Z^+)$ and $t^- = \cos(Z,Z^-)$. Assume the latent embeddings $Z$ are unit-normalized, so $\cos(z,z') \in [-1,1]$ (and equals the dot product), and that $p(Z)$ is approximately antipodally symmetric (e.g., close to isotropic on the unit sphere $\mathbb{S}^{d-1}$), so that $t^-=\cos(Z,Z^-)$ is approximately symmetric about $0$ (hence $\mathbb{E}[t^-]\approx 0$ and odd moments are negligible). The bound becomes:
\[
I(Z,Z^+) \ge \mathbb{E}[t^+] - \log \mathbb{E}[e^{t^-}].
\]

We analyze the second term, $-\log \mathbb{E}[e^{t^-}]$, where $t^-=\cos(Z,Z^-)$ and $Z^-\sim p(Z)$ is independent of $Z$. Under this approximate symmetry, the expectation of any odd power of $t^-$ is approximately zero: $\mathbb{E}[(t^-)^{2k+1}] \approx 0$ for $k \ge 0$.

We perform a Taylor expansion of $e^x$ around $x=0$: $e^x = \sum_{n=0}^\infty \frac{x^n}{n!}$.
Substituting $x = t^-$ and taking the expectation:
\[
\mathbb{E}[e^{t^-}] = \mathbb{E}\left[\sum_{n=0}^\infty \frac{(t^-)^n}{n!}\right] = \sum_{n=0}^\infty \frac{\mathbb{E}[(t^-)^n]}{n!}.
\]
Using the property that odd moments are approximately zero:
\[
\mathbb{E}[e^{t^-}] = \frac{\mathbb{E}[(t^-)^0]}{0!} + \frac{\mathbb{E}[(t^-)^1]}{1!} + \frac{\mathbb{E}[(t^-)^2]}{2!} + \frac{\mathbb{E}[(t^-)^3]}{3!} + \dots
\]
\[
\mathbb{E}[e^{t^-}] \approx 1 + 0 + \frac{\mathbb{E}[(t^-)^2]}{2} + 0 + \dots = 1 + \sum_{k=1}^\infty \frac{\mathbb{E}[(t^-)^{2k}]}{(2k)!}.
\]

Now, we take the logarithm using the Taylor expansion of $\log(1+x)$ around $x=0$, which is $x - \frac{x^2}{2} + O(x^3)$. Let $x = \sum_{k=1}^\infty \frac{\mathbb{E}[(t^-)^{2k}]}{(2k)!} = \frac{\mathbb{E}[(t^-)^2]}{2} + \frac{\mathbb{E}[(t^-)^4]}{24} + \dots$.

Assuming $\mathbb{E}[(t^-)^2]$ is small (typically the case in high dimensions under near-uniform hyperspherical marginals), we can approximate:
\[
\log \mathbb{E}[e^{t^-}] = \log(1+x) \approx x - \frac{x^2}{2}
\]
\[
\log \mathbb{E}[e^{t^-}] \approx \left(\frac{\mathbb{E}[(t^-)^2]}{2} + \frac{\mathbb{E}[(t^-)^4]}{24} + \dots\right) - \frac{1}{2}\left(\frac{\mathbb{E}[(t^-)^2]}{2} + \dots\right)^2
\]

Keeping only the leading terms:
\[
\log \mathbb{E}[e^{t^-}] \approx \frac{1}{2}\mathbb{E}[(t^-)^2] + \text{higher order terms}.
\]

For the raw inner-product choice $T(z,z')=z^\top z'$, the same DV representation can formally be written, but the second-order reduction to a quadratic negative term requires stronger assumptions: finite exponential moments, approximate zero mean, small variance, negligible higher-order cumulants, and explicit norm control. Without norm control, $\log \mathbb E[\exp(Z^\top Z^-)]$ can be dominated by vector norms rather than angular dependence. This is the scale ambiguity avoided by the cosine surrogate used in MABLE.

Substituting this approximation back into the DV lower bound:
\[
I(Z,Z^+) \gtrsim \mathbb{E}[t^+] - \frac{1}{2}\mathbb{E}[(t^-)^2].
\]

Recall the scaled normalized/cosine SCL-style variant $\mathcal L_{\text{SSCL}}=\tfrac12\mathcal L_{\text{SCL}} = -\mathbb E[t^+] + \tfrac12 \mathbb E[(t^-)^2]$. Minimizing $\mathcal L_{\text{SSCL}}$ is equivalent to maximizing $-\mathcal L_{\text{SSCL}} = \mathbb E[t^{+}] -\tfrac12\mathbb E[(t^{-})^{2}]$. Thus, we see that:
\[
I(Z,Z^+) \gtrsim -\mathcal L_{\text{SSCL}}
\quad\text{(equivalently, } I(Z,Z^+) \gtrsim -\tfrac12 \mathcal L_{\text{SCL}}\text{).}
\]
This shows that, under the stated boundedness, approximate symmetry, and small-variance assumptions, minimizing the normalized/cosine SCL-style surrogate used in MABLE approximately maximizes a second-order approximation to this particular DV-style mutual-information objective. It should not be read as a general claim that unconstrained raw inner-product SCL has the same controlled DV approximation without additional norm-control assumptions.

\subsection{Proof of Proposition~\ref{prop:lipschitz-pooling} - Lipschitz Pooling is Stable Under Optimal Matching}
\label{appendix:pooling-proofs}

\begin{proof}
Let $Z=\{z_i\}_{i=1}^n$ and $Z'=\{z'_i\}_{i=1}^n$ be two multisets of node embeddings, and fix an arbitrary ordering $(z'_1,\dots,z'_n)$ of $Z'$.
Let $\xi^\star\in S_n$ attain the minimum in the matching distance, so that
\[
d_{\mathrm{match}}(Z,Z')=\sum_{i=1}^n\|z_i-z'_{\xi^\star(i)}\|.
\]
Because $P$ is permutation-invariant, reordering inputs does not change its output:
\[
P(Z')=P(z'_1,\dots,z'_n)=P(z'_{\xi^\star(1)},\dots,z'_{\xi^\star(n)}).
\]
Applying the Lipschitz property of $P$ with respect to $\|\cdot\|_{\mathrm{sum}}$ to the ordered tuples
$(z_1,\dots,z_n)$ and $(z'_{\xi^\star(1)},\dots,z'_{\xi^\star(n)})$ gives
\[
\|P(Z)-P(Z')\|
=
\|P(z_1,\dots,z_n)-P(z'_{\xi^\star(1)},\dots,z'_{\xi^\star(n)})\|
\le
L\sum_{i=1}^n\|z_i-z'_{\xi^\star(i)}\|
=
L\,d_{\mathrm{match}}(Z,Z'),
\]
which is the claimed bound.
\end{proof}

\subsection{Proof of Proposition~\ref{prop:sum-mean-max-lipschitz} - Sum, Mean, and Max Pooling are Lipschitz}
\label{appendix:proof-sum-mean-max-lipschitz}

\begin{proof}
Let $\{z_i\}_{i=1}^n$ and $\{z'_i\}_{i=1}^n$ be collections of $n$ vectors in $\mathbb{R}^d$. We verify the Lipschitz bounds stated in Proposition~\ref{prop:sum-mean-max-lipschitz} under the specified output norms and input metrics.

\textbf{Sum Pooling.}
Under the $\ell_2$ norm for individual vectors and the output:
\[
\bigl\|\text{SumPool}(\{z_i\}) - \text{SumPool}(\{z'_i\})\bigr\|_{\ell_2}
\;=\;
\Bigl\|\sum_{i=1}^n z_i - \sum_{i=1}^n z'_i\Bigr\|_{\ell_2}
\;=\;
\Bigl\|\sum_{i=1}^n (z_i - z'_i)\Bigr\|_{\ell_2}.
\]
By the triangle inequality for vector sums:
\[
\Bigl\|\sum_{i=1}^n (z_i - z'_i)\Bigr\|_{\ell_2}
\;\le\;
\sum_{i=1}^n \|z_i - z'_i\|_{\ell_2}.
\]
Thus, $\|\text{SumPool}(\{z_i\}) - \text{SumPool}(\{z'_i\})\|_{\ell_2} \le 1 \cdot \sum_{i=1}^n \|z_i - z'_i\|_{\ell_2}$.
Hence the Lipschitz constant for SumPool under the $\ell_2$ norm is at most $1$.

\textbf{Mean Pooling.}
MeanPool is defined as SumPool scaled by $1/n$. Under the $\ell_2$ norm:
\[
\|\text{MeanPool}(\{z_i\}) - \text{MeanPool}(\{z'_i\})\|_{\ell_2}
\;=\;
\left\|\frac{1}{n}\sum_{i=1}^n z_i - \frac{1}{n}\sum_{i=1}^n z'_i\right\|_{\ell_2}
\;=\;
\frac{1}{n}\Bigl\|\sum_{i=1}^n (z_i - z'_i)\Bigr\|_{\ell_2}.
\]
By the triangle inequality for vector sums:
\[
\frac{1}{n}\Bigl\|\sum_{i=1}^n (z_i - z'_i)\Bigr\|_{\ell_2}
\;\le\;
\frac{1}{n}\sum_{i=1}^n \|z_i - z'_i\|_{\ell_2}.
\]
Thus, $\|\text{MeanPool}(\{z_i\}) - \text{MeanPool}(\{z'_i\})\|_{\ell_2} \le \frac{1}{n} \cdot \sum_{i=1}^n \|z_i - z'_i\|_{\ell_2}$.
Hence the Lipschitz constant for MeanPool under the $\ell_2$ norm is at most $1/n$.

\textbf{MaxPool.}
Define the $k$-th coordinate of MaxPool as $\text{MaxPool}(\{z_i\})_k = \max_{1 \le i \le n} (z_i)_k$. We consider the $\ell_\infty$ norm for the output of MaxPool. The difference in outputs is:
\[
\|\text{MaxPool}(\{z_i\}) - \text{MaxPool}(\{z'_i\})\|_{\infty}
\;=\;
\max_{k} \left|\max_{i} (z_i)_k - \max_{i} (z'_i)_k\right|.
\]
A standard property of the maximum function is $\left|\max_i a_i - \max_i b_i\right| \le \max_i \left|a_i - b_i\right|$. Applying this coordinate-wise:
\[
\max_{k} \left|\max_{i} (z_i)_k - \max_{i} (z'_i)_k\right|
\;\le\;
\max_{k} \left(\max_{i} \left|(z_i)_k - (z'_i)_k\right|\right).
\]
We can swap the order of maximization:
\[
\max_{k} \left(\max_{i} \left|(z_i)_k - (z'_i)_k\right|\right)
\;=\;
\max_{i} \left(\max_{k} \left|(z_i)_k - (z'_i)_k\right|\right).
\]
The term $\max_{k} \left|(z_i)_k - (z'_i)_k\right|$ is the $\ell_\infty$ norm of the vector difference $z_i - z'_i$, i.e., $\|z_i - z'_i\|_{\ell_\infty}$. So the expression becomes $\max_{i} \|z_i - z'_i\|_{\ell_\infty}$. Thus,
\[
\|\text{MaxPool}(\{z_i\}) - \text{MaxPool}(\{z'_i\})\|_{\infty}
\;\le\;
\max_{i} \|z_i - z'_i\|_{\ell_\infty},
\]

Using $\|v\|_{\ell_\infty}\le \|v\|_{\ell_2}$ for all $v\in\mathbb{R}^d$ and $\max_i a_i \le \sum_{i=1}^n a_i$ for $a_i\ge 0$, we obtain
\[
\max_{i} \|z_i - z'_i\|_{\ell_\infty}
\;\le\;
\max_{i} \|z_i - z'_i\|_{\ell_2}
\;\le\;
\sum_{i=1}^n \|z_i - z'_i\|_{\ell_2}.
\]
Combining the inequalities yields
\[
\|\text{MaxPool}(\{z_i\}) - \text{MaxPool}(\{z'_i\})\|_{\infty}
\;\le\;
\sum_{i=1}^n \|z_i - z'_i\|_{\ell_2},
\]
If the output is instead measured in $\ell_2$, then using $\|v\|_{\ell_2}\le \sqrt{d}\,\|v\|_{\ell_\infty}$ gives
\[
\|\text{MaxPool}(\{z_i\}) - \text{MaxPool}(\{z'_i\})\|_{\ell_2}
\;\le\;
\sqrt{d}\,\sum_{i=1}^n \|z_i - z'_i\|_{\ell_2},
\]
which yields the stated $\sqrt{d}$ bound.
\end{proof}

\subsection{Proof of Lemma~\ref{lem:attn-lipschitz} - Bounded-Domain Lipschitzness of Clipped Softmax Attention}
\label{appendix:attn-lipschitz-proof}

\begin{proof}
Let $Z=(z_1,\dots,z_n)$ and $Z'=(z'_1,\dots,z'_n)$ lie in $\mathcal{D}_B$, and use the product metric
\[
d_{\mathrm{sum}}(Z,Z'):=\sum_{i=1}^n\|z_i-z'_i\|_2.
\]
The normalization map $N_\varepsilon(z)=z/\max(\|z\|_2,\varepsilon)$ is Lipschitz for fixed $\varepsilon>0$ (with a constant depending only on $\varepsilon$). Since $\psi$ is Lipschitz, applying $\psi\circ N_\varepsilon$ coordinate-wise to the $n$ nodes gives a Lipschitz map from $Z$ to the raw score vector $\tilde s(Z)$. The map $\tilde s\mapsto \tilde s-\max_j \tilde s_j\mathbf{1}$ is Lipschitz because the maximum function is Lipschitz, coordinate-wise clipping to $[-c,0]$ is non-expansive, and softmax is smooth and Lipschitz on the compact cube $[-c,0]^n$. Therefore the attention-weight vector $\alpha(Z)=(\alpha_1(Z),\dots,\alpha_n(Z))\in\Delta^{n-1}$ is Lipschitz on $\mathcal{D}_B$; i.e., for some finite constant $L_\alpha$,
\[
\|\alpha(Z)-\alpha(Z')\|_1
=
\sum_{i=1}^n |\alpha_i(Z)-\alpha_i(Z')|
\le L_\alpha\, d_{\mathrm{sum}}(Z,Z').
\]
Now decompose the attention-weighted sums:
\[
\begin{aligned}
\|A(Z)-A(Z')\|_2
&=
\left\|\sum_{i=1}^n \alpha_i(Z)z_i-\sum_{i=1}^n \alpha_i(Z')z'_i\right\|_2\\
&\le
\left\|\sum_{i=1}^n \alpha_i(Z)(z_i-z'_i)\right\|_2
+
\left\|\sum_{i=1}^n \bigl(\alpha_i(Z)-\alpha_i(Z')\bigr)z'_i\right\|_2\\
&\le
\sum_{i=1}^n \alpha_i(Z)\|z_i-z'_i\|_2
+
B\sum_{i=1}^n |\alpha_i(Z)-\alpha_i(Z')|\\
&\le
\bigl(1+B L_\alpha\bigr)\,d_{\mathrm{sum}}(Z,Z').
\end{aligned}
\]
Thus $A$ is Lipschitz on $\mathcal{D}_B$. If $\sigma_{\max}(W_{pool})\le L_W$, then
\[
\|W_{pool}A(Z)-W_{pool}A(Z')\|_2
\le
L_W\|A(Z)-A(Z')\|_2,
\]
so $P_{\mathrm{attn}}=W_{pool}\circ A$ is Lipschitz on the same bounded domain.
\end{proof}

\subsection{Proof of Corollary~\ref{prop:cond-pro-pool} - Conditioned Projection Preserves Attention-Induced Separation}
\label{appendix:bilip-pool-proof}

\begin{proof}
Recall that the pooling operation is
\[
P_{\mathrm{attn}}(\{z_i\})
\;=\;
W\Bigl(\sum_{i=1}^n \alpha_i(\{z_k\}_{k=1}^n) \, z_i\Bigr),
\]
where the matrix $W \in \mathbb{R}^{d_G \times d}$ has $\sigma_{\min}(W) \ge m > 0$. This condition means $W$ is $m$-co-Lipschitz:
\[
\|\,W(u) - W(v)\|
\;\ge\;
m\,\|\,u - v\|
\quad
\text{for all } u,v \in \mathbb{R}^d.
\]
Let $\{z_i\}$ and $\{z'_i\}$ be two distinct sets of embeddings from the data distribution. Define their respective attention-weighted sums as:
\[
\Delta_z
\;=\;
\sum_{i=1}^n \alpha_i(\{z_k\}_{k=1}^n) \, z_i,
\quad
\Delta_{z'}
\;=\;
\sum_{i=1}^n \alpha_i(\{z'_k\}_{k=1}^n) \, z'_i.
\]
Let $\delta_{\text{pair}} = \|\Delta_z - \Delta_{z'}\|$ denote the actual separation achieved between these weighted sums by the attention mechanism $\alpha_i(\cdot)$ for this specific pair. The difference between the pooled outputs is
\[
\|\,P_{\mathrm{attn}}(\{z_i\}) - P_{\mathrm{attn}}(\{z'_i\})\|
\;=\;
\Bigl\| W(\Delta_z) \;-\; W(\Delta_{z'})\Bigr\|.
\]
Applying the $m$-co-Lipschitz property of $W$ with $u = \Delta_z$ and $v = \Delta_{z'}$:
\[
\Bigl\| W(\Delta_z) \;-\; W(\Delta_{z'})\Bigr\|
\;\ge\;
m\,\|\Delta_z - \Delta_{z'}\|.
\]
Substituting the definition of $\delta_{\text{pair}}$, we obtain
\[
\|\,P_{\mathrm{attn}}(\{z_i\}) - P_{\mathrm{attn}}(\{z'_i\})\|
\;\ge\;
m \cdot \delta_{\text{pair}},
\]
as desired. This inequality holds for any pair; the corollary highlights the consequence when $\delta_{\text{pair}} > 0$.
\end{proof}

\subsection[Ablation on feature-slice dimensionality de]{Ablation on feature-slice dimensionality \(d_e\)}
\label{appendix:ablation-feature-slice}

\begin{figure*}[h!]
    \centering
    \includegraphics[width=\textwidth]{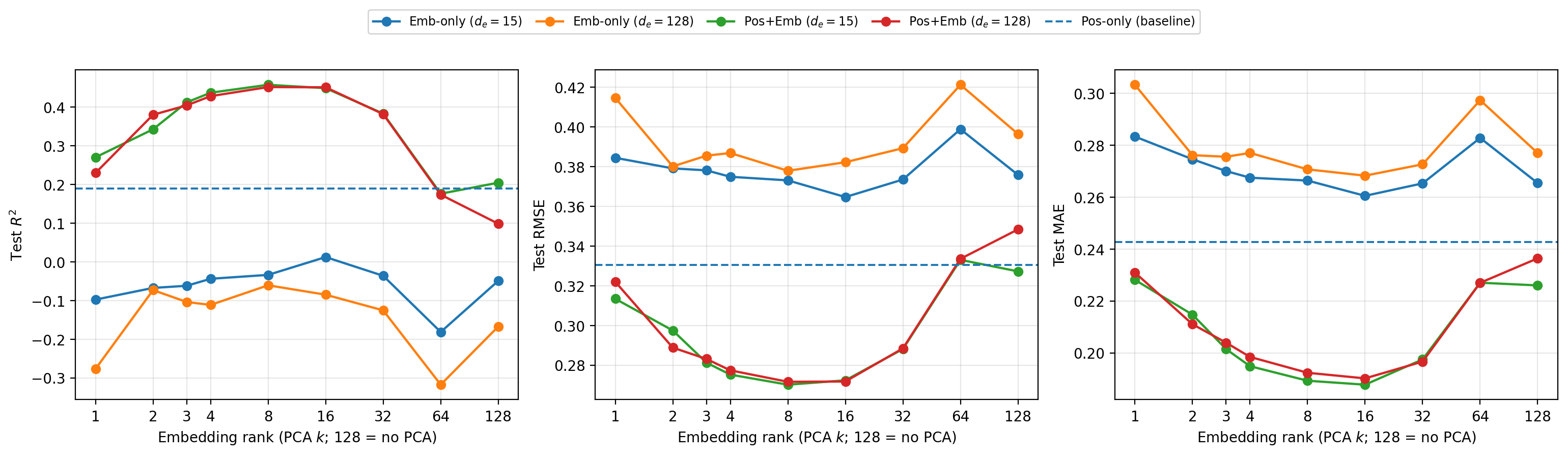}
    \caption{Embedding dimension ablation under spatial holdout (south$\rightarrow$north) with spatial CV on \textit{CuEq PK Regression} dataset (Experiment 1, Study 4). Comparison of Random Forest performance using pretrained embeddings produced with $d_e=15$ vs.\ $d_e=128$.
    Curves show \texttt{emb\_only} and \texttt{pos\_plus\_emb} across PCA rank $k$ (log-scale; $k=128$ denotes no PCA), and the dashed line indicates the \texttt{pos\_only} baseline.}
    \label{fig:cueqpk_de_ablation_spatial_rf}
\end{figure*}

\begin{table*}[h!]
\centering
\small
\setlength{\tabcolsep}{4pt}
\resizebox{\linewidth}{!}{
\begin{tabular}{rcccccccc}
\toprule
$d_e$ & Best step ($10^3$ updates) & Train total & Val total & Gap (Val--Train) &
Val graph-align & Val graph-uniform & Val recon (orig) & Val node-uniform \\
\midrule
8   & 84.8 & 0.0680 & 0.09109 & 0.02309 & 0.04088 & 0.03820 & $5.0138\times 10^{-3}$ & $6.9950\times 10^{-3}$ \\
15  & 68.0 & 0.0552 & 0.09120 & 0.03597 & 0.04265 & 0.03623 & $6.0839\times 10^{-3}$ & $6.2394\times 10^{-3}$ \\
128 & 80.6 & 0.0367 & 0.09034 & 0.05364 & 0.04195 & 0.03696 & $5.2452\times 10^{-3}$ & $6.1891\times 10^{-3}$ \\
\bottomrule
\end{tabular}}
\caption{\textbf{Feature-slice dimensionality ablation.} Best-checkpoint training and validation losses for three feature-slice widths $d_e$ (selected by minimum validation total loss) for \textit{CuEq PK Regression} dataset (Experiment 1, Study 4). While $d_e{=}128$ achieves the lowest validation total loss, $d_e{=}15$ is effectively comparable on validation loss while reaching its best checkpoint earlier (about \(\sim\)84\% of the 128d training steps in our runs) and exhibiting a smaller train--validation gap than $d_e{=}128$, consistent with reduced overfitting.}
\label{tab:cueqpk_de_ablation_losses}
\end{table*}

\newpage

\end{document}